\documentclass{article}

\usepackage[english]{babel}
\usepackage{enumerate}

\usepackage[letterpaper,top=2cm,bottom=2cm,left=3cm,right=3cm,marginparwidth=1.75cm]{geometry}

\usepackage{amsmath}
\usepackage{amssymb}
\usepackage{graphicx}
\usepackage{centernot}
\usepackage[colorlinks=true, allcolors=blue]{hyperref}
\usepackage{comment}
\usepackage{tikz}
\usepackage{soul}
\usepackage{subfigure}
\usepackage{bm}

\usepackage{amsthm}

\makeatletter
\renewenvironment{proof}[1][\proofname]{
  \par\pushQED{\qed}\normalfont
  \topsep6\p@\@plus6\p@\relax
  \trivlist\item[\hskip\labelsep\bfseries#1\@addpunct{.}]
  \ignorespaces
}{
  \popQED\endtrivlist\@endpefalse
}
\makeatother

\makeatletter
\newcommand{\oset}[3][0ex]{
	\mathrel{\mathop{#3}\limits^{
			\vbox to#1{\kern-2\ex@
				\hbox{$\scriptstyle#2$}\vss}}}}
\makeatother

\newtheorem{proposition}{Proposition}
\newtheorem{theorem}[proposition]{Theorem}

\newtheorem{lemma}[proposition]{Lemma}
\newtheorem{remark}[proposition]{Remark}

\newtheorem{assumption}[proposition]{Assumption}

\def\tr{{\rm tr}}

\def\diag{{\rm diag}}

\def\rank{{\rm rank}}

\def\block{{\rm blk}}

\usepackage{dsfont}
\newcommand{\1}{\mathds{1}}

\newcommand{\beq}{\begin{equation}}
\newcommand{\eeq}{\end{equation}}
\newcommand{\beqa}{\begin{eqnarray}}
\newcommand{\eeqa}{\end{eqnarray}}
\newcommand{\beqan}{\begin{eqnarray*}}
\newcommand{\eeqan}{\end{eqnarray*}}

\newcommand{\pde}[2]{ \frac{\partial #1}{\partial #2} }

\newcommand{\bite}{\begin{itemize}}
\newcommand{\eite}{\end{itemize}}
\newcommand{\benu}{\begin{enumerate}}
\newcommand{\eenu}{\end{enumerate}}

\title{Gradient Flow Equations for Deep Linear Neural Networks: A Survey from a Network Perspective}
\author{Joel Wendin and Claudio Altafini\footnote{Corresponding author: claudio.altafini@liu.se} \\ {\itshape\small Department of Electrical Engineering,}\\ {\itshape\small Link{\"o}ping University, SE-58183 Link{\"o}ping, Sweden}}

\begin{document}
	\maketitle

	\begin{abstract}
The paper surveys recent progresses in understanding the dynamics and loss landscape of the gradient flow equations associated to deep linear neural networks, i.e., the gradient descent training dynamics (in the limit when the step size goes to 0) of deep neural networks missing the activation functions and subject to quadratic loss functions. 
When formulated in terms of the adjacency matrix of the neural network, as we do in the paper, these gradient flow  equations form a class of converging matrix ODEs which is nilpotent, polynomial, isospectral, and with conservation laws. 
The loss landscape is described in detail. It is characterized by infinitely many global minima and saddle points, both strict and nonstrict, but lacks local minima and maxima.
The loss function itself is a positive semidefinite Lyapunov function for the gradient flow, and its level sets are unbounded invariant sets of critical points,  with critical values that correspond to the amount of singular values of the input-output data learnt by the gradient along a certain trajectory. 
The adjacency matrix representation we use in the paper allows to highlight the existence of a quotient space structure in which each critical value of the loss function is represented only once, while all other critical points with the same critical value belong to the fiber associated to the quotient space. It also allows to easily determine stable and unstable submanifolds at the saddle points, even when the Hessian fails to obtain them.

	\end{abstract}
	
\section{Introduction}

In the last decade, deep learning has proven remarkably successful in a number of practical applications across different fields, such as computer vision, speech recognition, and natural language processing \cite{lecun2015deep}. 
What is still lagging behind is the theoretical understanding of the reasons underpinning these successes.
From the point of view of optimization, the loss function used in deep learning is typically non-convex, and it is known since \cite{blum1988training} that even on simple examples of neural networks global optimality is an NP-hard problem. 
Furthermore, the input-output map of a deep neural network is highly nonlinear, due to the combination of weight multiplications across layers and nonlinear activation functions at the hidden layers.
Nevertheless, even simple first order gradient descent algorithms typically converge to small loss errors, without getting trapped into poor local minima. 
Furthermore, the trained networks generalize well to new data, in spite of the overparametrization typical of the deep learning regime.

The nonlinear nature of deep neural networks makes it hard to understand what principles may lie behind their efficient performances and what features are crucially contributing to them.
A convenient, yet simplified, setting where to explore these issues is provided by the so-called deep linear neural networks, i.e., multi-layer feedforward neural networks lacking the activation functions.
Even though not more expressible than linear systems (e.g. least-squares for regression problems), such class of deep networks is interesting because it exhibits several of the aforementioned features, like non-convex loss, overparametrization, trainability by gradient descent, nonlinearity of the training dynamics, but at the same time it is mathematically treatable, and its behavior can be understood analytically. 

In fact, in recent years a number of results have appeared in the literature for deep linear neural networks. 
While in the shallow network (i.e., 1 hidden layer) case with quadratic loss it was known since \cite{baldi1989neural} that no local minima can exist, the extension of this result to deep linear networks is due to \cite{kawaguchi2016deep}, who showed that any local minimum is also global, and that any other critical point is a saddle point. The same holds for any convex differentiable loss function provided that the network architecture lacks bottlenecks \cite{laurent2018deep,trager2019pure}.
A simple criterion to classify global minima and saddle points was provided in \cite{yun2017global}. See also \cite{zhou2017critical} for an alternative formulation.
Gradient dynamics of linear neural networks have been studied in many papers, both in continuous \cite{saxe2013exact,arora2019implicit,tarmoun2021understanding,chitour2018geometric,cohen2024lecture} and discrete time \cite{arora2018convergence,arora2018optimization,gidel2019implicit}, showing that different types of initializations lead to different learning dynamics: when the initial condition is near the origin, then the trajectory typically moves along near-flat portions of the loss landscape, with sharp dips when a new singular value of the input-output data is discovered. The learning is sequential, from the largest to the smallest singular value \cite{saxe2013exact,gidel2019implicit}, and it is also referred to as incremental (or saddle-to-saddle) learning \cite{gidel2019implicit,gissin2019implicit,jacot2021saddle}.
When instead a trajectory is initialized away from the origin, learning occurs with a higher convergence rate, and all singular values are learnt simultaneously. 
For wide enough networks, the learning dynamics associated to this type of initialization is often referred to as ``lazy'' learning: the initial values appear to be closer to a global minimum than to a saddle point and the training dynamics behaves as linear \cite{jacot2021saddle}.
The interplay between network width and initialization scale is investigated and related to lazy and ``active'' learning e.g. in \cite{tu2024mixed}. However, width alone is not sufficient to induce lazy learning, as, by carefully choosing the  initialization, it may occur in finitely wide networks \cite{kunin2024get}, and it can be avoided in infinitely wide networks \cite{chizat2024infinite}.
Initializing near the origin means imposing (near) ``balanced'' conservation laws, and leads to ``synchronized'' layers, i.e., along the trajectory and also at the critical point the singular values of the various layers are nearly identical. Also the singular vectors of the layers are normally nearly aligned.
In particular, in the so-called 0-balance case exact solutions are possible for shallow networks \cite{saxe2013exact,domine2023exact}, and also the analysis of the singular values of deep networks simplifies \cite{arora2019implicit}.
Another special case in which singular vectors of the layers stay aligned along a trajectory is the so-called decoupled initialization (also referred to as spectral or training aligned initialization) \cite{saxe2019mathematical,tarmoun2021understanding,gidel2019implicit,min2023convergence,lampinen2018analytic}.
Convergence rate of gradient descent for various initializations and topologies are analyzed in e.g. \cite{arora2018convergence,chitour2018geometric,tarmoun2021understanding,min2023convergence}.

Other papers we rely heavily upon are \cite{chitour2018geometric}, where ``pointwise convergence'' (i.e., convergence of each trajectory to a critical point) is shown using Lojasiewicz's theorem \cite{lojasiewicz1982trajectoires}, and the use of the Hessian quadratic form (instead of the Hessian matrix, which requires a vectorization of the state space) is exploited, and \cite{achour2021loss}, which gives a thorough description and an explicit parametrization of the critical points, including at second order, whereby critical points are classified into strict (when a negative direction in the Hessian quadratic form exists) and non-strict (when such a direction does not exist). 
The paper \cite{bah2022learning} treats gradient dynamics using a Riemannian geometry approach, to show that in the shallow case strict saddle points are almost always avoided \cite{chitour2018geometric,lee2019first}, and investigates aspects of the implicit regularization phenomenon, i.e., the tendency to fit the training data with models of low-complexity.
More specifically, in \cite{bah2022learning,achour2021loss} the global minima obtained once the evolution is restricted to manifolds in which the product of weight matrices is rank-constrained are computed, in a way which is reminiscent of the Eckart-Young theorem for truncated singular value decompositions (SVD). 
Different perspectives on the implicit regularization problem are discussed in \cite{gunasekar2017implicit,du2018algorithmic,li2020towards,arora2019implicit,belabbas2020implicit} and summarized e.g. in \cite{cohen2024lecture,li2020towards} for matrix factorization problems: at least when the initialization is near the origin, the conclusion of \cite{cohen2024lecture,li2020towards} is that the learning process tends to produce ``greedy low rank'' models.

Our aim in this paper is to review some of the results mentioned above, and to give a self-contained overview of the properties of the training dynamics and loss landscape of deep linear neural networks, focusing on the gradient flow equation associated to a quadratic loss function, i.e., on the ODE which is obtained from a gradient descent algorithm when the cost function is a mean squared error and the step size goes to $0$. 
For such ODE the state variables are the parameters of the network, i.e., the edge weights connecting consecutive layers of the network, from inputs to outputs, passing through the hidden layers.
To keep the exposition as simple as possible, we concentrate on the overdetermined and overparametrized case, in which sufficiently many input-output data points are available and the hidden layers have at least as many nodes as the output layer.

Our original contribution is to take a global perspective, and to reformulate the gradient flow equation in terms of the adjacency matrix of the network, rather than dealing with as many distinct matrix ODEs as there are layers in the network. 
This choice allows to compactify the notation, to streamline (and often simplify) the results, and to gain insight into the system properties.
For instance the feedforward nature of the network reflects in the block-shift structure of the adjacency matrix, which is therefore nilpotent. 
Products of the weight matrices corresponding to the different layers become powers of the adjacency matrix, up to the nilpotency index, which corresponds to the number of layers of the network. The adjacency matrix at that power corresponds to the input-output map of the network, and enters into the quadratic loss function.
Concerning the dynamics, when written in terms of the adjacency matrix, the gradient flow of a deep linear neural network forms an interesting class of matrix ODEs: it is polynomial, isospectral, with conservation laws and always converging.
As mentioned above, it does not have any local minima or maxima, but only infinitely many global minima and saddle points, both strict and non-strict. 

Recall that the standard way to understand the nature of the critical points consists in investigating the Hessian of the loss function, which however only provides a classification up to the second order variation, not a complete one. 
For instance, for the quartic loss function $ \mathcal{L}(x)= x^4/4$, the Hessian is vanishing at the critical point $ x=0$, but clearly $ x=0$ is the global minimum for $ \mathcal{L}(x)$, and it is a globally asymptotically stable equilibrium point for the gradient flow $ \dot x = - x^3 $.

For matricial state space representations (as in our case) the Hessian can always be intended as the matrix one obtains by vectorizing the state. 
When the Hessian is nonsingular, positive resp. negative definite Hessians correspond to local minima resp. maxima, while mixed sign eigenvalues correspond to saddle points.
However, when the Hessian is singular, as it is always the case for our gradient flow, determining the stability/curvature character of the critical points becomes a more challenging problem, and looking at the Hessian may no longer suffice. 
In particular, when we are dealing with saddle points, if the Hessian has at least one negative eigenvalue then the saddle point is strict, while when the Hessian is positive semidefinite the saddle point is non-strict. 
While there is literature on the fact that gradient-based algorithms almost surely avoid strict saddle points \cite{lee2019first,bah2022learning}, even non-isolated ones \cite{panageas2016gradient}, in principle the same may not be true for non-strict saddle points.
In fact, non-strict saddle points are often associated to plateaus in the local loss landscape, and around them the convergence of a gradient flow solution may slow down significantly, possibly leading to a premature stopping in a gradient descent training algorithm.
For matrix ODEs, an alternative to vectorizing the state is to maintain it in a matrix form, and to consider the Hessian quadratic form. This is the approach followed in \cite{chitour2018geometric,achour2021loss}, and also in this paper. Strictness or less of the saddle points can be deduced by probing the possible variations (themselves matrices in block-shift form). 
In the paper we also show that it is possible to bypass completely the Hessian picture, and to identify in a systematic way stable and unstable submanifolds of all critical points of the gradient flow through an explicit analysis of the loss landscape.

The singularity of the Hessian is a consequence of the fact that  no critical point of the gradient flow is isolated. In fact, each critical value of the loss function is associated to an unbounded invariant set for the dynamics. 
For instance, the loss function is a Lyapunov function for the gradient flow, but it is only positive semidefinite and it has a continuum of global minima, meaning that none of them can be asymptotically stable.
The fact of having infinitely-many global minima scattered around the state space facilitates convergence (any initial condition has some global minimum ``nearby'') but complicates the description of the loss landscape. 

To better understand this feature, recall that in our formulation the loss function depends on a power of the adjacency matrix, not on the adjacency matrix itself.
Infinitely-many adjacency matrices correspond to the same matrix power, meaning that there are infinitely-many adjacency matrices for each value of the loss function. 
A possible way to characterize the structure of the state space is to split it into  equivalence classes of the matrix power operation. 
The fibers over the quotient space associated to this equivalence relation correspond to all matrices having the same power.
In particular, when we restrict to critical points, the quotient space contains only one representative for each family of equivalent critical points, i.e., each level surface of the loss function (which is also an invariant set of the gradient flow) is associated to a single point in the quotient space.
Our equivalence relation is inspired by the decomposition of loss functions proposed in \cite{trager2019pure}.
In our formulation, the fibers over the quotient space induced by the equivalence relation correspond to changes of basis in hidden node space, plus addition of terms that disappear when taking powers.
Each trajectory of the gradient flow converges to a specific critical point inside the fiber associated to the critical value achieved asymptotically, depending on the initial condition.

Recall that what is being learnt by the gradient flow are the singular values of the input-output data, and that these enter into a power of the adjacency matrix of the deep linear neural network.
It is known since \cite{de1992structure} that the SVD does not extend well to matrix products (and therefore also to matrix powers), see also e.g. \cite{zhang2017matrix}, Ch.~5. 
This difficulty is what complicates the explicit description of the fibers and the  interpretation of the loss landscape: within each fiber over our quotient space are critical points (i.e., adjacency matrices) that look very differently even though they are associated to the same critical value of the loss function.
As already mentioned, also the gradient flow trajectories that converge to such critical points can behave rather differently depending on the initial conditions.
The aforementioned initialization procedures often used in the literature, 0-balance and decoupled, correspond to cases in which the SVD (in a special form which we call block-shift SVD) propagate well to matrix powers, which significantly simplifies the dynamics.
These special cases are investigated in detail in the paper.

The rest of this paper is structured as follows. The standard formulation of the gradient flow equations is presented in Section~\ref{sec:standard}, followed by the reformulation one gets when using adjacency matrices in Section~\ref{sec:adjacency}. Section~\ref{sec:conv-analysis} contains an overview of the convergence analysis and loss landscape of the gradient flow, while Section~\ref{sec:simplifications} gives more details on the simplified dynamics resulting from special initializations. 
Examples are given in Section~\ref{sec:examples}. Finally Section~\ref{sec:extensions} briefly mentions a series of possible extensions of the present work and Section~\ref{sec:conclusion} concludes the paper.

\tableofcontents

\section{Problem formulation}
\label{sec:standard}

\subsection{Notation}
\label{sec:notation}
Vectors are indicated with boldface, lower case, letters, $ \mathbf{b}\in \mathbb{R}^n$, and matrices with upper case (Latin or Greek) letters, $ B \in \mathbb{R}^{n\times m}$, of elements $ [B]_{jk} $. When we consider $B$ as a block matrix, we use the notation $ [B]_{(j,k)}$ to indicate its blocks. 
The symbol $ \1 $ indicates a vector of 1, and $\mathbf{e}_i$ is the elementary vector with 1 in the $i$-th position, while $ \mathbb{E}_{i,j} $ is the matrix having a 1 in the $(i,j)$ slot and 0 elsewhere, and $ \mathbb{I}_{n,m}$ is the $ n\times m $ matrix of $1$. 
Eigenvalues of a matrix $B\in\mathbb{R}^{n\times n}$ are denoted $\lambda_i(B)$, with $\lambda_1(B)\leq...\leq\lambda_n(B)$ and $ {\rm Sp}(B) = \{ \lambda_1(B), \ldots, \lambda_n(B) \}$ is the spectrum. Singular values of a matrix $B\in\mathbb{R}^{m\times n}$ are denoted $\sigma_i(B)$ and indexed in reverse order: $\sigma_1(B)\geq...\geq\sigma_r(B)\geq 0$, $r=\min\{n,m\}$. The null space of $B$ is $\mathcal{N}(B)$, and the column space is $\mathcal{R}(B)$. A matrix is said normal if it commutes with its transpose: $ [B, \, B^\top ] =B B^\top - B^\top B =0 $. We denote the Frobenius norm $\lVert\cdot\rVert_F$, and the associated inner product $ \langle  \cdot, \cdot \rangle_F$.
For a symmetric matrix $B$ which is positive (negative) semidefinite we write $ B \succeq 0 $ ($ B \preceq 0 $).

Given a set of binary variables $ s_i \in \{ 0, \, 1 \}$, $ i =1, \ldots, n$, we denote $ \mathbf{s} = [s_1\, \ldots\, s_n ]^\top $ the associated vector, and $ \mathcal{S} $ the indicator set: $ \mathcal{S} =\{ i \; \text{s.t.} \; s_i =1,\; i =1, \ldots, n  \}$. Sometimes we also call $ S =\diag(\mathbf{s}) $ the diagonal matrix having $ \mathbf{s} $ in the diagonal. 

Table~\ref{tab:symb-table} in Appendix~\ref{sec:symb-table} provides a quick reference to additional notation introduced throughout the paper.

\subsubsection{Basic calculus rules for matrices in block-shift form}
\label{sec:block-shift}
Given a matrix composed of $ (h+1) \times (h+1) $ matrix blocks of compatible dimensions, we use the notation ``$ \block_k(F_1, \ldots, F_{h+1-k}) $'' to indicate the matrix having the blocks $ F_1, \ldots, F_{h+1-k} $ in the $k$-th lower diagonal blocks, $k=0,1, \ldots, h$. 
For instance $ \block_0(F_1, \ldots, F_{h+1}) $ is the block-diagonal matrix,  $ \block_1(F_1, \ldots, F_h) $ is the following matrix in block-shift form
\[
\block_1(F_1, \ldots, F_h) =
\begin{bmatrix} 0 & \ldots &  & 0 \\ 
F_1 & \ddots & & \vdots \\
0  & \ddots & 0 &  \\
  & \ldots & F_{h}& 0  \\
\end{bmatrix},
\;\; 
\text{and} \;\; 
\block_h(F_1) = \begin{bmatrix} 0 & \ldots &  & 0 \\ 
\vdots & \ddots & \\ 0 & \\
F_1 & 0 & \ldots & 0\\
\end{bmatrix}.
\]
Notice that $ \block_2(F_2 F_1, \ldots, F_h F_{h-1}  ) = (\block_1(F_1, \ldots, F_h))^2 $. More generally, if we call $ F= \block_1(F_1, \ldots, F_h)$, then we have a compact way to express the powers of a matrix in block-shift form. For $ j>i$ denoting the matrix product  $ F_{i:j} = F_j F_{j-1}\ldots F_{i+1} F_i $, then $ F^k =  (\block_1(F_1, \ldots, F_h))^k = \block_k(F_{1:k}, \ldots, F_{h+1-k:h}) $, and $ F^h =  (\block_1(F_1, \ldots, F_h))^h = \block_h(F_{1:h})  $. 
The matrix $F$ is nilpotent of nilpotency index $ h+1$. 

If $ F=\block_1(F_1, \ldots, F_h) $ and $ G=\block_1(G_1, \ldots, G_h) $ have the same block sizes, then the first two ``mixed term'' products of order 1 and 2 in $G$ and of total order $h$ are
\begin{align}
F^{h-j} G F^{j-1} & = \block_h(F_h \ldots F_{j+1} G_{j} F_{j-1} \ldots F_1) = \block_h(F_{j+1:h}  G_{j} F_{1:j-1})
\label{eq:FGF} \\
F^{h-j-k} G F^{j-1} G F^{k-1} & = \block_h(F_h \ldots F_{j+k+1} G_{j+k} F_{j+k-1} \ldots F_{k+1} G_{k} F_{k-1} \ldots, F_1) \label{eq:FGFGF} \\
& = \block_h(F_{j+k+1:h}  G_{j+k} F_{k+1: j+k-1}G_{k} F_{1: k-1} ).
\nonumber 
\end{align}

\subsection{Gradient flow for deep linear networks: standard formulation}
\label{sec:grad-flow-basic}

Consider a set of input and output data, $\{\mathbf{x}_i\}^\nu_{i=1}$ and $\{\mathbf{y}_i\}^\nu_{i=1}$, represented in matrix form as $X=\begin{bmatrix}
    \mathbf{x}_1 & ... & \mathbf{x}_\nu
\end{bmatrix}\in\mathbb{R}^{d_x\times \nu}$, and $Y=\begin{bmatrix}
    \mathbf{y}_1 & ... & \mathbf{y}_\nu
\end{bmatrix}\in\mathbb{R}^{d_y\times \nu}$.
A {\em deep linear neural network} is a function $f:\mathbb{R}^{d_x}\rightarrow\mathbb{R}^{d_y}$ connecting the input layer (``layer 0'', with $ d_x$ nodes) to the output layer (``layer $h$'', of $ d_y $ nodes) through $ h-1$ hidden layers (with $ d_1, \ldots, d_{h-1}$ nodes). 
The function $ f$ is parameterized by a set of weight matrices $W_1,...,W_h$,  $W_i\in\mathbb{R}^{d_i\times d_{i-1}} $, $ i=1, \ldots, h$, where we conventionally identify $ d_0= d_x $ and $ d_h=d_y $.
The weight $ [W_i]_{jk}$ corresponds to the weighted edge connecting the $k$-th node of the $ (i-1)$-th layer with the $j$-th node of the $i$-th layer.

The  deep linear neural network can be written compactly as
\begin{equation}
\label{eq:prodW}
    f(X)=f(X;W_1,...,W_h) = W_h\cdots W_1 X = W_{1:h} X.
\end{equation}
The aim is to learn the matrices $W_i$ so that the squared loss
\begin{equation}\label{eq:orig_loss}
    \lVert Y-f(X)\rVert_F^2 = \sum_{i=1}^\nu \lVert \mathbf{y}_i-f(\mathbf{x}_i)\rVert_2^2
\end{equation}
is minimized.

Throughout the paper we make the following standard assumptions.
\begin{assumption}
\label{ass:size}
The regression problem is overdetermined: $ \nu \geq d_x $. $ X $ and $ Y$ have full row rank ($ d_x$ and $ d_y$ respectively). The input and hidden layers of the neural networks are large enough: $ d_i \geq d_y $, $i =x, 1, \ldots, h-1 $.
\end{assumption}
When $ h>1 $, $ f(X)$ is overparametrized, since it is just as expressive as a linear mapping $W\in\mathbb{R}^{d_y\times d_x}$, for which the closed-form solution minimizing $\lVert Y-WX\rVert_F^2$ is the least-squares (LS) solution: $W^\ast=YX^{\dagger}$, where $X^{\dagger}=X^\top(XX^\top)^{-1}\in\mathbb{R}^{\nu\times d_x}$ is the pseudo-inverse of $X$. When the least-squares problem is overdetermined (i.e., $ \nu \geq d_x $, as we assume here), the minimizer $W^\ast$ is unique. Nevertheless, deep linear neural networks exhibit interesting properties during learning, and can serve as a useful analytically-treatable testbed for the investigation of more complex nonlinear deep neural networks.

Under Assumption~\ref{ass:size}, following \cite{chitour2018geometric} (see Appendix~\ref{app:loss-func} for the details), we can equivalently formulate the loss function \eqref{eq:orig_loss} as
\begin{equation}\label{eq:loss}
    \mathcal{L}(W_1,...,W_h) = \frac{1}{2} \lVert \Sigma-W_{1:h}\rVert_F^2,
\end{equation}
where the inputs have been absorbed into the weight matrices, and 
\beq\label{eq:Sigma-diag}
\Sigma=\begin{bmatrix} 
\sigma_1 & 0 & \ldots & 0 & \ldots & 0 \\
 &  \ddots &  & \vdots  & & \vdots \\
 & &  \sigma_{d_y} & 0 & \ldots & 0
 \end{bmatrix}\in\mathbb{R}^{d_y\times d_x}
 \eeq
is a matrix of singular values computed from the data. 

The following assumption of further regularity in the data (generically verified) is also needed.
\begin{assumption}
\label{ass:distinct-sums}
The singular values $ \sigma_i $ are all nonzero and distinct: $ \sigma_i \neq \sigma_j$, $ \sigma_i, \sigma_j \neq 0$, $ i,j=1, \ldots, d_y $. 
Furthermore, all partial sums $ \sum_{j\in \mathcal{S}} \sigma_j^2 $, $ \mathcal{S} \subseteq \{ 1, \ldots, d_y \}$, are distinct. 
\end{assumption}

Under Assumptions~\ref{ass:size} and~\ref{ass:distinct-sums}, the least-squares solution of 
\beq
\mathcal{L}_{\rm LS} (W) = \frac{1}{2} \| \Sigma - W \|_F^2 
\label{eq:LS-W}
\eeq
is trivially $ W^\ast = \Sigma $. It is unique and s.t. $ \mathcal{L}_{\rm LS} (W^\ast ) =0$. 
When instead we consider the loss function~\eqref{eq:loss}, then the loss landscape is much more complex. A preliminary description is given in the following proposition.
A more detailed discussion will be given in Section~\ref{sec:conv-analysis}.

\begin{proposition}
\label{prop:LevelSurf}
Under Assumptions~\ref{ass:size} and~\ref{ass:distinct-sums}, the minimization problem 
\begin{equation}
\label{eq:minL1}
    \min_{W_1, \ldots, W_h } \mathcal{L}(W_1,...,W_h) 
\end{equation}
has an infinite number of optimal solutions, and none of them is isolated.
If $W_1^\ast ,...,W_h^\ast $ is one such optimal solution, then $\mathcal{L}(W_1^\ast ,W_2^\ast ,...,W_h^\ast)= 0$. 
Furthermore, the level sets of $ \mathcal{L}$ 
\[
\mathcal{M}_c = \mathcal{L}^{-1}(c) = \{ W_1, \ldots, W_h  \; \text{s.t.} \; \dot{\mathcal{L}}(W_1,...,W_h) =0, \; \mathcal{L}(W_1,...,W_h) = c \} , \quad c \in \mathbb{R}, \quad c\geq 0 
\]
are of dimension $ \geq \sum_{i=1}^{h-1} d_i^2 $ and unbounded. 
\end{proposition}
\begin{proof}
If $ W^\ast $ is the least-squares solution of \eqref{eq:LS-W} and $ W_1^\ast ,W_2^\ast ,...,W_h^\ast $ is an optimal solution of \eqref{eq:minL1}, from $ 0 \leq \mathcal{L}(W_1^\ast ,W_2^\ast ,...,W_h^\ast) \leq \mathcal{L}_{\rm LS} (W^\ast ) =0$, it follows that $\mathcal{L}(W_1^\ast ,W_2^\ast ,...,W_h^\ast)=0 $. If $ P_i \in \mathbb{R}^{d_i \times d_i } $, $i=1,\ldots, h-1$, are full rank square matrices corresponding to a change of basis at the nodes of the hidden layers, then also $ P_1 W_1^\ast ,P_2 W_2^\ast P_1^{-1} ,...,W_h^\ast P_{h-1}^{-1} $ is an optimal solution of \eqref{eq:minL1}. In particular, by continuity, no solution $ W_1^\ast ,W_2^\ast ,...,W_h^\ast $ is isolated, and the matrices $ P_i $ can be taken arbitrarily large in any norm, hence the level set $ \mathcal{M}_0 $ is unbounded. 
The only requirement on $ P_i $ is that $ \det (P_i) \neq 0$, which is a generic condition. Therefore $ \dim \left\{ P_i \in \mathbb{R}^{d_i \times d_i } \; \text{s.t.} \; \det (P_i)\neq 0 \right\} = d_i^2$, and $ \dim \left( \mathcal{M}_0 \right) \geq  \sum_{i=1}^{h-1} d_i^2$.
The same argument applies for any critical point of $ \mathcal{L}$ and any level surface $ \mathcal{M}_c $ with $ c\geq 0$.
\end{proof}

A standard method for solving \eqref{eq:minL1} is by gradient descent, in which all model parameters $\mathbf{W}=\{W_1,...,W_h\}$ are at each time step $t$ updated via
$\label{eq:descent}
\mathbf{W}(t+1)=\mathbf{W}(t)-\eta\nabla_\mathbf{W} \mathcal{L}
$, using some appropriate step size $\eta$. When the step size goes to 0, a continuous time formulation of gradient descent is obtained, denoted {\em gradient flow}, 
$\label{eq:flow}
\frac{d\mathbf{W}}{dt} = -\nabla_{\mathbf{W}}\mathcal{L}
$, 
which can be expressed block-wise as:
\begin{equation}\label{eq:block_grad}
    \frac{dW_j}{dt} = W_{j+1:h}^\top(\Sigma-W_{1:h})W_{1:j-1}^\top, \;\; j=1, \ldots, h,
\end{equation}
where $ W_{1:j-1}^\top = W_{1}^\top \dots W_{j-1}^\top$ and $ W_{1:0} =I_{d_x}$, see e.g.  \cite{saxe2013exact,arora2018convergence,chitour2018geometric} for a derivation.

\section{Network formulation of gradient flow}
\label{sec:adjacency}

Adopting a global perspective, a deep linear neural network can be represented as a directed acyclic graph $\mathcal{G}=(\mathcal{V},\mathcal{E}, A)$, where $\mathcal{V}$ is the set of $ n_v={\rm card}(\mathcal{ V})=d_x+d_1+...+d_{h-1}+d_y$ nodes, and the edge set $\mathcal{E}$ is defined by the feedforward architecture, with each edge $(k,j)\in \mathcal{E}$ having a weight equal to the corresponding value in the weighted adjacency matrix $A $. Itself, the weighted adjacency matrix $A$ can be written in block-shift form as

\begin{equation}
	A=\block_1 (W_1, W_2 \ldots, W_h) 
 \in \mathbb{R}^{n_v \times n_v}.
 \label{eq:A}
\end{equation}
In next sections the gradient flow equation \eqref{eq:block_grad} is expressed in terms of $A$, and its properties as a matrix ODE are analyzed.

\subsection{Adjacency matrix dynamics}

Denote $ \mathcal{A}$ the set of block-shift matrices \eqref{eq:A}: 
\[
\mathcal{A}= \{ A=\block_1 (W_1, W_2 \ldots, W_h) \in \mathbb{R}^{n_v \times n_v} \;\;\text{s.t.} \; \; W_i \in \mathbb{R}^{d_i \times d_{i-1} } \}.
\]
For us $ \mathcal{A} $ is the state space of the gradient flow, and it has $ d_x d_1 + d_1 d_2 + \ldots + d_{h-1} d_y $ degrees of freedom, corresponding to the weights $ [W_i]_{jk}$, $ i=1, \ldots, h$.

\begin{proposition}\label{prop:grad-A}
The loss function \eqref{eq:loss} can be written in terms of the adjacency matrix $ A \in \mathcal{A}$ as
\begin{equation}
\mathcal{L}(A) = \frac{1}{2} \left\| E- A^h \right\|_F^2 =  \frac{1}{2}  \tr\left( (E-A^h)(E-A^h)^\top \right),
\label{eq:lossA}
\end{equation}
and the gradient flow \eqref{eq:block_grad} as
\begin{equation}
	\dot{A}= - \nabla_A \mathcal{L} = \sum_{j=1}^h \left(A^{h-j}\right)^\top \left( E - A^h \right) \left( A^{j-1}
 \right) ^\top,
 \label{eq:grad-flowA}
\end{equation}
where 
\[
E= \block_h (\Sigma ) .
 \]    
\end{proposition}
 
\begin{proof}
First notice that, from Section~\ref{sec:block-shift}, the powers of the adjacency matrix $A$ are 
\[\label{eq:A2_Ak}
	A^2=\block_2(W_2 W_1, W_3 W_2 \ldots, W_h W_{h-1} )  
 \]
 and, iterating, 
 \beq
 \label{eq:Apowerh}
    A^h=\block_h (W_{1:h}) ,
\eeq
from which the expression for the loss function follows straightforwardly, observing that
\begin{equation*}
	E - A^{h}= \block_h ( \Sigma-W_{1:h} ).
\end{equation*}
The derivative of the loss function \eqref{eq:lossA} is
\begin{equation}
\label{eq:dotL}
\dot{\mathcal{L}}= 
\Bigl \langle \nabla_A \mathcal{L}  ,\,  \dot A \Bigr \rangle_F 
\end{equation}
where, using the denominator convention for differentials of scalars with respect to matrices \cite{Petersen2012matrix}
\begin{align*}
\nabla_A \mathcal{L}  = \frac{\partial \mathcal{L}(A) }{\partial A} & = 
- \frac{\partial }{\partial A} \bigl \langle E, \, A^h \bigr \rangle _F + \frac{1}{2} \frac{\partial }{\partial A} \left\| A^h \right\|_F^2 \\
& = -  \frac{\partial  }{\partial A} \tr\left( (A^h)^\top E \right) +  \frac{\partial  } {\partial A} \tr \left((A^h)^\top A^h \right)  \\
& = - \sum_{j=1}^h (A^{j-1})^\top  E (A^{h-j} )^\top
+  \sum_{j=1}^h (A^{j-1})^\top A^h (A^{h-j})^\top \\
& = - \sum_{j=1}^h (A^{j-1})^\top \left( E - A^h \right)( A^{h-j})^\top.
\end{align*}
Hence, choosing for $ \dot A$ the negated gradient direction $ - \frac{\partial \mathcal{L}(A) }{\partial A}  $, we get the gradient flow equation \eqref{eq:grad-flowA}.
\end{proof}

\begin{proposition}
\label{prop:A-block-struct}
The matrix $A(t)$ solving \eqref{eq:grad-flowA} is nilpotent and has the block structure \eqref{eq:A} for all $t$: $ A(0)\in \mathcal{A} \,\Longrightarrow \, A(t) \in \mathcal{A}$ $ \forall \, t>0$. Consequently, the gradient flow \eqref{eq:grad-flowA} is isospectral.
\end{proposition}

\begin{proof}
We want to show that we can write the dynamics of the blocks \eqref{eq:block_grad} in terms of $A$. Denoting $\Xi=\Sigma-W_{1:h}$, and rearranging the matrix ODEs \eqref{eq:block_grad} as blocks according to the construction of $A$ in \eqref{eq:A}:
\begin{equation}\label{eq:matr_ode_general}
\begin{split}
\frac{d}{dt} \block_1(W_1, \ldots, W_h) 
 = \block_1(W_{2:h}^\top \Xi,W_{3:h}^\top \Xi W_1^\top , W_{4:h}^\top \Xi W_{1:2}^\top,  \ldots ,  \Xi W_{1:h-1}^\top ). \\ 
\end{split}
\end{equation}
The left-hand side is clearly $ \dot A$. 
To show that indeed the right-hand side is equal to \eqref{eq:grad-flowA}, recall from Section~\ref{sec:block-shift} that the $ k$-th power of $A$ is 
\begin{equation}
    \label{eq:Apowerk}
	A^k=\block_k (W_{1:k}, W_{2:k+1}, \ldots , W_{h-k:h-1}, W_{h-k+1:h}) ,
\end{equation}
that is, $W_{\ell:k}=[A^{k-\ell+1}]_{(k+1,\ell)}$, or $W_{\ell:k}^\top=[A^{k-\ell+1}]^\top_{(\ell,k+1)}$, where sub-index is block-wise. The size of the blocks is preserved since $W_k$ remains the rightmost factor in the $k$-th block column, and $W_\ell$ remains the leftmost factor on the $(\ell+1)$-th block row. 
We can write \eqref{eq:block_grad} as
\begin{equation}\label{eq:new_block_ode}
	\frac{dW_j}{dt} = W_{j+1:h}^\top \Xi W_{1:j-1}^\top = [A^{h-j}]^\top_{(j+1,h+1)} \Xi  [A^{j-1}]^\top_{(1,j)} = \left[(A^{h-j})^\top (E - A^{h}) (A^{j-1})^\top\right]_{(j+1,j)},
\end{equation}
where we have used that $ E - A^h = \block_h(\Xi)$.
Combining \eqref{eq:matr_ode_general} and \eqref{eq:new_block_ode} we have
\begin{equation*}
	\left(\frac{dA}{dt}\right)_{(j+1,j)} = \frac{dW_j}{dt} = \left[(A^{h-j})^\top (E - A^{h}) (A^{j-1})^\top\right]_{(j+1,j)}.
\end{equation*}
The matrices $(A^{h-j})^\top$ and $(A^{j-1})^\top$ contain supra-diagonal blocks. Multiplying $(A^{h-j})^\top$ with $(E - A^{h})$ from the right leaves a single nonzero block at block position $(j+1,1)$. Now multiplying this product with $(A^{j-1})^\top$ from the right leaves a single nonzero block at block position $(j+1,j)$, while all other blocks are 0. 
Hence $A(t)$ and its time derivative have indeed identical block structure, implying that no new edges outside these common blocks can be created by the gradient flow ODE and hence that $ A(t) \in \mathcal{A} $ $ \forall \, t$.
It can easily be seen from the structure of $A$ that $A(t)$ is nilpotent with $(A(t))^{h+1}=0$. It follows from the previous arguments that $ A(t)$ stays nilpotent for all $t$, and therefore the isospectral property also follows: $ \lambda_i(A(t))=0$ for all $i$ and for all $t$. 
\end{proof}

\subsection{Conservation laws}
\label{sec:conserv_law}

The ODE \eqref{eq:grad-flowA} encodes a number of conservation laws.
Some occur trivially because $ A(t) $ is nilpotent: $ \lambda_i(A(t)) =0 $ for all $ i =1,\ldots, n_v $ and for all $ t$, and $ \tr(A(t)^k) =0 $ for all $  k\in \mathbb{N}$.
Other, less trivial, are known in the literature \cite{saxe2013exact,chitour2018geometric} and can be expressed in terms of $A$ and $A^\top$.
Denote \begin{equation}
\label{eq:Q}
Q(t) = [A(t), \, A(t)^\top ]= A(t) A(t)^\top - A(t)^\top A(t) 
\end{equation} 
and
\[
\mathcal{C} =\{ C \in \mathbb{R}^{n_v \times n_v }\;  \text{s.t.} \; C=  \block_0 (0_{d_x}, C_1, \ldots, C_{h-1} , 0_{d_y} ) \}
\]
where $ C_i = C_i^\top \in \mathbb{R}^{d_i \times d_i } $ are constant matrices.
Let also $ J= \block_0 (0_{d_x}, I_{d_1} \ldots  I_{d_{h-1}} , 0_{d_y} ) $.

\begin{proposition}
\label{prop:conserv-law-A}
For the gradient flow \eqref{eq:grad-flowA}, we have that $ J Q(t) = C$, for some constant $ C \in \mathcal{C} $ and for all $ t$. 
\end{proposition} 

\begin{proof}
In the proof we omit the dependence from $t$ for brevity.
Denoting $M=E-A^h$, and differentiating $ Q$, after some calculations one gets that in the following summation all terms cancel except those in $ A^h$
\begin{align*}
\dot{Q} 
& = \sum_{j=1}^h \left( (A^{h-j})^\top M (A^j)^\top - (A^{h-j+1})^\top M (A^{j-1})^\top + A^{j} M^\top  A^{h-j}-  A^{j-1} M^\top  A^{h-j+1} \right) \\
& = - (A^h)^\top  M + M (A^h )^\top + A^h M^\top - M^\top A^h \\
& = \block_0 \left( - (\Sigma - W_{1:h} )^\top W_{1:h} - W_{1:h}^\top (\Sigma - W_{1:h} ) , 0, \ldots, 0, ( \Sigma - W_{1:h} ) W_{1:h}^\top + W_{1:h} ( \Sigma - W_{1:h} )^\top \right), \nonumber 
\end{align*}
i.e., all intermediate blocks are vanishing except the first and the last, implying $ J \dot{Q}=0$. 
\end{proof}
By construction $ Q $ is symmetric: $ (A A^\top - A^\top A )^\top = A A^\top - A^\top A $.
Observing that 
\[
A A^\top =\block_0 (0, W_1 W_1^\top, \ldots,  W_h W_h^\top)
\]
and 
\[
A^\top A =\block_0 (W_1^\top W_1, \ldots,  W_h^\top W_h, 0)
\]
we have that $ Q $ is block-diagonal:
\[
Q = \block_0 (- W_1^\top W_1, W_1 W_1^\top - W_2^\top W_2 , \ldots,W_{h-1} W_{h-1}^\top - W_h^\top W_h ,W_h W_h^\top  ).
\]
All diagonal blocks of $ Q$, except the first and last, are the conservation laws considered in the literature: $ W_i W_i^\top - W_{i+1}^\top W_{i+1} = C_i = {\rm const}$ for all $t$, $ i=1,\ldots, h-1$, see \cite{chitour2018geometric}. 
Because of the symmetry in $C_i $, the number of such invariants of motion is at most $  \frac{1}{2} \left(  d_1(d_1+1) + \ldots + d_{h-1} (d_{h-1}+1)\right) $.

The expression \eqref{eq:Q} allows to interpret such conserved quantities as ``conservation of non-normality'' at the hidden layers of the deep linear network.

Notice that $ \dot{Q} $ can be expressed as the symmetric part of the matrix commutation $ [A^h, \, M^\top ] $ or as the symmetric part of $ [M, \, (A^h)^\top ]$:
\begin{align}
\dot{Q} 
& = \sum_{j=1}^h \left( (A^{h-j})^\top \,[ M, \, A^\top ]\, (A^{j-1})^\top + A^{j-1} \,[ A, \, M^\top ]\, A^{h-j} \right) \nonumber \\
&  = [A^h, \, M^\top ] + \left( [A^h, \, M^\top ] \right)^\top = [ M ,\,  (A^h )^\top ] +\left( [ M, \, (A^h )^\top] \right)^\top \nonumber \\
& =  [A^h, \, E^\top ] + \left( [A^h, \, E^\top ]\right)^\top  + 2  [A^h, \, (A^h)^\top ]. \label{eq:dotQ-comm}
\end{align}
As shown in the proof of Proposition~\ref{prop:conserv-law-A}, such terms belong to the first and last diagonal blocks of $ \dot{Q} $, so that $ J\dot{Q} =0$. 
Notice that $ J\dot{Q} =0$ corresponds to $ \langle J \pde{Q}{A} , \, \dot A \rangle_F =0$, i.e., $ J \pde{Q}{A}  $ and $ \dot A $ are orthogonal for all $ t$ and all $ A\in \mathcal{A}$. 
This can be reinterpreted geometrically as follows.
If instead of computing $ \dot Q $ we consider the differential $ d Q(A) $, then $ J dQ(A) $ corresponds to a matrix of one forms which are exact (as they are differentials of functions).
Hence the co-distribution $ \Delta_Q(A) ={\rm span} \{ J d Q(A) \}$ is integrable and orthogonal to the gradient flow: $ \langle \Delta_Q(A), \, \nabla_A \mathcal{L}(A) \rangle_F =0$ for all $ t$ and all $ A\in \mathcal{A}$.

\subsection{Hessian quadratic form}
For a matrix ODE like \eqref{eq:grad-flowA}, the Hessian is a 4-tensor. If we want to avoid working with such high-dimensional tensors, one alternative is to vectorize the ODE \cite{kawaguchi2016deep}.
The other alternative, which we follow in this paper, is to consider the Hessian quadratic form \cite{chitour2018geometric,achour2021loss}. 
Let us construct a Taylor expansion at a critical point $ A$ of \eqref{eq:grad-flowA} along the direction given by another $ V \in \mathcal{A}$.
Standard calculations lead to: 
\[
\mathcal{L}(A+ \epsilon V) = \mathcal{L}(A) - \epsilon \, \tr \left(\rho_1(A, V) (E - A^h)^\top \right) + \frac{\epsilon^2}{2}  \tr \left(\rho_1(A,V) \rho_1(A,V)^\top  - 2 \rho_2(A,V) (E - A^h)^\top \right) + o(\epsilon^3) 
\]
where 
\begin{align}
\rho_1(A,V) & = \sum_{j=1}^h A^{h-j} V A^{j-1} \label{eq:rho1}
\\
\rho_2(A,V) & = \sum^{h-2}_{ i, j, k =0 \atop i+j+k=h-2 
} A^i V A^j V A^k.
  \label{eq:rho2} 
\end{align}

Notice from \eqref{eq:rho1} that the first variation of $ \mathcal{L} $ along $ V$ is 
\[
\tr \left( \rho_1(A, V) (E - A^h)^\top \right) = 
\tr \left( \sum_{j=1}^h A^{j-1}  (E - A^h)^\top A^{h-j} V \right) =
\tr \left( \sum_{j=1}^h ( A^{h-j})^\top  (E - A^h) ( A^{j-1} )^\top V^\top \right) ,
\]
from which the gradient flow equation \eqref{eq:grad-flowA} follows right away (recall we use the denominator convention, which requires to consider the transpose in the direction to differentiate, i.e., $ V^\top$).
The Hessian quadratic form is given by the second variation terms:
\beq
\mathfrak{h}(A,V) =  \| \rho_1(A,V)\|_F^2 - 2 \tr \left( \rho_2(A,V) (E - A^h)^\top \right).
\label{eq:hess-quadr-form}
\eeq

\subsection{A block-shift singular value decomposition for $A$}
\label{sec:block-shift-svd}

The following proposition introduces a block-shift equivalent of the SVD to be used for the adjacency matrix $ A = \block_1(W_1, \ldots , W_h)\in \mathcal{A} $. It is obtained assembling the SVD at the individual blocks $ W_i $.

\begin{proposition}
\label{prop:svdA2}
Let $ A= \block_1(W_1, \ldots , W_h)\in \mathcal{A} $. If $ W_i = U_i \Lambda_i V_i^\top $ is the SVD of $W_i $, then $A = U \Lambda V^\top $, where $ \Lambda=\block_1( \Lambda_1, \ldots, \Lambda_h) $ plays the role of ``block-shift'' singular values of $A$, and $ U =\block_0(0, U_1, \ldots, U_h ) $ resp. $ V=\block_0(V_1, \ldots V_h, 0 ) $ are matrices containing the left resp. right singular vectors of the blocks $ W_i$. 
\end{proposition}

The proof is a straightforward calculation, after noting that $ A =\block_1 (U_1 \Lambda_1 V_1^\top, \ldots, U_h \Lambda_h V_h^\top ) $.

While the block-shift SVD described in Proposition~\ref{prop:svdA2} is an unconventional SVD because the singular values are in the lower diagonal blocks, it is possible to obtain also a ``standard'' SVD for $A$, see Appendix~\ref{sec:svd}.

\subsection{Equivariance of gradient flow dynamics to orthogonal transformations}

Denote
\[
\mathcal{P} = \{ P =\block_0 (I_{d_x},\, P_1, \, \ldots , P_{h-1} , I_{d_y}),  \text{  where  } P_i \in \mathbb{R}^{d_i \times d_i } , \;\; \det ( P_i)\neq 0 \}
\]
the set of linear changes of basis acting on the hidden nodes of $ A \in \mathcal{A}$.
Notice how $ \mathcal{P} $ is composed of $ d_{h-1}$ disconnected components of similarity transformations, one for each hidden layer.
Each block of weights $ W_i $ is affected by the changes of basis happening at its upstream and downstream hidden nodes. In fact, if $ A_j =\block_1( W_1^{(j)}, \ldots , W_h^{(j)})$, $ j=1,2$, then $ A_2 = P A_1 P^{-1} $ corresponds to $ W_i^{(2)}= P_i W_i^{(1)} P_{i-1}^{-1}$.

One can also ask when is the structure of the gradient flow system \eqref{eq:grad-flowA} compatible with a constant change of basis $P$. 
While this is not true in general, applying a change of basis $P \in \mathcal{P}$ such that each $ P_i $ is orthogonal does not alter the dynamics.
This is pointed out in \cite{saxe2019mathematical} (Suppl.), and shown explicitly below.
Denote $ \mathcal{P_O}  $ the subclass of orthogonal changes of basis in the hidden nodes: $ \mathcal{P_O} = \{ P \in \mathcal{P} \text{ s.t. }  P^{-1} = P^\top  \} $.

\begin{proposition}
\label{prop:constants-equival}
Consider the gradient flow system \eqref{eq:grad-flowA}. 
If $ A_1, \, A_2 \in \mathcal{A}$, are such that $ A_2 = P A_1 P^\top $ with $P\in \mathcal{P_O} $ constant, then $ \dot A_2 = P  \dot{ A }_1 P^\top $. Moreover, if $ A_2(0) = P A_1(0) P^\top $, then $ A_2(t) = P A_1(t) P^\top $ for all $t$.
\end{proposition}

\begin{proof}
    By construction, $ A^h= A^h P = A^h P^\top = P A^h = P^\top A^h = P^{-\top} A^h $ and similar relations hold for $E$ in place of $A^h$.
The change of basis with a constant orthogonal $P$, $ A_2 = P A_1 P^\top $, can therefore be expressed for instance as 
\begin{align*}
\dot A_2 = & \sum_{j=1}^h  P \left(A_1^{h-j}\right)^\top  P^{\top} P  \left( E - A_1^h \right) P^\top  P \left( A_1^{j-1} \right)^\top P^\top \\
= &  P  \sum_{j=1}^h \left(A_1^{h-j}\right)^\top \left( E - A_1^h \right)  \left( A_1^{j-1} \right)^\top P^\top \\
= &  P  \dot{ A }_1 P^\top
\end{align*}
where we have used $ P^\top P = I $.
To show that we get $ A_2(t) = P A_1(t) P^\top $ given $ A_2(0) = P A_1(0) P^\top $, we write
\begin{align*}
    A_2(t) & = A_2(0) + \int_0^t \dot{ A }_2(\tau) d\tau \\
    & = P A_1(0) P^\top + \int_0^t P  \dot{ A }_1( \tau ) P^\top d\tau \\
    & = P \left( A_1(0) + \int_0^t \dot{ A }_1(\tau) d\tau \right) P^\top = P A_1(t) P^\top .
\end{align*}
\end{proof}

Notice that equivariance transformations rotate also the conservation laws: $$ 
 Q_2 = A_2 A_2^\top - A_2^\top A_2 =  P ( A_1 A_1^\top - A_1^\top A_1 ) P^\top = P  Q_1 P^\top .$$

\subsection{Other properties of $A$}

\subsubsection{Frobenius norm of $A$: dynamics and equilibria}

Differentiating the Frobenius norm of $A$ we have the following.
\begin{proposition}
\label{prop:FrobA}
Along the solutions of \eqref{eq:grad-flowA}, the Frobenius norm of $A$ obeys to the following ODE:
\begin{equation}
    \label{eq:dotFrobA}
\frac{d}{dt} \| A\|_F^2  = 2h \, \bigl \langle E - A^h, A^h \bigr \rangle_F = 2h \left( \bigl \langle E , A^h \bigr \rangle_F - \| A^h\|_F^2 \right).
\end{equation}
$  \| A\|_F^2 $ stabilizes when $ E-A^h $ and $A^h $ are orthogonal in the Frobenius norm, or, equivalently, when the inner product $  \bigl \langle E , A^h \bigr \rangle_F $ equals $  \| A^h\|_F^2$.
Furthermore, the critical points of the gradient flow system \eqref{eq:grad-flowA} are critical points of  $  \| A\|_F $ but not necessarily vice versa.
\end{proposition}
\begin{proof}
A direct computation gives:
\begin{align*}
 \frac{d}{dt} \| A\|_F^2 & = \tr \left( (\dot{A})^\top A + A^\top \dot A \right) = 2 \tr \left( \sum_{j=1}^h (A^{j})^\top \left( E - A^h \right)( A^{h-j})^\top \right) \\
 & = 2 \tr \left( \sum_{j=1}^h \left( E - A^h \right)( A^{h})^\top \right) = 2h \, \tr \left( \left( E - A^h \right)( A^{h})^\top \right) .
\end{align*}
Concerning the final statement, if $A^\ast $ is an equilibrium point of \eqref{eq:grad-flowA}, then, from the previous expression, $   \frac{d}{dt} \left. \left( \| A\|_F^2 \right) \right|_{A=A^\ast} =0 $. The opposite implication is shown via a counterexample. Given any $\Sigma$ with $\sigma_i\neq\sigma_j$ for some $i\neq j$, we can set each weight to be proportional to the matrix of unit elements $W_i= a \, \mathbb{I}_{d_i, d_{i-1}}$, so that $W_{1:h}=a^hb \,\mathbb{I}_{d_y, d_x}$ for some $a>0$, and $b=d_1\cdots d_{h-1}$. Clearly this is not an equilibrium point of \eqref{eq:grad-flowA} since e.g., $A^{h-1}(E-A^h)\neq0$.
We get $\tr(W_{1:h}^\top W_{1:h})=a^{2h}b^2d_x d_y>0$, and $\tr( \Sigma^\top W_{1:h})=a^hb\lVert\Sigma\rVert_1 $, so by choosing  $ a = \left( \frac{ \lVert \Sigma \rVert_1 }{ b d_x d_y } \right)^\frac{ 1 }{ h }$ we obtain $\tr(W_{1:h}^\top W_{1:h})=\tr( \Sigma^\top W_{1:h})$, and hence $\frac{d}{dt} \| A\|_F^2=2h \left( \bigl \langle E , A^h \bigr \rangle - \| A^h\|_F^2 \right)=0$.
\end{proof}

What we see in simulation is that the norm $ \| A\|_F $ is typically increasing when the matrix $ A $ is initialized near the origin, and decreasing when instead $ A$ is initialized far away from the origin, even though it is not necessarily monotone, see Section~\ref{sec:ex-qualitat-anal}.

\subsubsection{Adding an $L_2 $ regularization}
The addition of an $ L_2 $ regularization term is expressed naturally in the adjacency matrix representation.
\begin{proposition}
Adding an $ L_2 $ regularization in the loss function \eqref{eq:lossA} 
\begin{equation}
\mathcal{L}^{\rm reg}(A) = \frac{1}{2} \left( \left\| E- A^h \right\|_F^2 + \gamma \| A\|_F^2 \right),
\label{eq:lossA-reg}
\end{equation}
where $ \gamma >0 $ is the regularization hyperparameter, leads to an extra linear negative term in the gradient flow \eqref{eq:grad-flowA}:
\begin{equation}
	\frac{dA}{dt}= - \nabla_A \mathcal{L} = \sum_{j=1}^h \left(A^{h-j}\right)^\top \left( E - A^h \right) \left( A^{j-1}
 \right) ^\top - \gamma A .
 \label{eq:grad-flowA-reg}
\end{equation}
\end{proposition}

\begin{proof}
Just differentiate the regularizing term: \[
\frac{1}{2} \frac{\partial \| A\|_F^2}{\partial A} = \frac{1}{2} \frac{\partial \tr ( A A^\top ) }{\partial A} = A
\]
and use Proposition~\ref{prop:grad-A}.
\end{proof}

\subsubsection{Numerical range of $A$}

Let the numerical range of $A$ be $ \mathcal{W}(A) = \{ x^\ast A x, \; x \in \mathbb{C}^{n_v} \; \text{s.t. } \; x^\ast x =1 \}$. Denote further $ A_{\rm sy} = \left( A+A^\top\right)/2 $ the symmetric part of $A $, and $ A_{\rm sk}  = \left( A-A^\top\right)/2 $ the skew symmetric part of $A $. 

\begin{proposition}
For the gradient flow \eqref{eq:grad-flowA} we have, for all $t$:
\begin{enumerate}
    \item The numerical range $ \mathcal{W}(A(t)) $ is a circular disk in $ \mathbb{C}$ centered at the origin.
    \item $ A_{\rm sy}(t) $ and $ A_{\rm sk}(t)$ have the same spectrum, modulo the imaginary unit: $ {\rm Sp}(A_{\rm sk}(t)) = i  {\rm Sp}(A_{\rm sy}(t)) $.
    \item The spectra of $ A_{\rm sy}(t) $ and of $ A_{\rm sk}(t) $ are symmetric with respect to the origin:\\ $ \lambda_i(A_{\rm sy}(t)) \in {\rm Sp}(A_{\rm sy} (t)) \; \Longrightarrow \; -\lambda_i (A_{\rm sy}(t)) \in {\rm Sp} (A_{\rm sy}(t))$, \\$ \lambda_i(A_{\rm sk}(t)) \in {\rm Sp}(A_{\rm sk}(t) ) \; \Longrightarrow \; -\lambda_i (A_{\rm sk}(t)) \in {\rm Sp} (A_{\rm sk}(t))$.
    \item $ \tr\left( (A(t)^k)^\top A(t) \right) =0 $ for all $  k \geq 2$.
\end{enumerate}
\end{proposition}
\begin{proof}
 Since $ A$ is in block-shift form, it follows from Theorem~1 of \cite{tam1994circularity} that its numerical range $ \mathcal{W}(A) $ is a circular disk in the complex plane centered in the origin. Since the block-shift structure is preserved for all $t$, $ \mathcal{W}(A(t)) $ is a circular disk for all $t$. In addition, the same theorem implies that the Hermitian part of $ e^{i \theta} A(t) $ has the same characteristic polynomial for all real $ \theta $, which in turn implies that both $A_{\rm sy} (t) $and  $ A_{\rm sk}(t) $ have the same spectrum, modulo the imaginary unit $i$.
Another consequence is that the spectrum of $A_{\rm sy}(t) $ and $ A_{\rm sk}(t)$ is symmetric:  see Corollary~1 of \cite{tam1994circularity}.
For the final statement see Corollary~2 of \cite{tam1994circularity}.
   
\end{proof}

\section{Convergence analysis}
\label{sec:conv-analysis}
In this section, we discuss dynamical properties of the gradient flow \eqref{eq:grad-flowA} and the structure of its critical points, reformulating known results from \cite{achour2021loss,kawaguchi2016deep,chitour2018geometric,yun2017global} in terms of the adjacency matrix $A \in \mathcal{A}$.

\subsection{A basic Lyapunov analysis}

A matrix $ A \in \mathcal{A} $ is a critical point of the loss function \eqref{eq:lossA} if $ \dot{\mathcal{L}}(A) =0$. It is a regular point otherwise. 
The following proposition provides a basic convergence analysis of the trajectories of the gradient flow.

\begin{proposition}
\label{prop:dotL}
Consider the gradient flow \eqref{eq:grad-flowA}.
\benu
\item The solution $A(t)$ of \eqref{eq:grad-flowA} exists and is bounded for all $t$ and for all initial conditions $ A(0) \in \mathcal{A}$.
\item $ A^\ast \in \mathcal{A} $ is an equilibrium point of \eqref{eq:grad-flowA} if and only if it is a critical point of $ \mathcal{L}$: $ \dot{\mathcal{L}}(A^\ast) =0 $.
\item Each solution $ A(t)$ converges to a critical point of $ \mathcal{L}$ when $ t\to \infty$, i.e., the $ \omega $-limit set of $A(t)$ is a single critical point of $ \mathcal{L}$. 
\eenu
\end{proposition}

\begin{proof}
\benu
\item 
From Proposition~\ref{prop:FrobA} and Lemma~\ref{lem:bounded} in Appendix~\ref{app:bounded}, if $ \gamma \in (0, 1) $ is a constant depending on $h$ and $ d_i$, we observe that
\begin{align}
	\frac{1}{2h}\frac{d}{dt} \lVert A \rVert^{2}_F &= \left\langle E,A^h\right\rangle_F - \lVert A^h \rVert^{2}_F \nonumber \\
	& \leq \lVert E \rVert_F \lVert A^h \rVert_F - \gamma\lVert A \rVert^{2h}_F + p\left(\lVert A \rVert^{2}_F\right) \nonumber \\
	& \leq - \gamma\lVert A \rVert^{2h}_F + \bar{p}\left(\lVert A \rVert^{2}_F\right), 
\end{align}
where $ p(\cdot) $ and $\bar{p}(\cdot) $ are polynomials of degree $\leq h-1$, leading to the conclusion that the trajectories of \eqref{eq:grad-flowA} are bounded.
\item The statement follows from the definition of gradient system, see \eqref{eq:dotL3}. 
\item The statement follows from the so-called Lojasiewicz's theorem (\cite{bah2022learning}, Theorem~3.1, or \cite{chitour2018geometric}, Theorem~2.2). 
Since $ \mathcal{L}$ is polynomial in $A$, it is a real analytic function. In addition, as the gradient system \eqref{eq:grad-flowA} has bounded trajectories, all its trajectories are confined into (sufficiently large) compact sets. It follows from Lojasiewicz's theorem that for each trajectory the $ \omega$-limit set is a single point and a critical point of $ \mathcal{L}$. 
 \eenu
\end{proof}

The loss function \eqref{eq:lossA} is a Lyapunov function for the gradient flow \eqref{eq:grad-flowA}, although only a positive semidefinite one. 

\begin{proposition}
\label{prop:Layp1}
Consider the gradient flow system \eqref{eq:grad-flowA}.
\benu
\item The loss function \eqref{eq:lossA} is positive semidefinite: $ \mathcal{L}(A) \geq 0 $ $ \forall \, A \in \mathcal{A}$. 
\item The derivative of the loss function \eqref{eq:lossA} along  \eqref{eq:grad-flowA} is negative semidefinite
\begin{equation}
\label{eq:dotL2}
\dot{\mathcal{L}} 
 = -  \sum_{j=1}^h \left\| \left(A^{h-j}\right)^\top \left( E - A^h \right) \left( A^{j-1}  \right) ^\top \right\|_F^2  \leq 0 \quad \forall \, A \in \mathcal{A}.
\end{equation}
\eenu
\end{proposition}

\begin{proof}
\benu
\item From Proposition~\ref{prop:LevelSurf}, we know that the minimization problem \eqref{eq:minL1} has infinitely-many optimal solutions and that none of them is isolated. Expressed in terms of optimal $A^\ast $, we have that there exists infinitely-many non-isolated $ A^\ast \in \mathcal{A}$ for which $ \mathcal{L}(A^\ast) =0 $, hence $ \mathcal{L}$ is only positive semidefinite. 
\item By construction, for the gradient flow it is, from \eqref{eq:dotL}, 
\begin{align}
\label{eq:dotL3}
\dot{\mathcal{L}}  & = - \left\| \frac{\partial \mathcal{L}(A) }{\partial A}  \right\|_F^2
=- \| \dot A  \|_F^2 
= - \tr \left((\dot{A}) ^\top \dot{A} \right) \\
& =  - \tr \left( \left( \sum_{j=1}^h \left(A^{h-j}\right)^\top \left( E - A^h \right) \left( A^{j-1}  \right) ^\top \right)^\top \sum_{j=1}^h \left(A^{h-j}\right)^\top \left( E - A^h \right) \left( A^{j-1}  \right) ^\top  \right)
\label{eq:dotL4} \\
& = -  \sum_{j=1}^h \tr \left( \left( \left(A^{h-j}\right)^\top \left( E - A^h \right) \left( A^{j-1}  \right) ^\top  \right) ^\top \left(A^{h-j}\right)^\top \left( E - A^h \right) \left( A^{j-1}  \right) ^\top   \right).
\label{eq:dotL5}
\end{align}
$\dot{\mathcal{L}}$ is negative semidefinite because it is the negation of a norm.
To show that \eqref{eq:dotL4} is equal to \eqref{eq:dotL5}, recall from the proof of Proposition~\ref{prop:A-block-struct} that the term $ 
\left(A^{h-j}\right)^\top \left( E - A^h \right) \left( A^{j-1}  \right) ^\top $ corresponds to a single nonzero block at position $ (j+1,j)$ in the expression for $A$.
Hence, when multiplied with $ (\dot{A}) ^\top $ as in \eqref{eq:dotL4}, there will be a single nonzero term, corresponding to a diagonal block in position $ (j+1, j+1)$, equal to 
$
W_{1:j-1}(\Sigma-W_{1:h})^\top W_{j+1:h} W_{j+1:h}^\top(\Sigma-W_{1:h})W_{1:j-1}^\top .
$
Summing over $j$ and taking trace, the result follows. 
Each term inside the trace in \eqref{eq:dotL5} is a Gram matrix, hence positive semidefinite.
 \eenu
\end{proof}

\subsection{Critical points of $ \mathcal{L}$: a survey of results}

The previous propositions are enough to characterize the asymptotic behavior of the gradient flow system \eqref{eq:grad-flowA}: each trajectory converges ``pointwise'' to a single equilibrium point which is a critical point of $ \mathcal{L}$.
The landscape of such critical points has been thoroughly investigated in the literature \cite{baldi1989neural,kawaguchi2016deep,yun2017global,zhou2017critical,chitour2018geometric,achour2021loss,trager2019pure}. We report here the main results (without proofs), reformulated in our setting and notations.

Recall that a critical point which is not a local minimizer or maximizer is called a saddle point.

\begin{proposition} 
\label{prop:critical-point-lit}
Consider the gradient flow system \eqref{eq:grad-flowA} and a critical point $ A^\ast \in \mathcal{A}$.
\benu
\item From \cite{kawaguchi2016deep}, Theorem~2.3:
\begin{enumerate}
\item $ \mathcal{L} $ is non-convex and non-concave;
\item Every local minimum is a global minimum;
\item Every critical point that is not a global minimum is a saddle point.
\end{enumerate}
\item From \cite{yun2017global}, Theorem~2.1:
\begin{enumerate}
\item Every critical point $ A^\ast $ s.t. $ \rank((A^\ast)^h) =d_y $ is a global minimum. \item Every critical point $ A^\ast $ s.t. $ \rank((A^\ast)^h) < d_y $ is instead a saddle point.
\end{enumerate}
\item From \cite{achour2021loss}, Propositions~1 and 2:
\benu
\item If $ A^\ast$ 
is s.t. $ \rank((A^\ast)^h) = r $, then there exists a unique $ \mathbf{s}=\begin{bmatrix} s_1 & \ldots & s_{d_y} \end{bmatrix}^\top $, $ s_i \in \{0, 1 \}$, $ \1^\top \mathbf{s} = r $ such that $ (A^\ast)^h =\block_h(S \Sigma) $, where $ S={\rm diag}(\mathbf{s})$. 
The associated critical value is 
\beq
\mathcal{L}(A^\ast) = \frac{1}{2} \sum_{i=1}^{d_y} (1- s_i) \sigma_i^2.
\label{eq:critical-values}
\eeq
\item Conversely, for any $ r \in \{ 1, \ldots, d_y \}$ 
and any $ \mathbf{s}=\begin{bmatrix} s_1 & \ldots & s_{d_y} \end{bmatrix}^\top $, $ s_i \in \{0, 1 \}$, $   \1^\top \mathbf{s} = r $, there exists a critical point $ A^\ast \in \mathcal{A} $ such that $ (A^\ast)^h =\block_h(S \Sigma) $, where $ S={\rm diag}(\mathbf{s})$. 
\eenu
\eenu

\end{proposition}
Notice that the set of  matrices $ A\in\mathcal{A} $ s.t. $ \rank(A^h) =d_y $ is dense in $ \mathcal{A}$, meaning that $ A\in\mathcal{A} $ s.t. $ \rank(A^h) <d_y $ occupy only a set of Lebesgue measure 0 in $\mathcal{A}$ \cite{achour2021loss}.
Nevertheless, there are critical points $ A^\ast $ in which $ \rank((A^\ast)^h) <d_y $: a standard example is $ A^\ast=0$, which according to the previous proposition is a saddle point ($A^\ast=0$ is not an optimal solution: $ \mathcal{L}(0) = \| \Sigma \|_F^2 >0$). 

Notice further the simple expression that the product $ A^h = \block_h(W_{1:h})$ assumes at critical points when $ \Sigma $ has the structure \eqref{eq:Sigma-diag}: both $ \Sigma $ and $ W_{1:h}$ are diagonal, with $ W_{1:h}$ omitting some diagonal entries $ \sigma_j $ when it is a saddle point. 
In particular, all critical points can be identified with the signature vectors $ \mathbf{s}=\begin{bmatrix} s_1 & \ldots & s_{d_y} \end{bmatrix}^\top $, $ s_i \in \{0, 1 \}$, or, equivalently, with the associated indicator set $ \mathcal{S} = \{ i \; \text{s.t.}\; s_i=1, \, i=1, \ldots, d_y\}$.

\subsection{Invariant sets and an equivalence relation}
\label{sec:invar-equiv}

Let $ \mathcal{M}$ be the set of all critical points of \eqref{eq:grad-flowA}: $ \mathcal{M}=\{ A \in \mathcal{A} \text{  s.t.  } \dot{\mathcal{L}} (A) =0 \}$.
$ \mathcal{M} $ foliates into the level sets $ \mathcal{M}_c = \{ A \in \mathcal{M}  \text{ s.t.  } \mathcal{L}(A) =c \} $, $ c\geq 0$, with $ \mathcal{M}_0$ being the set of global minimizers of \eqref{eq:lossA}.

LaSalle's invariance principle states that for a dynamical system endowed with a function $ \mathcal{L}$  s.t. $ \dot{\mathcal{L}} \leq 0$, all bounded solutions converge to the largest invariant set in $ \mathcal{M}$.
The argument does not require $ \mathcal{L} $ to be positive definite \cite{khalil2002nonlinear}.
Gradient systems like \eqref{eq:grad-flowA} are however special: the trajectories cross the level surfaces $ \mathcal{M}_c $ orthogonally (see Theorem~2, p. 201 in \cite{hirsch1974differential}), meaning that the solution of \eqref{eq:grad-flowA} is never ``sliding'' inside $ \mathcal{M}_c $ for any $ c $.
When Assumption~\ref{ass:distinct-sums} holds, this is enough to guarantee invariance of each $ \mathcal{M}_c $ as well as to investigate the stability character of its critical points from a dynamical perspective.

\begin{proposition}
Under Assumptions~\ref{ass:size} and~\ref{ass:distinct-sums}, the gradient system \eqref{eq:grad-flowA} admits $ 2^{d_y} $ disjoint level sets of critical points $ \mathcal{M}_c $ with critical values $ c= \frac{1}{2}\sum_{i=1}^{d_y} (1-s_i) \sigma_i^2 $, $ s_i \in \{0, 1 \}$, $ i=1, \ldots, d_y$.
Each $ \mathcal{M}_c $ is an unbounded invariant set of \eqref{eq:grad-flowA}. For $ c>0$, each $ A^\ast \in \mathcal{M}_c$ is a saddle point, while for $ c=0$, $ A^\ast \in \mathcal{M}_0$ is stable but not  asymptotically stable. 
\end{proposition}
\begin{proof}
From Proposition~\ref{prop:critical-point-lit}, the critical value $ c$ can assume only $ 2^{d_y} $ values $ \frac{1}{2}\sum_{i=1}^{d_y} (1-s_i) \sigma_i^2 $, $ s_i \in \{0, 1 \}$, $ i=1, \ldots, d_y$.
From Proposition~\ref{prop:dotL}, $ A^\ast \in \mathcal{M}_c $ corresponds to $ \left. \dot A \right|_{A=A^\ast} =0$, implying that $ \mathcal{M}_c $ is forward invariant. 
Unboundedness of the level surfaces $ \mathcal{M}_c $ is shown in Proposition~\ref{prop:LevelSurf}. 
For $ A^\ast \in \mathcal{M}_c$, when $ \mathcal{L}(A^\ast)>0$, then at least one of the $ s_i $ in \eqref{eq:critical-values} is equal to 0, hence $ \rank((A^\ast)^h)<d_y $, meaning, from Proposition~\ref{prop:critical-point-lit}, that $A^\ast$ is a saddle point. 
When $ c=0$, then $ \rank((A^\ast)^h) =d_y$, hence $ A^\ast $ is a global minimum. 
By Assumption~\ref{ass:distinct-sums}, the critical values \eqref{eq:critical-values} are all distinct.
From the analyticity of the gradient flow, it follows that the invariant sets $ \mathcal{M}_c $ must be disjoint.
Perturbing $ A^\ast \in \mathcal{M}_0$ locally, since $ \dot{\mathcal{L}}\leq 0 $ and the invariant manifolds $ \mathcal{M}_c $ are all disjoint in $ \mathcal{A}$, at any regular point $ A$ near $ \mathcal{M}_0$ the gradient flow can only decrease, returning towards $ \mathcal{M}_0$, hence $ A^\ast $ is Lyapunov stable. 
In Proposition~\ref{prop:LevelSurf} it is shown that no $ A^\ast  \in \mathcal{M}_0 $ is isolated, which implies that no $ A^\ast \in \mathcal{M}_0  $ can be asymptotically stable. 
\end{proof}

The $ 2^{d_y}$ invariant sets $ \mathcal{M}_c $ associated to the $ 2^{d_y}$ critical values of $ \mathcal{L}$ in Proposition~\ref{prop:critical-point-lit} are characterized by the fact that all entries of $ A^h $ are fixed univocally. For instance $ A^\ast \in \mathcal{M}_0 $ must correspond to $ (A^\ast)^h = \block_h(\Sigma)$. 
However, uniqueness of $ (A^\ast)^h $ does not imply uniqueness of $A^\ast $ (i.e., uniqueness of $ W_{1:h}^\ast $ does not imply uniqueness of $ W_1^\ast, \ldots ,  W_h^\ast$).

To investigate the issue, the procedure we  follow is analogous to the one introduced in \cite{trager2019pure}: the loss function map $ \mathcal{L} $ can be decomposed as a concatenation of two maps:
\[
\begin{array}{rccccc}
\mathcal{L}  : & \mathcal{A} & \stackrel{\phi} {\longrightarrow}  & \mathcal{A}_\phi & \stackrel{\left. \psi \right|_{\mathcal{A}_\phi}} {\longrightarrow} & \mathbb{R}^+\\
& A & \longmapsto & \phi(A)=A^h & \longmapsto & \|E - \phi(A) \|_F^2 
\end{array} .
\]
As $ \min_{i=1,\ldots, h-1} d_i \geq d_y $ and $ d_x \geq d_y$, the only case we are considering is the one which in \cite{trager2019pure} is denoted ``filling'', i.e., $ \phi $ is surjective in $ \mathbb{R}^{d_y \times d_x}$.  
However, while $ \psi$ is convex, $ \phi $ is not, and this makes the composite map $ \mathcal{L} = \psi \circ \phi $ non-convex.

In order to investigate the dynamics of the gradient flow, it is convenient to see $ \mathcal{A}_\phi $ as the quotient space of an equivalence relation induced by the map $  \phi $, denoted ``$ \sim $'': for $ A_1, \,  A_2 \in \mathcal{A} $, it is $ A_1 \sim A_2 $ if  $\phi(A_1)  = \phi(A_2) $  (i.e. $ A_1^h = A_2^h $).
The operation ``$ \sim$'' is trivially an equivalence relation, and as such it provides a partition of $ \mathcal{A}$ into disjoint equivalence classes: for each $ \Phi \in \mathcal{A}_\phi$ there is an associated fiber $ \phi^{-1}(\Phi)\subset \mathcal{A} $.
We want to give a compact description of the equivalence classes of ``$ \sim $''. For a given $ A_1 \in \mathcal{A} $, $A_2 \sim A_1 $ can originate from two different factors, possibly acting jointly:
\benu
\item change of basis in the hidden nodes via a matrix $ P \in \mathcal{P}$;
\item addition to $ A_1 $ of terms $ Z \in \mathcal{A} $ of nilpotency index $ \leq h$ mapped in the null space of some power $ A^k $, $ k\leq h $.
\eenu

The next proposition is a special case of a more general sufficient condition for equivalence treated in Proposition~\ref{prop:equival-classes-power} in the Appendix. 
\begin{proposition}
\label{prop:equival-classes-A}
Consider $ A_i\in \mathcal{A}$, $i=1,\, 2$.
Then $ A_1 \sim A_2 $ if $ A_2 = P (A_1 + Z ) P^{-1} $, where $ P \in \mathcal{P} $ and $ Z \in \mathcal{A} $ s.t. $ A_1 Z = Z A_1 =0 $, plus $ Z^h =0$.
\end{proposition}

Obtaining a necessary and sufficient condition for $ A_1 \sim A_2 $ appears a hard problem. Nevertheless, when we focus on the critical points, as in next section, the condition provided in Proposition~\ref{prop:equival-classes-A} is useful in obtaining a thorough description.

The major advantage in considering the equivalence relation $ \sim$ is that the quotient space $ \mathcal{A}_\phi $ inherits all critical points of $ \mathcal{A}$ with the same stability properties, but these critical points are isolated in $ \mathcal{A}_\phi$. 

\begin{proposition} 
Consider the gradient flow system \eqref{eq:grad-flowA}. 
Under Assumptions~\ref{ass:size} and~\ref{ass:distinct-sums},
\benu
\item Each invariant set $ \mathcal{M}_c $ is mapped by $ \phi $ into a single point $ \Phi_c^\ast = \phi(A^\ast)$, $ \forall \; A^\ast \in \mathcal{M}_c $. The set $ \mathcal{A}_\phi $ has $ 2^{d_y} $ critical points $ \Phi_c^\ast$, one for each invariant set $ \mathcal{M}_c $, of critical value given in \eqref{eq:critical-values}. 
\item Each critical point $ \Phi_c^\ast$ is isolated in $ \mathcal{A}_\phi $. 
\item Of these critical points $ \Phi_c^\ast$, $ 2^{d_y} - 1 $ are saddle points and correspond to signature vectors $ \mathbf{s} \neq \1$ in \eqref{eq:critical-values}, while $ \mathbf{s}=\1 $ corresponds to $ \Phi_0^\ast = \phi(A_0^\ast ) $ with $ \mathcal{L}(A_0^\ast )=0$, and is a global minimum. 

\eenu\end{proposition}

\begin{proof}
Let $ \Phi_k^\ast = \phi(A_k^\ast)$. 
By construction, if $ A_1^\ast, \, A_2^\ast \in \mathcal{M}_c $, then $ \Phi_1^\ast = \Phi_2^\ast$, i.e., the two critical points are mapped to the same value in $ \mathcal{A}_\phi$. 
Since this is true $ \forall \, A^\ast \in \mathcal{M}_c $, the first statement follows. 
When Assumption~\ref{ass:distinct-sums} holds, all critical values \eqref{eq:critical-values} are distinct and all $ \mathcal{M}_c $ are disjoint, hence each $ \Phi^\ast $ must be isolated in $ \mathcal{A}_\phi$, and the $ 2^{d_y}$ critical points of $ \mathcal{L}$ become isolated in $ \mathcal{A}_\phi$.
The stability character of $ \Phi_c^\ast = \phi(A_c^\ast) $ is inherited from $ A_c^\ast $.
In fact, by construction all $ \Phi_c^\ast $ s.t. $ \mathbf{s}\neq \1 $ correspond to $ \rank (\Phi_c^\ast) < d_y $.
\end{proof}

\subsection{Parametrization of all critical points of $ \mathcal{L}$}
\label{sec:param-crit-p}
In this section we follow \cite{achour2021loss}. 
Consider  $ \Phi^\ast \in \mathcal{A}_\phi $ associated to a critical point $ A^\ast \in \mathcal{A}$ (i.e., s.t. $ \phi(A^\ast )= \Phi^\ast$), of signature $ \mathbf{s} =\begin{bmatrix} s_1 & \ldots & s_{d_y} \end{bmatrix}^\top $. Denote $ \mathcal{S} = \{ i \text{ s.t. } s_i =1, \;  s_i \in \mathbf{s} \} $ of $ {\rm card}(\mathcal{S})=r \leq d_y $, and $ \mathcal{Q}=\{1, \ldots, d_y\}\setminus \mathcal{S}$. Denote further $ U_\mathcal{S}\in \mathbb{R}^{d_y \times r }  $ (resp. $ U_\mathcal{Q} \in \mathbb{R}^{d_y \times (d_y -r) } $) the subset of columns of $ I_{d_y} $ with indices in $ \mathcal{S} $ (resp. in $ \mathcal{Q}$). 
Then $ \Phi^\ast = \block_h (U_\mathcal{S} U_\mathcal{S}^\top \Sigma ) $, and there exists an infinite number of critical points in $ \phi^{-1}(\Phi^\ast) $.
The parametrization provided in \cite{achour2021loss} of such critical points can be rewritten in our notation as follows. 
Start from a ``near-diagonal'' representation $ A_1 \in \phi^{-1}(\Phi^\ast) $ expressed as 
\beq
A_1 = \block_1 \left( \begin{bmatrix} U_\mathcal{S}^\top \Sigma \\ 0 \end{bmatrix}, \begin{bmatrix} I_r & 0 \\ 0 & 0 \end{bmatrix} , \ldots  \begin{bmatrix} I_r & 0 \\ 0 & 0 \end{bmatrix} ,  \begin{bmatrix}  U_\mathcal{S} & 0  \end{bmatrix} \right),
\label{eq:A-param1}
\eeq
where the matrices of zeros have suitable dimensions (see $Z$ below).
The description of the fiber $ \phi^{-1}(\Phi^\ast) $ is given by $ A^\ast = P (A_1 + Z )P^{-1}$ with $ P\in \mathcal{P}$ and 
\beq
Z  = \block_1 \left( \begin{bmatrix} 0 \\ Z_1 \end{bmatrix}, \begin{bmatrix} 0 & 0 \\ 0 & Z_2 \end{bmatrix} , \ldots  \begin{bmatrix} 0 & 0 \\ 0 & Z_{h-1} \end{bmatrix} ,  \begin{bmatrix} 0 & U_\mathcal{Q} Z_h  \end{bmatrix} \right),
\label{eq:Z-param1}
\eeq
where the dimensions of the $Z_i $ blocks are 
$ Z_1 \in \mathbb{R}^{(d_1-r)\times d_x}$, 
$ Z_i \in \mathbb{R}^{(d_i-r) \times (d_{i-1}-r) }$, $ i=2,\ldots, h-1$,  
$ Z_h \in \mathbb{R}^{(d_y-r) \times (d_{h-1}-r) }$, and some of the $ Z_i $ blocks ($i=1, \ldots, h$) may be empty. 
In \cite{achour2021loss} there is a slight difference between the necessary condition (given in Proposition 9 of \cite{achour2021loss}) and the sufficient condition (given in Proposition~10 of \cite{achour2021loss}) for $A^\ast $ to be a critical point: both are based on the parametrization~\eqref{eq:A-param1}-\eqref{eq:Z-param1}, but sufficiency requires $ Z_i =0 $ and $ Z_j =0 $ for some $ i\neq j$.

Recall that a saddle point is said strict if its associated Hessian has at least a negative eigenvalue. 
It is said non-strict if all the eigenvalues of the Hessian are nonnegative. 
From Theorem~7 of \cite{achour2021loss}, we have the following classification.

\begin{proposition}
\label{prop:class-crit-points}
Consider the critical points of the gradient flow system \eqref{eq:grad-flowA}, represented as $ A^\ast = P (A_1 + Z )P^{-1}$, with $ A_1 $ and $ Z$ given in \eqref{eq:A-param1} and \eqref{eq:Z-param1}. Then $A^\ast $ can be split into:
\benu
\item Global minima: $ {\rm card}(\mathcal{S})= r =d_y $ (i.e., $\mathbf{s}=\1$) $  \Longrightarrow \;  U_\mathcal{S} =I_{d_y} , \; U_\mathcal{Q}=0$ (Proposition~11 of \cite{achour2021loss});
\item Saddle points: $ {\rm card}(\mathcal{S})=r<d_y $ (i.e., $ \mathbf{s}\neq \1 $), plus $ Z_i =0 $ and $ Z_j =0 $ for some $ i\neq j $ (Proposition~10 in \cite{achour2021loss}). Two cases are possible:
\benu
\item If $ \mathcal{S} \neq \{ 1, \ldots, r\} $ $ \Longrightarrow $ any $ A^\ast \in \phi^{-1}(\Phi^\ast) $ is a strict saddle point; 

\item If $ \mathcal{S} = \{ 1, \ldots, r\} $ $ \Longrightarrow $ the saddle point is an implicit regularizer for rank-$r$ matrices: $ \Phi^\ast = \block_h(U_\mathcal{S} U_\mathcal{S}^\top \Sigma) = {\rm argmin}_{R\in \mathbb{R}^{d_y \times d_x} \atop \rank(R)\leq r } \| R - \Sigma\|_F^2 $. The saddle point $A^\ast \in \phi^{-1}(\Phi^\ast) $ can be strict or non-strict,
according to ``tightenedness'' of $A^\ast$ \cite{achour2021loss}:
\benu
\item if $ A^\ast $ not tightened then $ A^\ast $ is a strict saddle point;
\item if $ A^\ast $ tightened then $ A^\ast $ is a non-strict saddle point.
\eenu
\eenu
\eenu
\end{proposition}
The concept of ``tightenedness'' of $A^\ast $ mentioned in Proposition~\ref{prop:class-crit-points} is given in \cite{achour2021loss} and has to do with the presence and location in $A^\ast $ of blocks of rank exactly equal to $r$. A proper definition is not needed here. It is enough to recall that a sufficient condition for a critical point $A^\ast$ to be tightened is that at least 3 of the blocks of $ A^\ast $ have rank $r$ (and all other blocks rank $ \geq r $). In terms of \eqref{eq:A-param1}-\eqref{eq:Z-param1}, this implies that at least 3 of the $ Z_i $ blocks in \eqref{eq:Z-param1} are equal to $0$.
Notice that $A_1 $ in \eqref{eq:A-param1} is always tightened, hence always leading to a non-strict saddle point when $  \mathcal{S} = \{ 1, \ldots, r\} $ with $ r<d_y$.

\subsection{Using the Hessian quadratic form}
When $A$ is a matrix, a critical point is strict if we can find a variation $ V \in \mathcal{A} $ in 
which the Hessian quadratic form $\mathfrak{h}$ is negative: $ \mathfrak{h}(A,V)<0$. 
It is non-strict if $ \mathfrak{h}(A,V)\geq 0 $ for all $ V \in \mathcal{A}$.
In this section we review some of the calculations of \cite{achour2021loss} in these two cases, as a way to show how our formalism can help reducing the notation and computational burden.

The following lemma is useful to establish that a change of basis in the hidden nodes does not change the nature of a critical point, akin to Lemma 16 of \cite{achour2021loss}. 
\begin{lemma} \label{lem:equivariant_rho}
    For all $A, V\in \mathcal{ A} $, $P \in \mathcal{P}$, and $k=1,...,h$, it is $ \rho_k(PAP^{-1},V) = \rho_k(A,P^{-1} V P)$.
\end{lemma}
\begin{proof}
    We have
\begin{align}
	\rho_k(A,V) & = \sum_{k_1+...+k_h = h-k \atop k \text{ factors } V} A^{k_1} V A^{k_2} V \cdots A^{k_{h-1}} V A^{k_h} \label{eq:rho_k_A} \\
	\rho_k(PAP^{-1},V) & = \sum_{k_1+...+k_h = h-k \atop k \text{ factors } V} P A^{k_1} P^{-1} V P A^{k_2} P^{-1} V \cdots P A^{k_{h-1}} P^{-1} V P A^{k_h} P^{-1} .\label{eq:rho_k_PAP}
\end{align}
From \eqref{eq:FGF}, \eqref{eq:FGFGF} and similar expressions when the number of factors $V$ is equal to $k$,  $\rho_k(PAP^{-1},V)$ is invariant to multiplication by $P$ and $P^{-1}$ from both left and right, as the first and last blocks of $P\in \mathcal{ P}$ are identities. By direct identification, then, $\rho_k(PAP^{-1},V) = \rho_k(A,P^{-1}VP)$.
\end{proof}

We continue by establishing the rather straightforward fact that each critical point of \eqref{eq:grad-flowA} must correspond to a ``degenerate'' Hessian, i.e., that the Hessian quadratic form must be vanishing along some nontrivial directions because the critical points are not isolated. 

\begin{proposition}
For any critical point $ A \in \mathcal{A} $ of\eqref{eq:grad-flowA} there exist variations $ V \in \mathcal{A}$ s.t. $ \mathfrak{h}(A,V)=0$. 
\end{proposition}

\begin{proof}
Consider the parametrization \eqref{eq:A-param1}-\eqref{eq:Z-param1} for critical points $ A = P ( A_1 + Z ) P^{-1} $. In the special case $ {\rm card}(\mathcal{S}) = 0 $, where $ A^h = 0 $, then $ \mathcal{R}(W_{1:i}) \subseteq \mathcal{N}(W_{i+1}) $ and $ W_{1:i} \neq 0 $ for some $ i = 1,...,h-1 $ (except if $ A = 0 $, but then it is trivial to find a direction $ V $). We can choose $ V=\block_1(0,...,0,V_i,0,...,0)$ where $ V_i $ has some columns from $ W_{1:i} $ so that $ A V = 0 $. This implies $ \rho_1(A,V) = \rho_2(A,V) = 0$ and therefore $ \mathfrak{h}(A,V) = 0$.

Assume now instead that $ {\rm card}(\mathcal{S}) > 0 $ so that the direction
\begin{equation*}
	V = \block_1 \left( \begin{bmatrix} -U_\mathcal{S}^\top \Sigma \\ 0 \end{bmatrix}, \begin{bmatrix} 0 & 0 \\ 0 & 0 \end{bmatrix} , \ldots  \begin{bmatrix} 0 & 0 \\ 0 & 0 \end{bmatrix} ,  \begin{bmatrix}  U_\mathcal{S} & 0  \end{bmatrix} \right)
\end{equation*}
is nontrivial. Consider first a critical point $ A = A_1 + Z $. From \eqref{eq:A-param1}-\eqref{eq:Z-param1} we have $ A_1 Z = Z A_1 = 0 $, and $ VZ = ZV = 0 $, and thus
\begin{equation*} 
		\rho_1(A,V) = \sum_{j=1}^{h} (A_1^{h-j} + Z^{h-j}) V ( A^{j-1} + Z^{j-1} ) = VA_1^{h-1} + A_1^{h-1}V = A_1^h - A_1^h = 0
\end{equation*}
\begin{equation*} 
    \rho_2(A,V) = \sum_{i,j,k\geq 0 \atop i+j+k=H-2} (A_1+Z)^{i} V (A_1+Z)^{j} V (A_1+Z)^k = \sum_{i,j,k\geq 0 \atop i+j+k=h-2} A_1^{i} V A_1^{j}V A_1^k = V A_1^{h-1} V = -A_1^h.
\end{equation*}
Since
\begin{align*}
    \tr{( \rho_2(A,V) ( E - A^h )^\top )} & = -\tr{( A_1^h ( E - A_1^h )^\top )} = -\tr{( U_\mathcal{S} U_\mathcal{S}^\top \Sigma ( \Sigma - U_\mathcal{S} U_\mathcal{S}^\top \Sigma )^\top )} \\
    & = -\tr{(  ( \Sigma^\top -  \Sigma^\top U_\mathcal{S}  U_\mathcal{S}^\top ) U_\mathcal{S} U_\mathcal{S}^\top \Sigma)} = -\tr{(  \Sigma^\top U_\mathcal{S} U_\mathcal{S}^\top \Sigma -  \Sigma^\top U_\mathcal{S}   U_\mathcal{S}^\top \Sigma)} = 0
\end{align*}
we get $ \mathfrak{h}(A,V) = \lVert \rho_1(A,V)  \rVert_F^2 - 2 \tr{( \rho_2(A,V) ( E - A^h )^\top )} = 0  $. For a general critical point $ A = P( A_1 + Z ) P^{-1} $, $ P \in \mathcal{P} $ we may by Lemma~\ref{lem:equivariant_rho} use the direction $ P^{-1} V P $ to obtain $ \mathfrak{h}(A,V) = 0 $.
\end{proof}

\subsubsection{Strict saddle point: computing negative directions of the Hessian quadratic form}

To show strictness of a saddle point, it is enough to generate a direction $ V \in \mathcal{A}$ along which the Hessian quadratic form $ \mathfrak{h}(A,V)$ computed at a critical point $ A \in \mathcal{A}$ becomes negative. 
Explicit methods for generating such a $V$ are given in \cite{achour2021loss, chitour2018geometric}.
Here we revisit the approach of \cite{achour2021loss}, in particular the strict saddle point case of Proposition~13 of \cite{achour2021loss}.

\begin{proposition}
\label{prop:var-neg-Hessian}
Consider a critical point $  A = P(A_1 +Z ) P^{-1} \in \mathcal{A}$ of signature $ \mathcal{S} $ s.t. $ {\rm card}(\mathcal{S} )=r $ and $ \mathcal{S}\neq \{1, \ldots, r\}$ , with $A_1 $ as in \eqref{eq:A-param1} and $ Z $ as in \eqref{eq:Z-param1}. Consider a variation $ V = P^{-1}  R  P \in \mathcal{A}$, where
$
R =  \block_1 \left( \begin{bmatrix} U_\mathcal{K}^\top  \Sigma  \\ 0 \end{bmatrix}, 0,  \ldots , 0, 
 \begin{bmatrix} U_\mathcal{K} & 0 \end{bmatrix} \right)
$, with $ U_\mathcal{K} =  \mathbb{E}_{i,j} \in \mathbb{R}^{d_y \times r} $, where $ (i, j)$ is any index pair s.t. $ i<j$, $ i \notin \mathcal{S} $, $ j \in \mathcal{S} $.
Then the Hessian quadratic form $ \mathfrak{h}(A,V)$ is negative.
\end{proposition}

\begin{proof}
Disregarding for now $ P$, consider the critical point $ A= A_1+Z $ and the variation $ V = R$. 
We need to compute $ \rho_1 $ and $ \rho_2 $ in \eqref{eq:rho1} and \eqref{eq:rho2}.
From \eqref{eq:A-param1}-\eqref{eq:Z-param1}, direct calculations show that $ A_1 Z = Z A_1 =0 $, $  R Z = Z R =0 $, while instead $ A_1 R \neq 0 $ and $ R A_1 \neq 0$, meaning that 
\begin{align*}
\rho_1(A,V)  =  \sum_{j=1}^{h} (A_1^{h-j} + Z_1^{h-j}) R  (A_1^{j-1} +Z_1^{j-1} ) = RA_1^{h-1} + A_1^{h-1}R. 
\end{align*}
where $ RA_1^{h-1} =  \block_h (U_\mathcal{K}U_\mathcal{S}^\top \Sigma) =  \block_h(\sigma_j \mathbb{E}_{i,j})$ and $ A_1^{h-1}R = \block_h (U_\mathcal{S}U_\mathcal{K}^\top \Sigma)= \block_h(\sigma_i \mathbb{E}_{j,i})$. Now, using $\mathbb{E}_{i,j}\mathbb{E}_{k,\ell}^\top=\delta_{j,\ell}\mathbb{E}_{i,k}$, we get $$\| \rho_1(A,V)\|_F^2 =  \tr{(\sigma_i^2 \mathbb{E}_{j,i}\mathbb{E}_{j,i}^\top + \sigma_i\sigma_j(\mathbb{E}_{i,j}\mathbb{E}_{j,i}^\top + \mathbb{E}_{j,i}\mathbb{E}_{i,j}^\top) + \sigma_j^2 \mathbb{E}_{i,j}\mathbb{E}_{i,j}^\top)} = (\sigma_i^2 + \sigma_j^2).$$
Concerning $ \rho_2(A,V)$, after similar simplifications, we have
\[
 \rho_2(A,V) =  \sum^{h-2}_{ i, j, k =0 \atop i+j+k=h-2 
} A_1^i R A_1^j R A_1^k .
\]
 Direct calculations lead to $ A_1^i R A_1^j =0 $ when both $ i >0 $ and $ j>0 $, plus $ R^2A_1^{h-2}= A_1^{h-2}R^2 =0 $, meaning that the only nonzero term in the expression above is
\[
\rho_2(A,V) =  R A_1^{h-2} R = \block_h \left( U_\mathcal{K} U_\mathcal{K}^\top \Sigma \right) .
\]
Notice that $ U_\mathcal{K} U_\mathcal{K}^\top \Sigma = \sigma_i  \mathbb{E}_{i,i} $, (where $ [S \Sigma]_{ii}=0$).
Then
\begin{equation*}
    \mathfrak{h}(A,V) =\| \rho_1(A,V)\|_F^2 - 2 \tr \left( \rho_2(A,V) (E - A^h)^\top \right) = (\sigma_i^2 + \sigma_j^2) - 2  \tr{\left( \sigma_i \mathbb{E}_{i,i} (\Sigma - S \Sigma)^\top \right) } = -(\sigma_i^2 - \sigma_j^2) < 0
\end{equation*}
since under Assumption~\ref{ass:distinct-sums} it is $\sigma_i^2 > \sigma_j^2$.
If we now consider the critical point $ A = P(A_1+Z)P^{-1}$, choosing as variation $ V = P^{-1} R P $, we still have $ \mathfrak{h}(A,V) <0$, see Lemma~\ref{lem:equivariant_rho}. 
\end{proof}

\begin{remark}
\label{rem:diag-case}
Even when a saddle point is strict, there could be regions of the state space in which it behaves as a non-strict critical point, i.e., in which the Hessian quadratic form is always nonnegative. An example is given in Section~\ref{sec:diag-evol}, where we show that in the submanifold of diagonal matrices, which is invariant under \eqref{eq:grad-flowA}, the Hessian can be positive semidefinite. This submanifold is however of zero Lebesgue measure in $ \mathcal{A} $, hence it is avoided by a generic trajectory of \eqref{eq:grad-flowA}.
\end{remark}

\subsubsection{Non-strict saddle point: nonnegativity of the Hessian quadratic form}
It is shown in \cite{achour2021loss} that critical points characterized by $\mathcal{S}=\{1,...,r\}$, $r<d_y$, that are tightened correspond to non-strict saddle points. 
When restricted to the sufficient condition for tightness mentioned at the end of Section~\ref{sec:param-crit-p}, showing that the Hessian quadratic form is nonnegative becomes much shorter using our notation.

\begin{proposition}
    Consider a critical point $  A = P(A_1 +Z ) P^{-1} \in \mathcal{A}$ of signature $ \mathcal{S} = \{1, \ldots, r\}$ , with $A_1 $ and $ Z $ as in \eqref{eq:A-param1}-\eqref{eq:Z-param1}. Let three of the matrices $Z_1,...,Z_h$ in \eqref{eq:Z-param1} be zero matrices, implying that $A$ is tightened. Then the Hessian quadratic form $ \mathfrak{h}(A,V)$ is nonnegative for all $ V \in \mathcal{A}$.
\end{proposition}

\begin{proof}
First we introduce some notation. Let $I_\mathcal{S}\in\mathbb{R}^{d_y\times d_y }$ be the matrix with 1s on the first $r$ diagonal elements. Let $I_\mathcal{Q}=I_{d_y}-I_\mathcal{S}$. Let $\mathcal{I}_\mathcal{S}\in\mathbb{R}^{n_v \times n_v }$ be the matrix with $I_\mathcal{S}$ as upper left and lower right blocks and $0$ elsewhere, and let $\mathcal{I}_\mathcal{Q}$ be the matrix with $I_\mathcal{Q}$ as upper left and lower right blocks and $0$ elsewhere. In this way, $\mathcal{I}_\mathcal{S}+\mathcal{I}_\mathcal{Q}$ acts as the identity when multiplying by $E$ or $A^h$ from either left or right and $\mathcal{I}_\mathcal{Q} A_1= A_1 \mathcal{I}_\mathcal{Q} =\mathcal{I}_\mathcal{S} Z= Z \mathcal{I}_\mathcal{S} = 0$. Let also $\Sigma_\mathcal{S}\in \mathbb{R}^{r \times r}$, $\Sigma_\mathcal{Q}\in \mathbb{R}^{(d_y-r) \times (d_y-r)}$ be diagonal matrices containing the $r$ first respectively $d_y-r$ last ordered singular values on the diagonals.
    Let us disregard $P$ for now and recall that $\mathfrak{h}(A,V) $ has the expression \eqref{eq:hess-quadr-form}, where $ V =\block_1( V_1, \ldots, V_h )$. We begin by splitting $\lVert \rho_1(A,V)  \rVert_F^2$ into three terms:
    \begin{align}
    	\lVert \rho_1(A,V)  \rVert_F^2 & = \left\lVert \sum_{j=1}^{h} (A_1 + Z)^{h-j}V(A_1 + Z)^{j-1} \right\rVert_F^2 = \left \lVert \sum_{j=1}^{h} (A_1^{h-j} + Z^{h-j})V(A_1^{j-1} + Z^{j-1}) \right \rVert_F^2 \label{eq:first_ord_line1} \\
    	& = \left \lVert (\mathcal{I}_\mathcal{S}+\mathcal{I}_\mathcal{Q}) \sum_{j=1}^{h} (A_1^{h-j} + Z^{h-j})V(A_1^{j-1} + Z^{j-1}) (\mathcal{I}_\mathcal{S}+\mathcal{I}_\mathcal{Q}) \right \rVert_F^2 \\
    	& = \left \lVert \mathcal{I}_\mathcal{S} \sum_{j=1}^{h} A_1^{h-j} V A_1^{j-1} \mathcal{I}_\mathcal{S} \right \rVert_F^2 + \left \lVert \mathcal{I}_\mathcal{S} \sum_{j=1}^{\ell_1} A_1^{h-j} V Z^{j-1} \mathcal{I}_\mathcal{Q} \right \rVert_F^2 + \left \lVert \mathcal{I}_\mathcal{Q} \sum_{j=h-\ell_3-1}^{h} Z^{h-j} V A_1^{j-1} \mathcal{I}_\mathcal{S} \right \rVert_F^2 \label{eq:first_ord_terms}
    \end{align}
    where we have used $\lVert B_1 + B_2 \rVert_F^2 = \lVert B_1 \rVert_F^2 + \lVert B_2 \rVert_F^2 + 2\tr{(B_1B_2^\top)}$ for compatible matrices, and $\mathcal{I}_\mathcal{S}^\top\mathcal{I}_\mathcal{Q} = 0$ to set the cross terms to zero. By the assumption that three blocks (of indices $ 1\leq \ell_1<\ell_2<\ell_3 \leq h $) in $Z$ are zero, all terms with only powers of $Z$ in \eqref{eq:first_ord_line1} are zero. 

    In the critical point we consider it is $E-A_1^h = \text{blk}_h\left(\begin{bmatrix}
	0 & 0 & 0 \\ 0 & \Sigma_\mathcal{Q} & 0
    \end{bmatrix}\right) = \mathcal{I_Q}E$. The second term in the Hessian quadratic form is
    \begin{align}
    	\tr{(\rho_2(A,V)(E - A^h)^\top)} & = \tr{ \left( \left[ \sum_{i,j,k\geq 0 \atop i+j+k=h-2} (A_1 + Z)^{i} V (A_1 + Z)^{j}V (A_1 + Z)^k \right] (\mathcal{I_Q} E )^\top \right)} \nonumber \\
    	&= \tr{ \left( (\mathcal{I_Q} E )^\top \left[ \sum_{i,k\geq 0; j> 0 \atop i+j+k=h-2} Z^{i} V A_1^{j} V Z^k \right]   \right)}\label{eq:second_ord_terms0} \\
    	& = \tr{ \left( \Sigma_\mathcal{Q}^\top \left[ \sum_{k=\ell_3}^{h} Z_h \cdots Z_{k+1} V^{(k)}_{21} \right]\left[ \sum_{j=1}^{\ell_1} V^{(j)}_{12} Z_{j-1}\cdots Z_{1} \right]  \right)} \label{eq:second_ord_terms}
    \end{align}
    where we have expressed the $k$-th block of $V$ as $V_k=\begin{bmatrix}
    	V^{(k)}_{11} & V^{(k)}_{12} \\ V^{(k)}_{21} & V^{(k)}_{22}
    \end{bmatrix}$
    and
    \begin{align*}
    	[A_1^{h-k} V Z^{k-1}]_{(h+1,1)} &= \begin{bmatrix}
    		I_r & 0 
    	\end{bmatrix}\begin{bmatrix}
    	V^{(k)}_{11} & V^{(k)}_{12} \\ V^{(k)}_{21} & V^{(k)}_{22}
    	\end{bmatrix}
    	\begin{bmatrix}
    		0 \\  Z_{k-1}\cdots Z_1
    	\end{bmatrix} = V^{(k)}_{12} Z_{k-1}\cdots Z_1 \\
    	[Z^{h-k} V A_1^{k-1}]_{(h+1,1)} &= \begin{bmatrix}
    		0 & Z_{h}\cdots Z_{k+1}
    	\end{bmatrix}\begin{bmatrix}
    		V^{(k)}_{11} & V^{(k)}_{12} \\ V^{(k)}_{21} & V^{(k)}_{22}
    	\end{bmatrix}
    	\begin{bmatrix}
    		\Sigma_\mathcal{S} \\ 0 
    	\end{bmatrix} =Z_{h}\cdots Z_{k+1}V^{(k)}_{21}\Sigma_\mathcal{S} 
    \end{align*}
    with $Z_{h}\cdots Z_{k+1} V^{(k)}_{21} \Sigma_\mathcal{S} \in \mathbb{R}^{(d_y-r)\times r}$ and $V^{(k)}_{12} Z_{k-1}\cdots Z_1 \in \mathbb{R}^{r\times (d_y-r)}$. In \eqref{eq:second_ord_terms0} we again used that terms with only powers of $Z$ and $V$ are zero, and that $\mathcal{I_Q}A_1 = A_1 \mathcal{I_Q}= 0$ to eliminate terms with powers of $A_1$ on the left and right of the $V$:s.

    Consider now the last term in \eqref{eq:first_ord_terms}:
    \begin{align}
    	\left\lVert \mathcal{I}_\mathcal{Q} \sum_{j=h-\ell_3-1}^{h} Z^{h-j}VA^{j-1} \mathcal{I}_\mathcal{S} \right\rVert^2_F &= \left\lVert \sum_{k=\ell_3}^{h} Z_h \cdots Z_{k+1} V^{(k)}_{21} \Sigma_\mathcal{S} \right\rVert^2_F  \\
    	&= \left\lVert G \Sigma_\mathcal{S} \right\rVert^2_F = \sum_{j=1}^{r} \sum_{k=1}^{d_y-r} \sigma_j^2 G_{k,j} \\
    	& = \sum_{j=1}^{r} \sum_{k=1}^{d_y-r} (\sigma_j^2 - \sigma_{k+r}^2) G_{k,j}^2 + \sum_{j=1}^{r} \sum_{k=1}^{d_y-r} \sigma_{k+r}^2 G_{k,j}^2 \\
    	& = \sum_{j=1}^{r} \sum_{k=1}^{d_y-r} (\sigma_j^2 - \sigma_{k+r}^2) G_{k,j}^2 + \left\lVert \Sigma_\mathcal{Q}^\top G \right\rVert^2_F \label{eq:sing_val_split}
    \end{align}
    where $G= \sum_{k=\ell_3}^{h} Z_h \cdots Z_{k+1}V^{(k)}_{21} \in \mathbb{R}^{(d_y-r)\times r} $. Since $\sigma_j > \sigma_{k+r}$ for all $j=1,...,r$, $k=1,...,d_y-r$, the first sum in \eqref{eq:sing_val_split} is nonnegative. Denote now $ \tilde{G} = \sum_{j=1}^{\ell_1} V^{(j)}_{12}Z_{j-1} \cdots Z_{1} $ the second bracketed factor in \eqref{eq:second_ord_terms}, so that the second term in \eqref{eq:first_ord_terms} is
    \begin{equation}
    	\left \lVert \mathcal{I}_\mathcal{S} \sum_{j=1}^{\ell_1} A_1^{h-j} V Z^{j-1} \mathcal{I}_\mathcal{Q} \right \rVert_F^2= \lVert \tilde{G} \rVert_F^2 = \lVert \tilde{G}^\top \rVert_F^2. \label{eq:gtilde}
    \end{equation}
     We may now write
    \begin{equation}
    	\tr{(\rho_2(A,V)(E-A^h)^\top)} = \tr{ \left( \Sigma_\mathcal{Q}^\top \left[ \sum_{k=\ell_3}^{h} Z_h \cdots Z_{k+1}V^{(k)}_{21} \right]\left[ \sum_{j=1}^{\ell_1} V^{(j)}_{12} Z_{j-1}\cdots Z_{1} \right] \right)} = \tr{( \Sigma_\mathcal{Q}^\top G \tilde{G})}. \label{eq:hess_second_term}
    \end{equation}
    From \eqref{eq:first_ord_terms}, \eqref{eq:sing_val_split} and \eqref{eq:gtilde} we have $\lVert \rho_1(A,V)  \rVert_F^2 \geq \lVert \Sigma_\mathcal{Q}^\top G \rVert_F^2 + \lVert \tilde{G}^\top \rVert_F^2$, and with \eqref{eq:hess_second_term} it is
    \begin{equation*}
    	\mathfrak{h}(A,V) = \lVert \rho_1(A,V)  \rVert_F^2 - 2\tr{(\rho_2(A,V)(E-A^h)^\top)} \geq \lVert \Sigma_\mathcal{Q}^\top G \rVert_F^2 + \lVert \tilde{G}^\top \rVert_F^2 - 2\tr{( \Sigma_\mathcal{Q}^\top G \tilde{G})} = \lVert \Sigma_\mathcal{Q}^\top G - \tilde{G}^\top \rVert_F^2 \geq 0
    \end{equation*}
    To see that this holds for any critical point $PAP^{-1}=P(A_1+Z)P^{-1}$, we just apply Lemma~\ref{lem:equivariant_rho}. If there exists $V$ such that $ \mathfrak{h}(PAP^{-1},V) < 0 $, there must also exist $ V^\prime =P^{-1}VP$ such that $\mathfrak{h}(A, V^\prime )<0$, which we have disproved.
\end{proof}

\subsection{A classification of stable/unstable submanifolds of the critical points}
\label{sec:stab_unstab_manif}

Intuitively, the presence of a negative direction in the Hessian guarantees the existence of an unstable submanifold along the same direction, at least locally around the critical point. 
The condition is sufficient but not necessary, because directions that vanish at the second order term of the Taylor expansion may give a contribution at higher order terms.
The construction of equivalence classes over $ \mathcal{A}_\phi$ developed in Section~\ref{sec:invar-equiv} allows us to explicitly compute stable and unstable submanifolds at critical points, without relying on a local series expansion. 
In particular, next theorem shows how to compute systematically submanifolds connecting different critical points belonging to different level sets $ \mathcal{M}_{c_i}$.
It therefore provides a tool for loss landscape analysis which is non-local, i.e., not restricted to critical points and their neighborhoods.

Given two critical points $ A_1$ and $ A_2 $, if $ \Phi_i = A_i^h $, then it is possible to consider the convex combination $ \Phi(\alpha) = (1-\alpha) \Phi_1 + \alpha \Phi_2$, $ \alpha \in [0, \, 1 ]$. This gives rise to a nonlinear (nonunique) combination also in the fiber $ \phi^{-1}(\Phi(\alpha)) $. Denote $A(\alpha)$ one element of $ \phi^{-1}(\Phi(\alpha)) $ s.t. $ A(0)=A_1 $ and $ A(1) =A_2 $.

\begin{theorem}
\label{thm:arcs}
Consider the gradient flow system \eqref{eq:grad-flowA} and let $ A_1$ and $ A_2 $ be two critical points of signature $ \mathcal{S}_1 $ and $ \mathcal{S}_2$. Consider the arc $ A(\alpha) \in \phi^{-1}(\Phi(\alpha)) $, where $ \Phi(\alpha) = (1-\alpha) \Phi_1 + \alpha \Phi_2$, $ \alpha \in [0, \, 1 ]$, $ \Phi_i = A_i^h$. Under Assumptions~\ref{ass:size} and~\ref{ass:distinct-sums},
\benu
\item If $ \mathcal{S}_1 =\mathcal{S}_2$, then the entire arc $ A(\alpha) $ collapses into a single point in  $ \mathcal{A}_\phi$, and $ \mathcal{L}(A(\alpha)) = \sum_{i=1}^{d_y} (1- s_i) \sigma_i^2 $ $\forall \, \alpha \in [0, 1]$.

\item If $ \mathcal{S}_1  $ and $ \mathcal{S}_2$ s.t. $ \mathcal{S}_1 \subset \mathcal{S}_2 $ (which implies $ \mathcal{L}(A_1)> \mathcal{L}(A_2)$) then $ \mathcal{L}(A(\alpha)) $ is monotonically decreasing as $ \alpha $ passes from $ 0$ to $ 1$. This implies that, along the arc $ A(\alpha)$, $ A_1 $ has an unstable submanifold, while $ A_2 $ has a stable one. 

\item A critical point $A$ of signature $ \mathcal{S} $ s.t. $ {\rm card}(\mathcal{S})= k $ has $ 2^{d_y-k} -1$ unstable submanifolds and $ 2^k - 1 $ stable ones. Therefore every $ \mathcal{S}\neq \{ 1, \ldots, d_y \}  $ has at least an unstable submanifold, while the critical point $ A=0$  (i.e., $ \mathcal{S}=\emptyset $) has $ 2^{d_y}-1 $ unstable submanifolds. 

\item If $ \mathcal{S}_1  $ and $ \mathcal{S}_2$ differ for a signature of ``mixed type'' (i.e., neither $\mathcal{S}_1 \subset \mathcal{S}_2 $ nor $ \mathcal{S}_1 \supset \mathcal{S}_2 $)  then nothing can be said of the stability properties of the arc $ A(\alpha)$. 
\eenu
\end{theorem}
\begin{proof}
\benu
\item Trivial, since $ \mathcal{S}_1 = \mathcal{S}_2$ implies $ \Phi_1 = \Phi_2 = \Phi(\alpha) $ $ \forall \, \alpha \in [0,1]$, hence $ A(\alpha) \in \phi^{-1} (\Phi_1)$ $ \forall \, \alpha \in [0,1]$. 

\item Use the representation \eqref{eq:A-param1}, for which it is $ \Phi_i = A_i^h= \block_h(U_{\mathcal{S}_i} U_{\mathcal{S}_i}^\top \Sigma) = \block_h( S_i \Sigma) $ (where $ S_i = \diag(\mathbf{s}_i)$). Hence $ \Phi(\alpha) = (1-\alpha) \Phi_1 + \alpha \Phi_2 = \block_h( ((1-\alpha) S_1 + \alpha S_2 ) \Sigma)$. If $ \mathcal{S}_1 \subset \mathcal{S}_2 $, then all $ \sigma_i $ present in $A_1 $ are present also in $A_2$, but some of those in $A_2 $ are missing in $A_1$. Denote $ \mathcal{T} = \mathcal{S}_1 \cap \mathcal{S}_2 $ and $ \mathcal{K} = \mathcal{S}_2 \setminus \mathcal{S}_1 $ and $ S_\mathcal{T}$, $ S_\mathcal{K} $ the associated diagonal matrices.
Then it is $ \Phi(\alpha) = \block_h \left( (S_\mathcal{T} + \alpha S_\mathcal{K} ) \Sigma \right)$, and, for any $ A(\alpha ) \in \phi^{-1}(\Phi(\alpha)) $, 
$$ 
\mathcal{L}(A(\alpha) ) = \frac{1}{2} \| E - \Phi(\alpha) \|_F^2 = \frac{1}{2}\sum_{i\in \{1, \ldots, d_y\} \setminus \mathcal{S}_2}  \sigma_i^2 + \frac{1}{2}\sum_{i\in \mathcal{K}} (1-\alpha)^2 \sigma_i^2 
$$ 
is a monotonically decreasing function of $ \alpha$. Hence, along the arc $ A(\alpha)$, $A_1 $ has to have an unstable submanifold and $ A_2 $ a stable submanifold. 

\item For a given $ \mathcal{S}$, just apply the previous statement choosing $ \mathcal{S}_1 = \mathcal{S} $ and selecting $\mathcal{S}_2 $ s.t. $ \mathcal{S}_1 \subset \mathcal{S}_2 $ to find the unstable submanifolds and $ \mathcal{S}_1 \supset \mathcal{S}_2 $ to find the stable ones.
Any $ \mathcal{S}\neq \emptyset $ provides a direction along which $A = 0 $ has an unstable submanifold. There are obviously $ 2^{d_y}-1 $ such directions.

\item As is easily shown by counterexamples,  $ \mathcal{L}(A(\alpha)) $ is not monotone along the arc $ A(\alpha)$, see Section~\ref{sec:ex-qualitat-anal} and Fig.~\ref{fig:sim2}. 
\eenu
\end{proof}

Notice that the arcs considered in Theorem~\ref{thm:arcs} cannot be trajectories of the gradient flow, as they connect two critical points of \eqref{eq:grad-flowA}.
In general, we do not know how the arcs $A(\alpha)$ relate to the trajectories of \eqref{eq:grad-flowA}. 
For instance, in spite of the absence of local minima in $ \mathcal{L}$, the arcs $ A(\alpha)$ leading to a non-monotone  $ \mathcal{L}(A(\alpha)) $ mentioned in item~4 of Theorem~\ref{thm:arcs} are actually characterized by the presence of local minima/maxima along $ \mathcal{L}(A(\alpha)) $, see Fig.~\ref{fig:sim2} below.

\begin{remark}
Even though the loss function $ \mathcal{L}$ is not a Morse function (it has degenerate critical points, e.g. $A=0$), it somehow behaves like a Morse function if we restrict the state space to the $ A(\alpha)$ arcs considered in Theorem~\ref{thm:arcs}.
For instance, the ``maximally non-strict'' critical point $A=0$ has all arcs towards all other critical points that are unstable, and achieves the maximal critical value $ \mathcal{L}(0)= \frac{1}{2}\sum_{i=1}^{d_y} \sigma_i^2 $ among all critical points. However $ A=0$ has also stable submanifolds, and hence  it is not a maximum, see example in Section~\ref{sec:1-node-1-layer}.
\end{remark}

\subsection{Convergence rate}
\label{sec:conv-rate}

Lojasiewicz's theorem guarantees the existence of an explicit upper bound for  $ \dot{\mathcal{L}} $ as a function of $ \mathcal{L}$ \cite{lojasiewicz1982trajectoires}. 
For the gradient system \eqref{eq:grad-flowA}, a (not necessarily strictly negative) upper bound can always be computed explicitly.
An example is given in the following proposition.
\begin{proposition}
\label{prop:lojasiewicz-bound}
For the loss function $ \mathcal{L}$ we have 
\beq
\label{eq:dotL6}
\dot{\mathcal{L}} \leq  - \sum_{j=1}^h  \prod_{i =1}^{j-1} \lambda_{\min}(W_i^\top  W_i) \prod_{i =j+1}^{h} \lambda_{\min}(W_i W_i^\top) \mathcal{L}.
\eeq
\end{proposition}
\begin{proof}
As mentioned in the proof of Proposition~\ref{prop:Layp1}, eq.~\eqref{eq:dotL2} can be rewritten using the $ W_i $ matrices as 
\beq
\begin{split}
 \label{eq:dotL-inW} 
\dot{\mathcal{L}} 
 = -  \sum_{j=1}^h \tr \left(  W_{1:j-1}(\Sigma-W_{1:h})^\top W_{j+1:h} W_{j+1:h}^\top(\Sigma-W_{1:h})W_{1:j-1}^\top \right).
\end{split}
\eeq
The proof of expression \eqref{eq:dotL6} uses arguments similar to those of Proposition~2 of \cite{chitour2018geometric}, in particular the cyclic identity of the trace and the inequality $ \tr( F^\top F G^\top G) \geq \lambda_{\min} (F^\top F) \tr( G^\top G) $ applied repeatedly on each term of \eqref{eq:dotL-inW}:
\begin{align*}
&\tr \left(  W_{1:j-1}(\Sigma-W_{1:h})^\top W_{j+1:h} W_{j+1:h}^\top(\Sigma-W_{1:h})W_{1:j-1}^\top  \right) \\
& = \tr \left( W_{j+1} W_{j+1} ^\top  W_{j+2:h}^\top(\Sigma-W_{1:h})W_{1:j-1}^\top W_{1:j-1}(\Sigma-W_{1:h})^\top  W_{j+2:h}\right) \\
 & \geq \lambda_{\min}( W_{j+1} W_{j+1} ^\top ) \tr \left( W_{j+2:h}^\top(\Sigma-W_{1:h})W_{1:j-1}^\top W_{1:j-1}(\Sigma-W_{1:h})^\top W_{j+2:h} \right) \\
 & \qquad \text{and, iterating,} \\
 &\geq  \lambda_{\min}(W_{j+1} W_{j+1}^\top )\ldots \lambda_{\min}(W_{h} W_{h}^\top )  \tr \left( (\Sigma-W_{1:h})W_{1:j-1}^\top W_{1:j-1}(\Sigma-W_{1:h})^\top  \right) \\
  &=   \lambda_{\min}(W_{j+1} W_{j+1}^\top )\ldots \lambda_{\min}(W_{h} W_{h}^\top ) \tr \left( W_{j-1}^\top W_{j-1}  W_{1:j-2} (\Sigma-W_{1:h})^\top (\Sigma-W_{1:h}) W_{1:j-2}^\top \right) \\
& \geq  \lambda_{\min}(W_{j-1}^\top  W_{j-1})  \lambda_{\min}(W_{j+1} W_{j+1}^\top )\ldots \lambda_{\min}(W_{h} W_{h}^\top ) \tr \left( W_{1:j-2} (\Sigma-W_{1:h})^\top (\Sigma-W_{1:h}) W_{1:j-2}^\top \right) \\
 & \qquad \text{and, iterating,} \\
 & \geq   \lambda_{\min}(W_1^\top  W_1)\ldots  \lambda_{\min}(W_{j-1}^\top  W_{j-1})  \lambda_{\min}(W_{j+1} W_{j+1}^\top )\ldots \lambda_{\min}(W_{h} W_{h}^\top ) \tr \left(  (\Sigma-W_{1:h})^\top (\Sigma-W_{1:h})  \right).
\end{align*}
Summing over the $ h$ layers one gets \eqref{eq:dotL6}.
\end{proof}
All terms $ \prod_{i =1}^{j-1} \lambda_{\min}(W_i^\top  W_i) \prod_{i =j+1}^{h} \lambda_{\min}(W_i W_i^\top) $ in \eqref{eq:dotL6} are at least nonnegative, hence it is enough that one of them is positive for all $ t $ to guarantee an exponential convergence rate.
The bound in Proposition~\ref{prop:lojasiewicz-bound} however depends on the trajectory followed by the gradient flow, which makes it difficult to apply. Furthermore, it is problematic in presence of any ``local bottleneck'' $ d_{j-1} > d_j < d_{j+1} $, since in this case $ \lambda_{\min}(W_{j+1} W_{j+1}^\top) = 0 $ and $\lambda_{\min}(W_j^\top W_j) = 0$. At least one of these two factors necessarily appears in all terms of the expression \eqref{eq:dotL6}.

In some special cases, by choosing the dimensions $ d_i $ opportunely, the upper bound in \eqref{eq:dotL6} can be rendered time-invariant (i.e., independent from the trajectory followed by the gradient flow), meaning that exponential convergence of $ \mathcal{L}$ is guaranteed.
One case is provided in  \cite{chitour2018geometric} Proposition~2. A slightly more general one is in the following proposition.
\begin{proposition}
\label{prop:conv-rate-sp-case1}
If the network layers fulfill $d_x \leq d_1 \leq ... \leq d_j \geq ... \geq d_{h-1} \geq d_y$ (and $d_y \leq d_x$ if Assumption~\ref{ass:size} holds) for some $ j = 1, \ldots, h$, and $ C_i $ has at least $d_{i-1}$ positive eigenvalues for $ i < j $, and at least $d_{i+1}$ negative eigenvalues for $ i \geq j $, then the convergence of the gradient flow is exponential with rate constant at least 
\beq
\prod_{i=1}^{j-1} \lambda_{1+d_{i+1}-d_{i}} ( C_{i} ) \prod_{i=j}^{h-1} \lambda_{1+d_{i}-d_{i+1}} ( -C_{i} ) > 0.
\label{eq:dotL10}
\eeq
\end{proposition}
\begin{proof}
In \eqref{eq:dotL6}, consider the factor for the $j$-th term, $ \prod_{i =1}^{j-1} \lambda_{\min}(W_i^\top  W_i) \prod_{i =j+1}^{h} \lambda_{\min}(W_i W_i^\top) $. Applying Weyl's inequality to the conservation laws $C_i = W_i W_i^\top - W_{i+1}^\top W_{i+1} $ we get
\begin{align}
   \lambda_k ( W_i W_i^\top ) & \geq \lambda_k ( C_i ) + \lambda_{\min} ( W_{i+1}^\top W_{i+1} ) \geq \lambda_k ( C_i ) \label{eq:weyl1} \\ 
   \lambda_k ( W_i^\top W_i ) & \geq \lambda_k ( -C_{i-1} ) + \lambda_{\min} ( W_{i-1} W_{i-1}^\top ) \geq \lambda_k ( -C_{i-1} ). \label{eq:weyl2}
\end{align}
Noting that the eigenvalues of $ W_i^\top W_i $ are identical to those of $ W_i W_i^\top $ (or vice versa), possibly with extra zero eigenvalues, we have for $i < j$: $ \lambda_k ( W_i^\top W_i ) = \lambda_{k+d_{i+1}-d_{i}} ( W_i W_i^\top ) $, and for $i > j$: $ \lambda_{k} ( W_i W_i^\top ) = \lambda_{k+d_{i-1}-d_{i}} ( W_i^\top W_i ) $. From \eqref{eq:weyl1}, \eqref{eq:weyl2} we get
\begin{align*}
	\lambda_{\min} ( W_i^\top W_i ) & = \lambda_{1+d_{i+1}-d_{i}} ( W_i W_i^\top ) \geq \lambda_{1+d_{i+1}-d_{i}} ( C_{i} ) > 0 & i < j \\
	\lambda_{\min} ( W_i W_i^\top ) & = \lambda_{1+d_{i-1}-d_{i}} ( W_i^\top W_i ) \geq \lambda_{1+d_{i-1}-d_{i}} ( -C_{i-1} ) > 0 & i > j 
\end{align*}
and hence \eqref{eq:dotL10} follows and 
\begin{equation*}
\dot{\mathcal{L}} \leq - \prod_{i=1}^{j-1} \lambda_{1+d_{i+1}-d_{i}} ( C_{i} ) \prod_{i=j}^{h-1} \lambda_{1+d_{i}-d_{i+1}} ( -C_{i} ) \mathcal{ L} <0.
\end{equation*}
\end{proof}
In Proposition~\ref{prop:conv-rate-sp-case1}, when $j=1$, we obtain the sufficient condition of Proposition~2 of \cite{chitour2018geometric}. 
Another bound on the convergence rate is given in \cite{min2023convergence}, in which the constraints on layer widths in Proposition~\ref{prop:conv-rate-sp-case1} are relaxed, but the conditions on the conservation laws change slightly. As the proof is rather involved, we state only the result in the following proposition.  Note that the conserved quantities $D_i$ in \cite{min2023convergence} are for us $ -C_{h-i} $.
\begin{proposition}
    \label{prop:conv-rate-general}
    For a network with layer widths $d_i \geq d_x,d_y$ and weights such that $C_i \preceq 0 $ for $i \geq j$ and $C_i \succeq 0 $ for $i < j$ for some $j\in\{1,...,h\}$, the convergence of gradient flow is exponential with rate not lower than
    \beq
    \label{eq:conv-rate-gen}
        \alpha_h = \prod_{i=1}^{h-1} \alpha_{(i)}, \qquad \alpha_{(i)} = 
        \begin{cases}
            \sum\limits_{\ell = j}^{i} \lambda_{1+d_\ell - d_y}(-C_\ell) & i \geq j \\
            \sum\limits_{\ell = i}^{j-1} \lambda_{1+d_\ell - d_x}(C_\ell) & i < j .
        \end{cases}
    \eeq
\end{proposition}
Since the factors $\alpha_{(i)}$ in \eqref{eq:conv-rate-gen} are sums of several eigenvalues of the $C_i$:s, the rate might be faster than in \eqref{eq:dotL10}. However, which eigenvalues enter into the expression for the rate depends on the particular layer widths.

\section{Simplified gradient flow dynamics}
\label{sec:simplifications}

So far we have not placed any restriction on the weights of the network when studying the gradient flow dynamics. There are however several special cases worth investigating in detail, as they lead to extra insight into the dynamics of \eqref{eq:grad-flowA}. 
We list now a series of definitions, (partially) standard from the literature.
Consider a network $ A =\block_1(W_1, \ldots, W_h)\in \mathcal{ A } $ and the associated block-shift SVD $ A= U \Lambda V^\top $ introduced in Section~\ref{sec:block-shift-svd} (or, in terms of the $ W_i $ blocks, $ W_i = U_i \Lambda_i V_i^\top$). 
Recall from Section~\ref{sec:conserv_law} that $  J= \block_0 (0_{d_x}, I_{d_1} \ldots  I_{d_{h-1}} , 0_{d_y} ) $. $A$ is said 
\bite
\item {\em aligned} if $ V^\top U = J  $ (or, equivalently, $ V_{i+1}^\top U_i =I $ for all $ i=1,\ldots, h-1$).
\item {\em decoupled} if $A$ is aligned and $ U_h = I $, $ V_1 = I $.
\item {\em block-shift diagonal} if $ W_i $ is diagonal for all $ i=1, \ldots, h$. 
\item  {\em 0-balanced} if the conserved quantity in Proposition~\ref{prop:conserv-law-A} is $ C = 0 $ (or equivalently, if $ W_i W_i^\top - W_{i+1}^\top W_{i+1} = 0 $ for all $ i = 1,...,h-1 $).
\eite

The 0-balance condition was introduced in \cite{fukumizu1998effect} and has since been used in several works \cite{arora2018convergence,arora2018optimization,arora2019implicit,domine2023exact}. A relatively early work studying both decoupled and 0-balance dynamics is \cite{saxe2013exact}. Decoupled dynamics has also been used in many papers, including \cite{tarmoun2021understanding,saxe2019mathematical,gidel2019implicit,min2023convergence,lampinen2018analytic}, where it is sometimes called {\em spectral} or {\em training aligned} initialization.
As shown below, the block-shift diagonal case is just a special case of decoupled dynamics that is easier to deal with.

The rest of this section summarizes the structural properties of the various cases, and shows some consequences for the gradient flow dynamics. The cases are for instance useful to illustrate sequential learning of the modes of the data, but in themselves not sufficient for this phenomenon to occur.
Also the combination of 0-balance and block-shift diagonal is insightful: it essentially unveils a scalar ODE exhibiting several of the features of the entire class of gradient flows.

\subsection{Aligned networks}

The alignment condition corresponds to the left and right singular vectors of two consecutive weight matrices being identical. This implies that when taking the product $ W_{1:h} $ all inner factors cancel each other, and we have that the SVD of the product is simply $ W_{1:h}= U_h \Lambda_{1:h} V_1^\top $.

When we consider a general $ A\in \mathcal{A}$, the SVD operation does not extend well to matrix powers. This is the main reason for the complicated expression that the critical points $A^\ast $ assume even though the expression for $(A^\ast)^h $ is very simple. The alignment condition provides an exception to this rule, at least as long as we consider the block-shift SVD of Proposition~\ref{prop:svdA2}. Roughly speaking, the operations of taking powers and computing the block-shift SVD commute when $A$ is aligned.

\begin{proposition}
\label{prop:A^k-SVD}
    If $ A \in \mathcal{ A } $ is aligned, then $ A^k = U \Lambda^k V^\top $ for all $ k=1,\ldots, h$.
\end{proposition}

\begin{proof}
When $ A = U \Lambda V^\top $ is aligned, we have $ V^\top U = J  $. Thus $ A^k = U \Lambda J \Lambda \cdots J \Lambda V^\top = U ( \Lambda J )^{k-1} \Lambda V^\top $. Since $ \Lambda J = \block_1(0, \Lambda_2, ..., \Lambda_h )$, from Section~\ref{sec:block-shift},  we have that $ ( \Lambda J )^{k-1} = \block_{k-1}(0,\Lambda_{2:k}, \Lambda_{3:k+1}, \ldots \Lambda_{h+2-k:h})$.
Using \eqref{eq:FGF}, we finally have $ A^k = U ( \Lambda J )^{k-1} \Lambda V^\top = U \block_k (\Lambda_{1:k} \Lambda_{2:k+1}, \ldots, \Lambda_{h+1-k:h}) V^\top = U \Lambda^k V^\top$.
\end{proof}

The alignment argument extends to orthogonal changes of basis in the hidden nodes. 
\begin{proposition}
If $ A\in \mathcal{A} $ is aligned, then also $ A' = P A P^\top $ with $ P \in \mathcal{P_O}$ is aligned.
\end{proposition}
\begin{proof}
   If $ A= U \Lambda V^\top  $ is aligned, for $ A' = P A P^\top = P U \Lambda V^\top P^\top $ it is $ V^\top P^\top P U = V^\top U = J $.
\end{proof}
However, as shown in the next proposition, the alignment property is not preserved by the gradient flow dynamics, while both decoupling and 0-balance are.
\begin{proposition}
\label{prop:special-impl}
Consider the gradient flow \eqref{eq:grad-flowA}. 
\bite
\item $ A(0) \in \mathcal{A} $ is aligned $\;  \;\not{\!\!\!\!\Longrightarrow}$ $  A(t) $ is aligned.
\item $ A(0) \in \mathcal{A} $ is decoupled $ \Longrightarrow $ $ A(t) $ is decoupled for all $ t$. 
\item $ A(0) \in \mathcal{A} $ is 0-balanced $ \Longrightarrow $ $ A(t) $ is 0-balanced for all $ t$. 
\eite
\end{proposition}

\begin{proof}
Assuming $ A(0) $ is aligned and applying the following change of basis $ P = \block_0(I,U_1, \ldots, U_{h-1}, I) \in \mathcal{P_O}$ at $ A(0)$ (omitting the time index for brevity), we see that the resulting network
\begin{equation}
    A^\prime = P^\top A P = P^\top U \Lambda  V^\top  P = \block_1 \left( \Lambda_1 V_1^\top, \Lambda_2, \ldots, \Lambda_{h-1}, U_h \Lambda_h \right)
    \label{eq:align-orth-transf}
\end{equation}
 is nearly diagonal.
 Assume also that $ A(0) $ is not 0-balanced, hence $ C_i \neq 0 $ for some $ i $ (if $ A(0) $ is 0-balanced then, from Proposition~\ref{prop:0-balanced} below, $A$ is aligned for all $t$ and there is nothing to show).
From \eqref{eq:block_grad}, then,  
\begin{equation}
    \label{eq:Widot_aligned}
    \dot{ W }_i = W_{i+1:h}^\top (\Sigma - W_{1:h}) W_{1:i-1}^\top = \Lambda_{i+1:h}^\top (U_h^\top \Sigma V_1 - \Lambda_{1:h}) \Lambda_{1:i-1}^\top .
\end{equation}
Obviously $ \dot{ W }_i(0) $ is not diagonal, so the singular vectors of $ W_i $ cannot be stationary for any $ i $. Consider the conservation law $ C_i = W_i W_i^\top - W_{i+1}^\top W_{i+1} $. Since $ A(0) $ is aligned, we can diagonalize $ C_i $ using $ U_i(0) = V_{i+1}(0) $:
\begin{equation*}
     C_i  = U_i(0) ( \Lambda_i(0) \Lambda_i(0)^\top - \Lambda_{i+1}(0)^\top \Lambda_{i+1}(0) ) U_i(0)^\top,
\end{equation*}
thereby obtaining its eigendecomposition. 
If the singular vectors are aligned for all $ t $, but not stationary as implied by \eqref{eq:Widot_aligned}, it must be that $ U_i(t)$, s.t. $ U_i(t) = V_{i+1}(t) \neq U_i(0) = V_{i+1}(0) $, also diagonalizes $ C_i $. However, if for instance $ C_i $ has $ d_i $ unique eigenvalues, its normalized eigenvectors $ U_i(0) $ are unique, and we cannot have $ U_i(t)$ and $ U_i(0) $, $ U_i(t) \neq U_i(0) $, both diagonalizing $ C_i $. This implies that $ U_i(t) $ does not diagonalize $ W_{i+1}(t)^\top W_{i+1}(t) $, hence $ V_{i+1}(t) \neq U_i(t) $, and the system is not aligned at time $ t > 0 $. 

For the second statement, recall that decoupling means that $ A $ is aligned and $ U_h = I $, $ V_1 = I$. Applying the same transformation as in \eqref{eq:align-orth-transf} we now get $ A^\prime = P^\top A P = \block_1 \left( \Lambda_1, \Lambda_2, \ldots, \Lambda_{h-1}, \Lambda_h \right) $, which is block-shift diagonal, see Proposition~\ref{prop:decoupl-diag} below. By the same Proposition, if $ A(0) = P A^\prime(0) P^\top $ then $ A(t) = P A^\prime(t) P^\top $ where $ A^\prime(t) $ is block-shift diagonal for all $ t $, hence $ A(t) $ is decoupled for all $ t $.

The third statement follows since $ C = 0 $ is a conserved quantity, hence $ A(t) $ is 0-balanced for all $ t $.

\end{proof}

Since alignment is not preserved by the dynamics, in the next sections we focus on the properties of decoupled and balanced gradient flows.

\subsection{Decoupled dynamics}
\label{sec:decoup-dyn}
Decoupling refers to the singular modes of the gradient flow being decoupled, and from Proposition~\ref{prop:special-impl}, it is a property preserved by the dynamics. 
It was noted in \cite{saxe2013exact} that decoupling of the modes can be achieved if one can find an orthogonal change of basis that diagonalizes each weight matrix. Not any constant change of basis preserves the structure of the gradient flow dynamics, but in light of Proposition~\ref{prop:constants-equival}, we know that this happens if  $P $ is an orthogonal matrix.

\begin{proposition}
\label{prop:decoupl-diag}
Any decoupled $ A \in \mathcal{A}$ is block-shift diagonalizable through a change of basis $ P \in \mathcal{P_O}$, and its gradient flow is equivariant to a block-shift diagonal gradient flow.
If the block-shift SVD of $A $ is $ A = U \Lambda V^\top $, then the change of basis achieving diagonalization is $ P = \block_0(I,U_1, \ldots, U_{h-1}, I) \in \mathcal{P_O}$ and the block-shift diagonal matrix is $ \Lambda$.
\end{proposition}
\begin{proof}
Recall that $ A$ decoupled has $ U_i=V_{i+1}$, which implies that $P$ can be expressed also as $ P = \block_0(I, V_2, \ldots, V_{h-1}, I) $. 
Consequently, the block-shift SVD $ A = U \Lambda V^\top $ can also be written $ A = P \Lambda P^\top $, as shown by direct calculations. 
Since $ P \in \mathcal{P_O}$, from Proposition~\ref{prop:constants-equival}, we then have 
\[
\dot{A} = P  \left( \sum_{j=1}^h  \Lambda^{h-j \top } (E - \Lambda^h ) (\Lambda^{j-1})^\top  \right) P^\top = P \dot \Lambda P^\top.
\]
$ \Lambda$ is in block-shift diagonal form, and stays so for all $ t$. All its right and left singular vectors matrices are equal to the identity.

\end{proof}

\subsubsection{A special case of decoupled dynamics: block-shift diagonal gradient flow}
\label{sec:diag-evol}

Consider $A^{\diag}= \block_1(W_1, \ldots, W_h) \in  \mathcal{A} $ in block-shifted diagonal form, that is, in which the $ W_i $ are all diagonal matrices (for instance,  obtained diagonalizing some decoupled adjacency matrix, as in Proposition~\ref{prop:decoupl-diag}). If in \eqref{eq:grad-flowA} we choose $A^{\diag}(0) $ in which all $ W_i(0) $ are diagonal, it is straightforward to see that $ W_i(t)$ stay diagonal for all $t$, i.e., the set $ \mathcal{A}^{\diag} = \{ A= \block_1(W_i) \in  \mathcal{A} \; \text{s.t.}\; W_i(0) \; \text{diagonal}, i=1,\ldots, h \}$ is a forward invariant subset of $ \mathcal{A} $ for \eqref{eq:grad-flowA}.
Without loss of generality, the decoupled and diagonal dynamics can be represented by considering only the first $ d_y $ input-output connections, and by assuming that the $ j$-th input is connected only with the $j$-th output.
To simplify the notation, the last $ d_x - d_y $ inputs and the last $ d_i-d_y $ neurons in the hidden layers are disregarded, as they are disconnected from the outputs. 
With these conventions, the remaining elements of $ \mathcal{A}^{\diag} $ can be easily expressed in vectorized form. 
Denote $ \bm{\xi}_i =\begin{bmatrix}\xi_{i,1} & \ldots & \xi_{i,d_y} \end{bmatrix}^\top  = \diag\left( [W_i]_{1:d_y,1:d_y} \right)$, $i=1, \ldots, h$ (where the $ [W_i]_{1:d_y,1:d_y} $ is the restriction of $ W_i $ to its upper left $ d_y \times d_y $ block), $ \bm{\xi} = \begin{bmatrix}\bm{\xi}_{1}^\top & \ldots & \bm{\xi}_{h}^\top \end{bmatrix}^\top \in \mathbb{R}^{h d_y}$, and $ \bm{\sigma} =\diag(\Sigma)$.
The gradient flow \eqref{eq:grad-flowA} restricted to $ \mathcal{A}^{\diag} $ can then be written in vector form as 
\beq
\dot{\bm{\xi}}_i = \bar{\bm{\xi}}_i \circ (\bm{\sigma} - \hat{\bm{\xi}}) , \quad i=1, \ldots, h
\label{eq:grad-flow-xi}
\eeq
where
\begin{align*}
\bar{\bm{\xi}}_i & = \begin{bmatrix} \bar{\xi}_{i,1} \\ \vdots \\
\bar{\xi}_{i,d_y} 
\end{bmatrix}  \; \text{  with  } \; \bar{\xi}_{i,j} =\prod_{\substack{k=1 \\ k\neq i }}^h \xi_{k,j} , \quad 
i=1, \ldots, h, \; j=1 , \ldots, d_y 
\\
\hat{\bm{\xi}} & = \begin{bmatrix} \hat{\xi}_{1} \\ \vdots \\
\hat{\xi}_{d_y} 
\end{bmatrix}  \; \text{  with  } \; \hat{\xi}_{j} =\prod_{k=1 }^h \xi_{k,j} , \quad  j=1 , \ldots, d_y 
\end{align*}
and $\circ $ is the Hadamard product. In components, \eqref{eq:grad-flow-xi} reads
\[
\dot{\xi}_{i,j} = f_{i,j}(\bm{\xi}, \sigma_j) = \bar{\xi}_{i,j}( \sigma_j -  \hat{\xi}_{j} ), \quad i=1, \ldots, h, \; j=1 , \ldots, d_y .
\]
The $ (h-1) d_y $ conservation laws for \eqref{eq:grad-flow-xi} can be expressed as
\begin{align*}
   \pi_{i,j}(\bm{\xi}) =  \xi_{i,j}^2 - \xi_{i+1,j}^2 - c_{i,j} =0  \quad i=1, \ldots, h-1, \; j=1 , \ldots, d_y .
\end{align*}
The system \eqref{eq:grad-flow-xi} represents $ d_y $ decoupled input-output dynamics, in which $ \hat{\xi}_j$ expresses the connection from the $ j$-th input to the $j$-th output. 
Also the loss function can be written as a sum of terms corresponding to these input-output connections:
$ \mathcal{L}(\bm{\xi} ) = \frac{1}{2} \| \bm{\sigma} - \hat{\bm{\xi}} \|_2^2 = \sum_{j=1}^{d_y} \mathcal{L}_j( \bm{\xi}_{\bullet,j}) $ with $\mathcal{L}_j ( \bm{\xi}_{\bullet,j}) = \frac{1}{2}(\sigma_j- \hat{\xi}_j)^2$ and $ \bm{\xi}_{\bullet,j} = \begin{bmatrix} \xi_{1,j} & \ldots & \xi_{h,j} \end{bmatrix}^\top $ the vector of weights in the $j$-th connection.

\begin{proposition}
\label{prop:diag-crit}
If $ \bm{\xi}^\ast $ is a critical point of the gradient flow \eqref{eq:grad-flow-xi}, then for each $ j=1, \ldots, d_y $ either $ \hat{\xi}_j^\ast  = \sigma_j $ or $ \xi_{i,j}^\ast =0$  for at least 2 layers $ i=1, \ldots, h$.
\end{proposition}
\begin{proof}
We know already from Proposition~\ref{prop:dotL} that each trajectory of the gradient flow converges to a critical point. 
Assume that $ \dot{\xi}_{ij}=0$. Then either $ \hat{\xi}_j =  \xi_{1,j} \cdot \ldots \cdot  \xi_{h,j} = \sigma_j $, meaning that $ \dot{\xi}_{ij}=0$ for all $i = 1 , \ldots, h$, or $ \bar{\xi}_{i,j} = \xi_{1,j} \cdot \ldots \cdot \xi_{i-1,j} \xi_{i+1,j}\cdot \ldots \cdot  \xi_{h,j} =0$, meaning that at least one component $ \xi_{k_1,j}$, $ k_1\neq i$, must be equal to 0. 
Considering the layer $ k_1$, the same argument implies that $ \xi_{k_2,j}=0$ for some $ k_2\neq k_1$. Since at least one among $ \xi_{k_1,j}$ and $ \xi_{k_2,j}$ appears in the expressions for $ \bar{\xi}_{i,j}$ at the remaining layers, it means that $ \dot{\xi}_{i,j} =0 $ for all $ i=1, \ldots, h$.
As the same argument applies for all $ j=1, \ldots, d_y$, the result follows.    
\end{proof}
The interpretation of this result is that for each component $ j=1, \ldots, d_y$, at stationarity either the input-output expression is exact ($\hat{\xi}_j = \sigma_j $) or there is no input-output relation at all ($ \hat{\xi}_j=0$).
On $ \mathcal{A}_\phi $, the $ 2^{d_y}$ critical points correspond to all possible on/off choices in these $ d_y $ input-output  connections. 
For instance, for the $ j$-th input-output pair, the infinitely-many critical points corresponding to all possible choices of $ \xi_{1,j}, \ldots, \xi_{h,j}$ s.t. $ \hat{\xi}_j=\sigma_j $ (or $ \hat{\xi}_j=0$) are all collapsed into a single critical point in $ \mathcal{A}_\phi$.

Denote $ f(\bm{\xi}, \bm{\sigma}) =\begin{bmatrix} f_{1,1} (\bm{\xi},\sigma_1) & \ldots & f_{1,d_y}  (\bm{\xi},\sigma_{d_y}) &  f_{2,1} (\bm{\xi},\sigma_1) & \ldots & f_{h,d_y} (\bm{\xi},\sigma_{d_y})  \end{bmatrix}^\top  \in \mathbb{R}^{h d_y } $ the vector field \eqref{eq:grad-flow-xi}. 
$ f(\bm{\xi}, \bm{\sigma}) $ can be considered a function of the data, i.e., of the singular values $ \sigma_j$, $ j=1, \ldots, d_y$. As we vary $ \bm{\sigma}$, $ f(\bm{\xi}, \bm{\sigma}) $ spans a $ d_y$-dimensional vector space at each $ \bm{\xi}$, $ {\rm span} ( f(\bm{\xi},\bm{\sigma})) $, which is clearly involutive (the Lie bracket is $ [  f(\bm{\xi}, \bm{\sigma}) , \,  f(\bm{\xi}, \bm{\sigma}') ]=0$ for all $ \bm{\sigma}' \neq \bm{\sigma} $). 
Differentiating the conservation laws we get $ (h-1) d_y $ exact one-forms $ d \pi_{i,j}(\bm{\xi}) = 2 \xi_{i,j} - 2 \xi_{i+1, j}$, which are independent among each other and orthogonal to $ f(\bm{\xi}, \bm{\sigma}) $: 
\[
d \pi_{ij}(\bm{\xi}) f(\bm{\xi}, \bm{\sigma}) = 2( \xi_{i,j} f_{i,j}(\bm{\xi}, \sigma_j) - \xi_{i+1,j} f_{i+1,j}(\bm{\xi},\sigma_j)) = 2( \xi_{i,j} \bar{\xi}_{i,j}  - \xi_{i+1,j} \bar{\xi}_{i+1,j} )(\sigma_j - \hat{\xi}_j)=0.
\]
Hence if $ \Delta_\pi(\bm{\xi}) = {\rm span} (d \pi_{i,j} ) $, $ \dim(\Delta_\pi(\bm{\xi})) = (h-1)d_y $, $ \Delta_\pi(\bm{\xi}) \perp {\rm span} ( f(\bm{\xi},\bm{\sigma}))  $ and $ \Delta_\pi(\bm{\xi}) \oplus {\rm span} ( f(\bm{\xi},\bm{\sigma}))  = \mathbb{R}^{h d_y}$.

The derivative of the loss function is $
\dot{\mathcal{L}}(\bm{\xi}) = \sum_{j=1}^{d_y} \dot{\mathcal{L}}_j( \bm{\xi}_{\bullet,j}) $, with
\beq
\label{eq:dotLj-xi}
\dot{\mathcal{L}}_j ( \bm{\xi}_{\bullet,j}) = - 2 \left( \sum_{i=1}^{h} \bar{\xi}_{i,j}^2 \right) \mathcal{L}_j ( \bm{\xi}_{\bullet,j}) .
\eeq
When at least one of the terms of the sum in \eqref{eq:dotLj-xi} is nonzero along a trajectory, then $\mathcal{\dot{\mathcal{L}}}(\bm{\xi}) \leq \dot{\mathcal{L}}_j( \bm{\xi}_{\bullet,j}) < 0$ and the convergence is with exponential rate.
From Proposition~\ref{prop:diag-crit}, a saddle point $ \bm{\xi}^\ast $ is characterized by $ \xi_{i,j}^\ast =0 $ for at least two layers $ i=1, \ldots, h$. In particular, if $ \bm{\xi}^\ast \neq 0$, then it follows from the argument used in the proof of Proposition~\ref{prop:diag-crit} that there exists trajectories of \eqref{eq:grad-flow-xi} characterized by 2 or more $ \xi_{i,j}=0$ for all times $ t$, and the sufficient conditions of Propositions~\ref{prop:lojasiewicz-bound} and \ref{prop:conv-rate-sp-case1} for exponential convergence fail to be satisfied. Nevertheless, $ \bm{\xi} $ is still exponentially converging to the saddle point $ \bm{\xi}^\ast$. If for a trajectory 2 or more $ \xi_{i,j}=0$ for all times $ t$, also in Proposition~\ref{prop:conv-rate-general} one obtains $\alpha_{(j)}=0$, giving a trivial lower bound $\alpha_h = 0$ on the convergence rate.

Let $\bm{\xi}_\bullet = \begin{bmatrix} \bm{\xi}_{\bullet,1}^\top & \ldots & \bm{\xi}_{\bullet,d_y}^\top \end{bmatrix}^\top $ be a reordering of the elements of $\bm{\xi}$. Consider now the Hessian of $\mathcal{L}$ restricted to $ \mathcal{A}^{\diag} $, written $ \nabla_{\bm{\xi}_\bullet}^2 \mathcal{L} (\bm{\xi}_\bullet) $. It is clear that $\frac{\partial^2 \mathcal{L}}{\partial \xi_{\bullet,j} \partial \xi_{\bullet,\ell}} = 0$ for different input-output connections $j\neq \ell$, thus $ \nabla_{\bm{\xi}_\bullet}^2 \mathcal{L} $ is a block-diagonal matrix consisting of $d_y$ blocks
\beq
    \label{eq:Lj-hessian}
    \nabla_{\bm{\xi}_{\bullet,j}}^2 \mathcal{L}_j = \bar{\bm{\xi}}_{\bullet,j} \bar{\bm{\xi}}_{\bullet,j}^\top + (\hat{\xi}_j - \sigma_j) \nabla_{\bm{\xi}_{\bullet,j}}^2 \hat{\xi}_j, \quad j=1,\ldots,d_y.
\eeq
The Hessian is positive semidefinite if and only if all blocks $ \nabla_{\bm{\xi}_{\bullet,j}}^2 \mathcal{L}_j $ are positive semidefinite.
From Proposition~\ref{prop:diag-crit}, we know that a saddle point $ \bm{\xi}_\bullet^\ast $ is characterized by, for each $j=1,\ldots,d_y$, either $\sigma_j = \hat{\xi}_j^\ast$, or $ \xi_{i,j}^\ast =0 $ for at least two layers $ i=1, \ldots, h$. If the former is true, the corresponding block $\nabla_{\bm{\xi}_{\bullet,j}}^2 \mathcal{L}_j = \bar{\bm{\xi}}_{\bullet,j}^\ast \bar{\bm{\xi}}_{\bullet,j}^{\ast\top} $ is always positive semidefinite, so we will consider the components for which the latter holds (there exists at least one such block, otherwise we are at a global optimum). For such a component $j$, we instead have $\bar{\bm{\xi}}_{\bullet,j}^\ast \bar{\bm{\xi}}_{\bullet,j}^{\ast\top} = 0$.
If more than two layers $ i=1, \ldots, h$ fulfill $ \xi_{i,j}^\ast =0 $, then $ [ \nabla_{\bm{\xi}_{\bullet,j}}^2 \hat{\xi}_j^\ast ]_{k,\ell} = \prod_{i\neq k,\ell} \xi_{i,j}^\ast = 0$, for all $k,\ell=1,\ldots,h$, and $ \nabla_{\bm{\xi}_{\bullet,j}}^2 \mathcal{L}_j = 0 $, and if all other blocks $ \nabla_{\bm{\xi}_{\bullet,k}}^2 \mathcal{L}_k$ are positive semidefinite, the saddle point is non-strict.
If exactly two weights $ \xi_{k,j}^\ast = \xi_{\ell,j}^\ast =0 $, then $ \nabla_{\bm{\xi}_{\bullet,j}}^2 \mathcal{L}_j( \bm{\xi}_\bullet^\ast ) = (\hat{\xi}_j^\ast - \sigma_j)(\bm{e}_k \bm{e}_\ell^\top + \bm{e}_\ell \bm{e}_k^\top ) \prod_{i\neq k,\ell} \xi_{i,j}^\ast $ which has two non-zero eigenvalues $\pm (\hat{\xi}_j^\ast - \sigma_j) \prod_{i\neq k,\ell} \xi_{i,j}^\ast$ with corresponding eigenvectors $\bm{e}_k+\bm{e}_\ell$ and $\bm{e}_k-\bm{e}_\ell$. Hence, all such saddle points are strict.
As mentioned in Remark~\ref{rem:diag-case}, when restricting the Hessian to $ \mathcal{A}^\diag $ we can have non-strict saddle points associated to $A^\ast$ which are not minimizers of the rank constrained regression problem, i.e., non-strict saddle points with $\mathcal{S}\neq \{1,...,r\}$ in item 2 of Proposition~\ref{prop:class-crit-points}. Notice that the direction $V$, used in Proposition~\ref{prop:var-neg-Hessian} to show strictness of a saddle point for which $\mathcal{S}\neq \{1,...,r\}$, does not belong to $ \mathcal{A}^\diag $.
The considerations of Theorem~\ref{thm:arcs} about the monotonic character of the arcs $ \mathcal{A}(\alpha)$ apply unchanged also here.

\subsubsection{Hessian of a decoupled dynamics}

In the section above, for block-shift diagonal networks, i.e., $ A^\diag, V^\diag \in \mathcal{ A }^\diag $, we obtained conditions on the $\xi_{i,j}$:s under which the Hessian is positive semidefinite or has a negative eigenvalue. From Proposition~\ref{prop:decoupl-diag}, the diagonal case is a special case of decoupling, and the same conditions hold also for a decoupled network. Let $ A^\prime = P A^\diag P^\top $ with $ P \in \mathcal{P_O} $, and consider $ \rho_k(A^\prime,V^\prime) $. By Lemma~\ref{lem:equivariant_rho},
\[
\rho_k(A^\diag, V^\diag ) = \rho_k(P^\top A^\prime P, V^\diag) = \rho_k(A^\prime , P V^\diag P^\top),
\]
so for the variation $ V^\prime $ to be equivalent to a diagonal variation it must be $ V^\prime = P V^\diag P^\top $. It is immediate that $ \rho_k(A^\diag, V^\diag ) = \rho_k(A^\prime,V^\prime) $, hence $ \mathfrak{h}(A^\diag, V^\diag ) = \mathfrak{h}(A^\prime,V^\prime) $, and the conditions on the $\xi_{ij}$s in the previous section translate into conditions on the singular values $[\Lambda_{i}]_{jj}$ of the decoupled $A^\prime$.

\subsection{Balanced dynamics}
\label{sec:balance}

Consider a network $ A =\block_1(W_1, \ldots, W_h)\in \mathcal{ A } $ which is 0-balanced.
As the next proposition shows, a consequence of 0-balance is that the singular values of the blocks $ W_i = U_i \Lambda_i V_i^\top $ are identical and the associated singular vectors aligned. The slight abuse of notation $ \Lambda_i = \Lambda_j $ used in this proposition should be understood as the entries on the main diagonal being equal $ [\Lambda_i]_{kk} = [\Lambda_j]_{kk} $, for $ k = 1,..., d_y $, and $ [\Lambda_i]_{kk} = [\Lambda_j]_{kk} = 0 $ for $ k > d_y $ whenever the elements exist.
As before, let $ A = U \Lambda V^\top$ be the block-shift SVD of $A$.

\begin{proposition}
\label{prop:0-balanced}
If $ A \in \mathcal{ A } $ is 0-balanced, then $A$ is aligned but not necessarily decoupled.
Also, the singular values in all layers are identical: $ \Lambda_i = \Lambda_j$ for all $ i, j= 1, \ldots, h$, and for all $t$.
 All 0-balanced trajectories of the gradient flow converge to decoupled critical points as $ t \rightarrow \infty $.
In particular, for a 0-balanced critical point $A^\ast $ it is  $A^\ast = P \Lambda P^\top  $ for some $ P\in \mathcal{P_O}$, and $ (A^\ast)^h =\Lambda^h$.
\end{proposition}

\begin{proof}
Consider the SVDs $ W_i = U_i S_i V_i^\top $. The condition $ W_i W_i^\top = W_{i+1}^\top W_{i+1} $ implies $U_i \Lambda_i \Lambda_i^\top U_i^\top = V_{i+1} \Lambda_{i+1}^\top \Lambda_{i+1} V_{i+1}^\top $.
Notice that $ \rank ( W_i ) = r$, $ \forall i $, otherwise for some $ i $ it is $ W_i W_i^\top \neq W_{i+1}^\top W_{i+1} $.
From this, it is clear that the nonzero singular values of $ \Lambda_i $ and $ \Lambda_{i+1} $ must be identical, and we can pick an SVD such that $ U_i = V_{i+1} $ for all $ i = 1,...,h-1 $, including when some singular values in $ \Lambda_i $ are equal.
Moreover, $ U_i $ is orthogonal and thus $ V_{i+1}^\top U_i = U_i^\top U_i = I $. Assembling the block-shift form, we have $ V^\top U = J$, i.e., alignment. Since $ U_h $ and $ V_1 $ are left out of these identities, $A$ need not be decoupled.
From Proposition~\ref{prop:special-impl}, any 0-balanced $ A $ is aligned for all $ t $, because $ C = 0$ is a constant of motion. 
Also the argument $ \Lambda_i = \Lambda_j $ $ \forall \, i,j = 1, \ldots, h$, holds for all $t$.
As $ t \rightarrow \infty $ the trajectory under the gradient flow approaches a critical point $ A^\ast $ such that $ ( A^\ast )^h = \block_h( U_h^\ast \Lambda_{1:h}^\ast V_1^{\ast \top} )=  \block_h( S \Sigma ) $ with $ S = \diag(\mathbf{s}) $, $ \mathbf{s} = \begin{bmatrix}
	s_1 & \cdots & s_{d_y}
\end{bmatrix}^\top $, $ s_i \in \{0,1\} $, implying that $ U_h^\ast = I $, $ V_1^\ast = I $ since $ S \Sigma $ is diagonal.
Since $ \Lambda_i = \Lambda_j$, it is also $ \Lambda_{1:h}=\Lambda_i^h$, hence $(A^\ast)^h = \block_h(\Lambda_i^h) = \Lambda^h$.
The critical point $ A^\ast $ is 0-balanced and decoupled, and must therefore have a block-shift SVD $A^\ast = P \Lambda P^\top  $ in which $ P=\block_0 (I, P_1, \ldots, P_{h-1}, I) \in \mathcal{P_O}$, see Proposition~\ref{prop:decoupl-diag}.
\end{proof}

Notice that the condition that singular vectors are aligned, $ V_{i+1}^\top U_i = I $, for $ i=1,...,h-1 $, is not sufficient for 0-balance, not even in the static case. In fact, if any $ W_i $ and $ W_{i+1} $ have distinct singular values then $ U_i \Lambda_i \Lambda_i^\top U_i^\top - V_{i+1} \Lambda_{i+1}^\top \Lambda_{i+1} V_{i+1}^\top \neq 0 $, hence $ C \neq 0 $.

A 0-balanced network is tailored to contain the least amount of redundancy possible for expressing the input-output relationship of the data. 

\begin{proposition}
\label{prop:0-bal-crit-point}
Consider a 0-balanced critical point of $ \mathcal{L}$, $ A^\ast \in \mathcal{A}$.
In the representation \eqref{eq:A-param1}-\eqref{eq:Z-param1}, if $ A^\ast = P( A_1 + Z ) P^{-1} $ then it is $ Z = 0 $.
\end{proposition}
\begin{proof}
If $A^\ast =\block_1(W_i^\ast)$ is a 0-balanced critical point and $ (A^\ast)^h $ is of rank $ r $, then from $ W_i W_i^\top - W_{i+1}^\top W_{i+1} =0$, all blocks of $ A^\ast $ have the same rank: $ \rank(W_i) = \rank(W_j)  \geq r $. From Proposition~\ref{prop:0-balanced}, it is also $ (A^\ast)^h = \Lambda^h $, so in fact $ \rank(W_i^\ast) = r $. Since $ Z $ is complementary to $ A_1 $, any $ Z_i \neq 0 $ would increase the rank of all components $ W_i^\ast$, and thereby of $ \Lambda^h $ which would then have rank greater than $ r $, leading to a contradiction.
\end{proof}
The meaning of Proposition~\ref{prop:0-bal-crit-point} is that 0-balanced critical points are as tightened as possible (i.e., $ Z=0$ in \eqref{eq:Z-param1}).

It is straightforward to generate random 0-balanced initializations, see for instance Procedure 1 in \cite{arora2018convergence}. Any orthogonal change of basis can be applied to the result of this algorithm to obtain an equivalent 0-balanced system. A more realistic initialization with similar properties to 0-balance is also considered in \cite{arora2018convergence}, to which we turn in the following.

\subsubsection{Approximate balance}
\label{sec:delta-bal}

One can obtain a behavior similar to that which is exhibited under 0-balance by using a small random initialization. This way of initializing weights is more general and it is also commonly used for nonlinear neural networks \cite{glorot2010understanding}.

A network $ A = \block_1(W_1, \ldots, W_h) \in \mathcal{ A } $ is said to be approximately balanced ($ \delta $-balanced, cf. \cite{arora2018convergence}) if for all $ i $ it holds that $ \| W_i W_i^\top - W_{i+1}^\top W_{i+1} \|_F \leq \delta $ for some small $ \delta>0$. A consequence of this is that the nonzero singular values of each block are approximately identical as the system evolves over time.

\begin{proposition}\label{prop:delta-balance}
	Let $ A $ be $ \delta $-balanced and denote the SVD of each block $ W_i = U_i \Lambda_i V_i^\top $ with $ \varsigma_{i,k} = [\Lambda_i]_{kk} $. Then for all $ i = 1,...,h-1 $, 
	\begin{equation}
		\delta^2 \geq \left\| \Lambda_i \Lambda_i^\top - \Lambda_{i+1}^\top \Lambda_{i+1} \right \|_F^2 = \sum_{k=1}^{d_i} (\varsigma_{i,k}^2 - \varsigma_{i+1,k}^2)^2. 
        \label{eq:dbal-singval}
	\end{equation}
\end{proposition}

\begin{proof}
	The statement follows from the Wielandt-Hoffman theorem (see Theorem 8.1.4 of \cite{golub2013matrix}):
	\begin{align*}
		\delta^2 & \geq  \| C_{i} \|_F^2 = \| W_i W_i^\top - W_{i+1}^\top W_{i+1} \|_F^2 \\&\geq \sum_{k=1}^{d_i} (\lambda_k( W_i W_i^\top ) - \lambda_k( W_{i+1}^\top W_{i+1} ))^2  = \| \Lambda_i \Lambda_i^\top - \Lambda_{i+1}^\top \Lambda_{i+1} \|_F^2 = \sum_{k=1}^{d_i} (\varsigma_{i,k}^2 - \varsigma_{i+1,k}^2)^2.
	\end{align*}
\end{proof}
In practice, Proposition~\ref{prop:delta-balance} implies that for generic initializations $ A(0) $ close enough to the origin ($ \delta $-balanced), the singular values of all blocks ``synchronize'' during learning, i.e., they either simultaneously increase, or stay close to zero. 
From \eqref{eq:dbal-singval} it also follows that $ \lvert \varsigma_{i,k}^2 - \varsigma_{i+i,k}^2 \rvert = ( \varsigma_{i,k} + \varsigma_{i+i,k} ) \lvert \varsigma_{i,k} - \varsigma_{i+i,k} \rvert \leq \delta $ for all $ i = 1,...,h-1$ and $ k = 1,..., d_i $, i.e., the larger $ \varsigma_{i,k} $ and $ \varsigma_{i+i,k} $ are with respect to $ \delta$, the closer they must be.
In particular, if $ \delta\ll \sigma_{d_y}^{1/h} $, then when the $ \sigma_k$ singular value of $ \Sigma $ is learnt (i.e., $ \varsigma_{i,k}, \varsigma_{i+1,k} \approx \sigma_k^{1/h} $), it must be $ \lvert \varsigma_{i,k} - \varsigma_{i+1,k} \rvert < \frac{\delta}{2 \sigma_k^{1/h} } \ll \delta $.

It is also possible to gain some insight about the singular vectors from the conservation laws. 
If $ \delta $ is small enough, such that $ \varsigma_{i+1,k}^2 \approx \varsigma_{i,k}^2 $ is a good approximation, then we have
\begin{align}
    \delta^2 & \geq \| C_{i} \|_F^2 = \| W_i W_i^\top \|_F^2 + \| W_{i+1} W_{i+1} ^\top \|_F^2 - 2 \| W_{i+1} W_i \|_F^2 \nonumber \\
    & = \sum_{k=1}^{d_i} \varsigma_{i,k}^4  + \sum_{k=1}^{d_i} \varsigma_{i+1,k}^4 - 2 \| U_{i+1} \Lambda_{i+1} V_{i+1}^\top U_i \Lambda_i V_i^\top  \|_F^2  \nonumber \\
    & = \sum_{k=1}^{d_i} ( \varsigma_{i,k}^4  + \varsigma_{i+1,k}^4 )  - 2 \| \Lambda_{i+1} V_{i+1}^\top U_i \Lambda_i  \|_F^2 \nonumber \\
    & = \sum_{k=1}^{d_i} ( \varsigma_{i,k}^4  + \varsigma_{i+1,k}^4 ) - 2 \sum_{k,\ell=1}^{d_i} \varsigma_{i,k}^2 \varsigma_{i+1,\ell}^2 (v_{i+1,\ell}^\top u_{i,k})^2 \nonumber \\
    & \approx 2 \sum_{k=1}^{d_i} \varsigma_{i,k}^4 - 2 \sum_{k,\ell=1}^{d_i} \varsigma_{i,k}^2 \varsigma_{i,\ell}^2 (v_{i+1,\ell}^\top u_{i,k})^2 \nonumber \\
    & = 2 \boldsymbol{\mu}_i^\top ( I - \bar{\Theta} ) \boldsymbol{\mu}_i = 2 \boldsymbol{\mu}_i^\top \left( I - \frac{\bar{\Theta} + \bar{\Theta}^\top}{2} \right) \boldsymbol{\mu}_i \label{eq:dbal-lapl-deriv}
\end{align}
where $ \boldsymbol{\mu}_i = [ \varsigma_{i,1}^2 , \cdots , \varsigma_{i,d_i}^2 ]^\top $, and $ \bar{\Theta} $ is a doubly stochastic matrix with elements $ [\bar{\Theta}]_{k,\ell} = (v_{i+1,k}^\top u_{i,\ell})^2 $. This property of $ \bar{\Theta} $ follows from the fact that $ U_i $ and $ V_{i+1} $ are orthogonal matrices with $ u_{i,k} $ and $ v_{i+1,k}^\top $ as their $ k $-th columns, respectively, so that for all $ k $ we have $ \| v_{i+1,k}^\top U_i \|_2 = \sum_{\ell = 1}^{d_i} (v_{i+1,k}^\top u_{i,\ell})^2 = 1 $ and $ \sum_{\ell = 1}^{d_i} (v_{i+1,\ell}^\top u_{i,k})^2 = 1 $, which are the row and column sums of $ \bar{\Theta} $. The last equality in \eqref{eq:dbal-lapl-deriv} follows from transposing the scalar $ \boldsymbol{\mu}_i^\top \bar{\Theta} \boldsymbol{\mu}_i = (\boldsymbol{\mu}_i^\top \bar{\Theta} \boldsymbol{\mu}_i)^\top = \boldsymbol{\mu}_i^\top \bar{\Theta}^\top \boldsymbol{\mu}_i $. Now, $ \Theta = ( \bar{\Theta} + \bar{\Theta}^\top ) / 2 $ is also doubly stochastic since it is a convex combination of doubly stochastic matrices. Let $ L = I - \Theta $ be the Laplacian matrix associated with $ \Theta $. Then, from \eqref{eq:dbal-lapl-deriv},
\begin{equation}
    \delta^2 \geq 2\boldsymbol{\mu}_i^\top L \boldsymbol{\mu}_i = \sum\limits_{k,\ell=1}^{d_i} [\Theta]_{k,\ell} (\varsigma_{i,k}^2 - \varsigma_{i,\ell}^2 )^2. \label{eq:dbal-laplacian}
\end{equation}
Here $ [\Theta]_{k,\ell} $ for $ k \neq \ell $ quantifies the ``non-alignment'' of singular vectors in $ V_{i+1} $ and $ U_i $.
Both $ \Theta $ and $ \boldsymbol{\mu}_i $ are time-varying under gradient flow.

As mentioned above, a $ \delta $-balanced initialization generically manifests as synchronized singular values at all blocks, either simultaneously increasing or staying close to zero. 
For the singular values that increase, the corresponding singular vectors approximately align as this happens. More specifically, \eqref{eq:dbal-laplacian} suggests that we have the following cases for the singular vectors of $ A $:
\bite
    \item $ \varsigma_{i,k}^2 \approx \varsigma_{i+1,k}^2 <  \delta^2 $, then no constraint is imposed on $ u_{i,k} $ and $ v_{i+1,k} $ for any $ i = 1,...,h-1$;
    \item $ \varsigma_{i,k}^2 \approx \varsigma_{i+1,k}^2 > \delta^2 $ then $ v_{i+1,k}^\top u_{i,\ell} \approx 0 $ for any $ \ell $ such that $ \lvert \varsigma_{i,k}^2 - \varsigma_{i,\ell}^2 \rvert \geq \delta $. In particular, if for some $ k $ it holds that $ \lvert \varsigma_{i,k}^2 - \varsigma_{i,\ell}^2 \rvert \geq \delta $ for all $ \ell \neq k $, then $ [\Theta]_{k,k} = ( v_{i+1,k}^\top u_{i,k} )^2 \approx 1 $. 
\eite
In particular, if some singular values $ \sigma_k $ and $ \sigma_\ell $, $ k \neq \ell$, of $ \Sigma $ are close enough to each other, the singular values learnt by the blocks will also be close, i.e., $ \varsigma_{i,k}^2 - \varsigma_{i,\ell}^2 $ is small. This allows $ [\Theta]_{k,\ell} = \frac{1}{2} \left( (v_{i+1,\ell}^\top u_{i,k})^2 + (v_{i+1,k}^\top u_{i,\ell})^2 \right) $ to be large in \eqref{eq:dbal-laplacian}, meaning that $ v_{i+1,k} $ and $ u_{i,k} $ may be far from aligned. The singular vectors corresponding to close singular values span a subspace approximately orthogonal to all other singular vectors, but they are themselves not necessarily aligned. On the other hand, if all singular values are well separated, the corresponding $ d_y $ singular vectors are nearly aligned, i.e., $ [V_{i+1}^\top U_i]_{1:d_y,1:d_y} \approx I_{d_y} $. Examples of $ \delta $-balanced gradient flows are shown in Section~\ref{sec:ex-qualitat-anal} (see in particular Fig.~\ref{fig:dbalance}).

In light of Proposition~\ref{prop:0-bal-crit-point}, we have the following interpretation of sequential learning in the $ \delta$-balanced case. 
As is often mentioned in the literature \cite{saxe2013exact,arora2019implicit,cohen2024lecture}, and as we will also see in simulation in Section~\ref{sec:examples}, sequential learning proceeds from the largest to the smallest singular value of the data. 
When $ A(t)$ learns the first singular value $ \sigma_1 $, then, because of $ \delta$-balancedness, it has to pass near $ A^\ast = P (A_1 + Z) P^{-1}$ of signature $ \mathcal{S}=\{ 1\}$ (a tightened, and hence non-strict critical point, see Proposition~\ref{prop:class-crit-points}), and characterized by $ Z $ near 0 (i.e., $ \| Z \|_F $ small). This it tantamount to say that $ A(t)$ passing near $ A^\ast $ is near-tightened. In fact, recalling \eqref{eq:second_ord_terms}, the negative term of the Hessian quadratic form 
\[
\tr{(\rho_2(A,V)(E - A^h)^\top)} = \tr{ \left( \Sigma_\mathcal{Q}^\top \left[ \sum_{k=\ell_3}^{h} Z_h \cdots Z_{k+1} V^{(k)}_{21} \right]\left[ \sum_{j=1}^{\ell_1} V^{(j)}_{12} Z_{j-1}\cdots Z_{1} \right]  \right)}
\]
is small when $ Z $ is small enough. What dominates $ \mathfrak{h}(A,V) $ is therefore the normed term, i.e., $ \mathfrak{h}(A,V) \approx \lVert \rho_1(A,V)  \rVert_F^2 $, suggesting that in many directions the landscape is flat or has positive curvature. $A(t)$ follows therefore a plateau, until it learns $ \sigma_2 $. Still owing to $ \delta$-balancedness, this means that $A(t)$ is passing near a critical point $ A^{\ast \ast} = P (A_2 + Z) P^{-1}$ of signature $ \mathcal{S}=\{ 1, 2\}$, another non-strict saddle point with $ Z$ small. The procedure is then repeated until completion.

While this pattern happens generically, it does not happens for all $ \delta$-balanced initial conditions. For instance in Section~\ref{sec:diag-bal} we show a case in which 0-balanced and decoupled initial conditions fail at achieving sequential learning. 
It is at the moment not clear if sequential learning occurs whenever all singular modes are coupled in the dynamics (i.e., decoupling is missing).

\subsection{Combining balanced and diagonal dynamics}
\label{sec:diag-bal}

Assume now that on the diagonal $ W_i $ of Section~\ref{sec:diag-evol} we want to impose the 0-balanced condition of Section~\ref{sec:balance}: $ W_i W_i^\top = W_{i+1}^\top W_{i+1} $, $ i=1, \ldots, h-1$. Since this holds for all $t$, combining Proposition~\ref{prop:0-balanced} with the diagonality of $ W_i $ we have that all layers have the same singular values and singular vectors for all $t$.
Using the notation of Section~\ref{sec:diag-evol}, we have that $ \xi_{i,j} ^2 = \xi_{k,j}^2 $ for all $ i, k = 1, \ldots ,h$, and $ j=1,\ldots, d_y$, i.e., all diagonal entries of $ W_i $ and $ W_k $ are equal up to sign. We can express this by saying that  $ |\xi_{i,j} | = z_j \geq 0 $ is the $j$-th diagonal component of all weight matrices. 
Let us consider for simplicity the case in which all $ \xi_{i,j}\geq 0 $. Then the gradient flow equations \eqref{eq:grad-flow-xi}, which are already decoupled along the input-output paths, become the collection of $ d_y $ independent scalar ODEs
\beq
\dot z_j = z_j^{h-1} (\sigma_j - z_j^h ) , \qquad j=1, \ldots, d_y 
\label{eq:grad-flow-z}
\eeq
or, in vector form, $ \dot{\mathbf{z}} = \mathbf{z}^{h-1} \circ ( \bm{\sigma} - \mathbf{z}^h) $. 
Each ODE is decoupled from the others, meaning that each initial condition $ z_j(0)$ can be chosen independently from the others, while respecting the overall 0-balance structure.
The scalar ODE \eqref{eq:grad-flow-z} corresponds to a generalized logistic equation  \cite{blumberg1968logistic,tsoularis2002analysis}, and reproduces several of the features that can be found in the gradient flow \eqref{eq:grad-flowA}:

\bite
\item The graph of \eqref{eq:grad-flow-z}, drawn in Fig.~\ref{fig:diag-bal}(a), has 3 critical points $ \pm \sigma_j $ and $0$. Both $ \pm \sigma_j $ are locally asymptotically stable, while $0$ is unstable.
\item On both sides of  $\sigma_j$, the two terms of \eqref{eq:grad-flow-z} tend to work against each other, and, in spite of the different homogeneity order, ensure the asymptotic stability of $ \sigma_j$. This feature is standard to all logistic models.
\item When $ h>2$, the Hessian of \eqref{eq:grad-flow-z}, $ \nabla_{z_j}^2\mathcal{L}(z_j) = (h z_j^h -(h-1) (\sigma_j -  z_j^h)) z^{h-2} $, is homogeneous in $ z_j$, and hence it is vanishing at the critical point $ z_j=0$.
\item As shown in Fig.~\ref{fig:diag-bal}(a), around the origin there is a nearly flat (but not exactly flat, see lower panel Fig.~\ref{fig:diag-bal}(a)) plateau, meaning that escaping from a neighborhood of the origin may require a long time. Convergence to the critical point is instead fast when the initial condition is larger than the  singular value $ \sigma_j$. 
\item Sequential learning, from the largest to the smallest singular value,  occurs near the origin but is only guaranteed when the various input-output ``channels'' have identical initial condition: if $ \sigma_j > \sigma_k $ and $ z_j(0) = z_k(0) $ are both small ($ < \sigma_k^{1/h}$), then $ \sigma_j $ is discovered before $ \sigma_k $, see upper panel in Fig.~\ref{fig:diag-bal}(b). 
\item If $ z_j(0) \neq z_k (0) $, then there is no guarantee that sequential learning occurs, even though 0-balance is respected, see lower panel in Fig.~\ref{fig:diag-bal}(b).
\item Increasing $ h$ increases the steepness of $ \mathcal{L}(z_j)$ and the sharpness of the transition when $ z_j $ exits the plateau and approaches $ \sigma_j$ from below.
\eite
Most of these features are found also in the models we simulate in the next section.

\begin{figure}[htb!]
\centering
\subfigure[]{ 
\includegraphics[trim=3.5cm 8.5cm 10.5cm 8.5cm, clip=true, width=5cm]{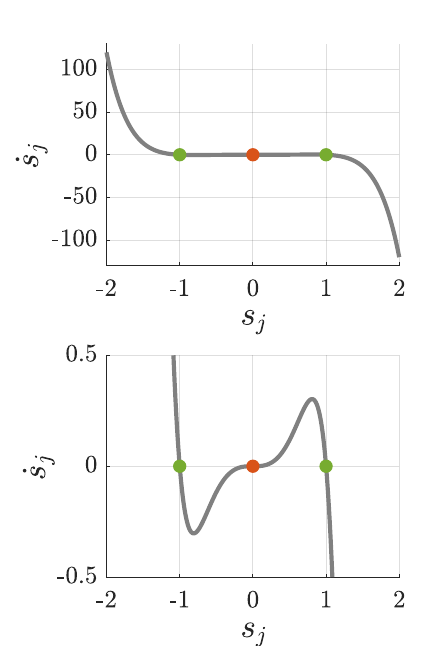}}
\subfigure[]{ 
\includegraphics[trim=3.5cm 8.5cm 10.5cm 8.5cm, clip=true, width=5cm]{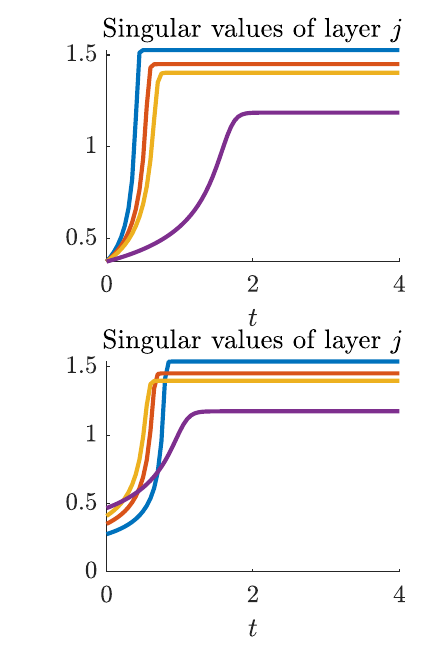}}
\caption{Example in Section~\ref{sec:diag-bal}. (a): Graph $( s_j, \dot{s}_j) $ of the scalar equation~\eqref{eq:grad-flow-z}. Green dots are stable critical points. The origin (in red) is unstable. The lower panel is a zoomed-in version of the upper panel. (b): Singular values of any layer along two gradient flow trajectories. On the upper panel $ s_j(0)=s_k(0) $ and the learning is sequential. On the lower panel $ s_j(0) \neq s_k (0) $ and the learning is not sequential. }
\label{fig:diag-bal}
\end{figure}

\section{Examples}
\label{sec:examples}

\subsection{Qualitative analysis of the gradient flow ODEs}
\label{sec:ex-qualitat-anal}

We consider in this example a deep linear neural network with $h=5 $ layers of dimension $ d_x =11$, $ d_1=\ldots = d_4 =6$ and $ d_y =4 $. 
The state space consists of matrices $A \in \mathcal{A}\subset \mathbb{R}^{39 \times 39} $ having $ d_x d_1 + d_1 d_2 + \ldots + d_{4} d_y =198$ variables and  $  \frac{1}{2} \left(  d_1(d_1+1) + \ldots + d_{4} (d_4+1)\right) =84 $ conservation laws.

In Fig.~\ref{fig:sim1}, we consider two trajectories of the gradient flow \eqref{eq:grad-flowA}, one initialized near the origin (Fig.~\ref{fig:sim1}(a)) and the other away from the origin (Fig.~\ref{fig:sim1}(b)). Notice that $A(0)$ close enough to the origin corresponds to the $ \delta$-balanced case discussed in Section~\ref{sec:delta-bal}.
As can be seen from the time scales, when $A(0)$ is near the origin, the convergence rate is typically much smaller, because the loss landscape is much more flat than away from the origin. 
Observe that $A=0$ is the ``most non-strict'' critical point, in the sense that for it all terms of the Taylor expansion are equal to 0, even though we know from Theorem~\ref{thm:arcs} that it is also the critical point with the largest number of unstable submanifolds. 
When $ \| A(0) \|_F $ is small, along the learning trajectory the singular values are typically learnt one by one, from the largest to the smallest one. Once learnt, they are no longer modified by the gradient flow, and neither are the associated singular vectors. This sequential learning is clearly reminiscent of a ``principal component analysis'' approach, as pointed out already in \cite{baldi1989neural}.
The behavior reminds also of the Eckart-Young theorem, which states that the best rank-$r $ approximation of a matrix is given by the singular value decomposition truncated at order $ r $. For small initial conditions, the gradient flow seems to apply this rule in learning the singular values of $A^h$, and this reflects also in the singular values of $A$. In other words, no unnecessary components (singular values not visible in the output) are learnt by the weight matrices $W_i$, which rhymes with the ``greedy low-rank'' implicit regularization principle mentioned e.g. in \cite{arora2019implicit,cohen2024lecture,li2020towards} in matrix factorization contexts (e.g. ``matrix sensing'' problem) for small $ \|A(0)\|_F $. 
As discussed in Sec.~\ref{sec:delta-bal}, in this regime, all layers contribute in an approximately equal way to each singular value being discovered, and the singular vectors generically come closer to being aligned, at least when the $ \sigma_j $ are well-separated in comparison with $ \delta$. In Fig.~\ref{fig:sim1}(a), for instance, $ \| V_{k+1}^{ \top} U_{k} - I \|_F^2 $ tends towards zero for all pairs of blocks, $k=1,...,h-1$. 
However, as pointed out in Sec.~\ref{sec:delta-bal}, if the $ \sigma_j $ are too close, approximate alignment in the sense $ [V_{k+1}^\top U_k]_{1:d_y,1:dy} \approx I_{d_y} $ may be missing. In Fig.~\ref{fig:dbalance}, a $ \delta $-balanced trajectory learns singular values that are relatively close (panel (a)) and pairwise very close (panel (b)). In particular, Fig.~\ref{fig:dbalance}(b) illustrates how the system moves towards alignment when a new distinct singular value is learnt (first and third), but moves away from alignment when another almost identical singular value is learnt (second and fourth). The singular vectors corresponding to singular values far enough apart are still approximately orthogonal in the case of both Fig.~\ref{fig:dbalance}(a) and (b).

No hierarchical learning is instead visible when $ A(0)$ is large, see Fig.~\ref{fig:sim1}(b). The convergence rate is orders of magnitude faster in this latter case, and the contribution of the different layers to the singular values of $ A^h $ may vary. The singular vectors show no clear sign of aligning. 

In both cases, while the individual singular values of $A$ (i.e., the singular values of the individual layers) may change with the initial condition, what is learnt univocally are the singular values of $A^h$, as expected.  
While at a global optimum $ A^h $ has always $ d_y $ nonzero singular values regardless of the initialization, when  $ \|A(0)\|_F$ is small $ A $ has $ h d_y $ ``large'' singular values, and the remaining ones are instead small in comparison, see Fig.~\ref{fig:sim1}(a).
This is not necessarily the case when $ \|A(0)\|_F$ is away from the origin:  $A$ can have more than $ h d_y $ singular values of comparable size, see Fig.~\ref{fig:sim1}(b). This suggests that the implicit regularization towards low rank is restricted to near-balanced initializations and not universal for the class of deep linear neural networks.

Notice that while $ \| A(t)^h\|_F $ is monotone (since $ \mathcal{L}(A(t)) $ is monotone), $ \| A(t) \|_F $ is not, see Fig.~\ref{fig:sim1}(b), even though it shares all critical points of $ \mathcal{L}(A)$, see Proposition~\ref{prop:FrobA}. 

Fig.~\ref{fig:sim2} shows the arcs $ \mathcal{L}(A(\alpha))$ discussed in Section~\ref{sec:stab_unstab_manif}. When $ \mathcal{S}_1 \subset \mathcal{S}_2$, then the loss function along the arc, $ \mathcal{L}(A(\alpha))$, monotonically decreases when passing from the critical point $A_1 $ to the critical point $A_2$, meaning that $ A_1 $ has an unstable submanifold along the arc direction, while $ A_2 $ has a stable submanifold in the arc direction. The converse occurs when $ \mathcal{S}_1 \supset \mathcal{S}_2$. When $ \mathcal{S}_1 \not{\!\!\subset}\; \mathcal{S}_2 $ and $ \mathcal{S}_1 \not{\!\!\supset}\; \mathcal{S}_2 $ (called ``mixed case'' in the lower right panel of Fig.~\ref{fig:sim2}) then what often happens is that $ \mathcal{L}(A(\alpha))$ is not monotone along the arc, but shows local minima/maxima. Obviously such arcs are not gradient flow trajectories.
Different arcs in the fiber $ \phi^{-1}(\Phi(\alpha)) $ may even give different types of monotonicity. 

\begin{figure}[htb!]
\centering
\subfigure[]{
\begin{minipage}{7.5cm}
\includegraphics[trim=3.5cm 8cm 3.3cm 8cm, clip=true, width=7.5cm]{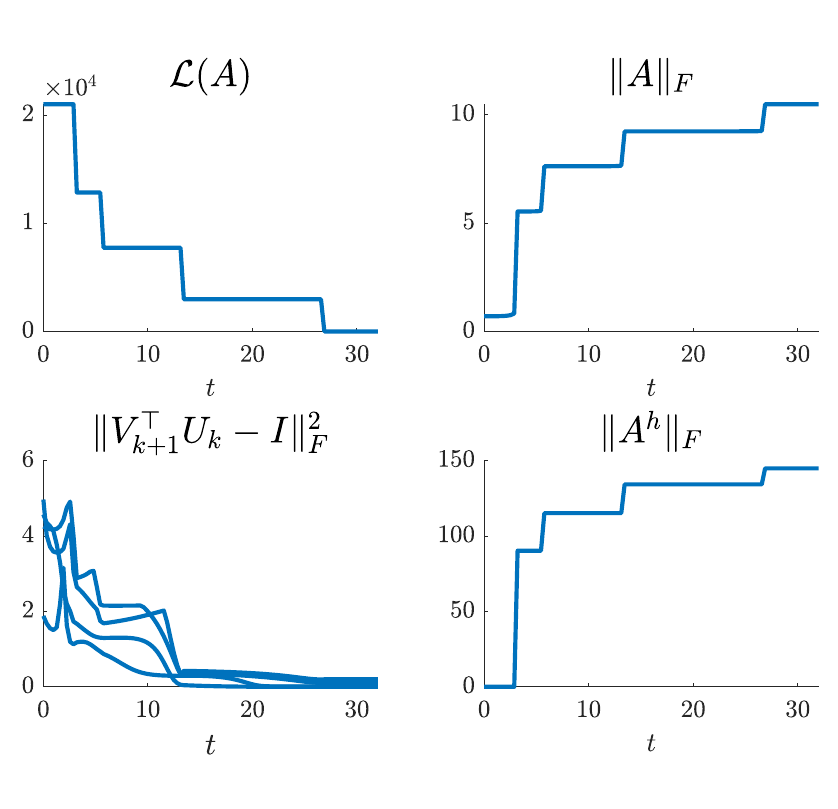} 
\includegraphics[trim=3.5cm 8cm 3.3cm 9cm, clip=true, width=7.5cm]{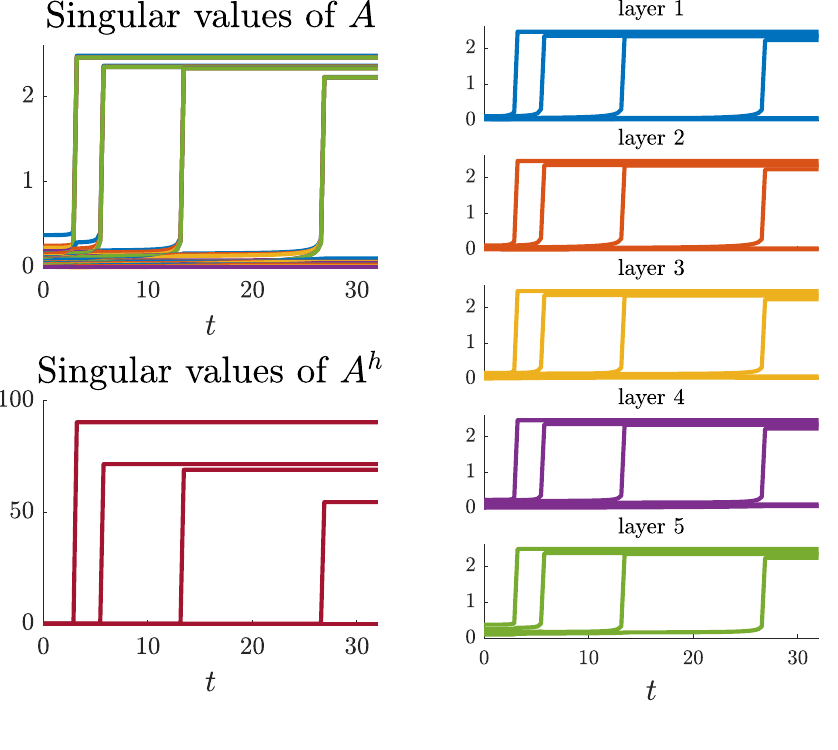}
\end{minipage}}
\subfigure[]{
\begin{minipage}{7.5cm}
\includegraphics[trim=3.5cm 8cm 3.3cm 8cm, clip=true, width=7.5cm]{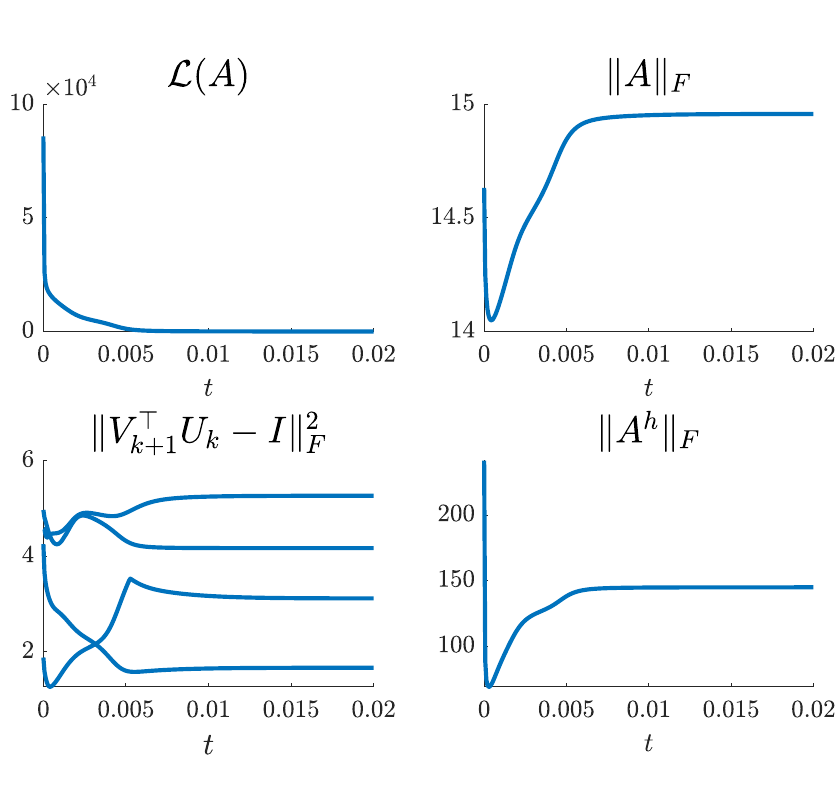}
\includegraphics[trim=3.5cm 8cm 3.3cm 9cm, clip=true, width=7.5cm]{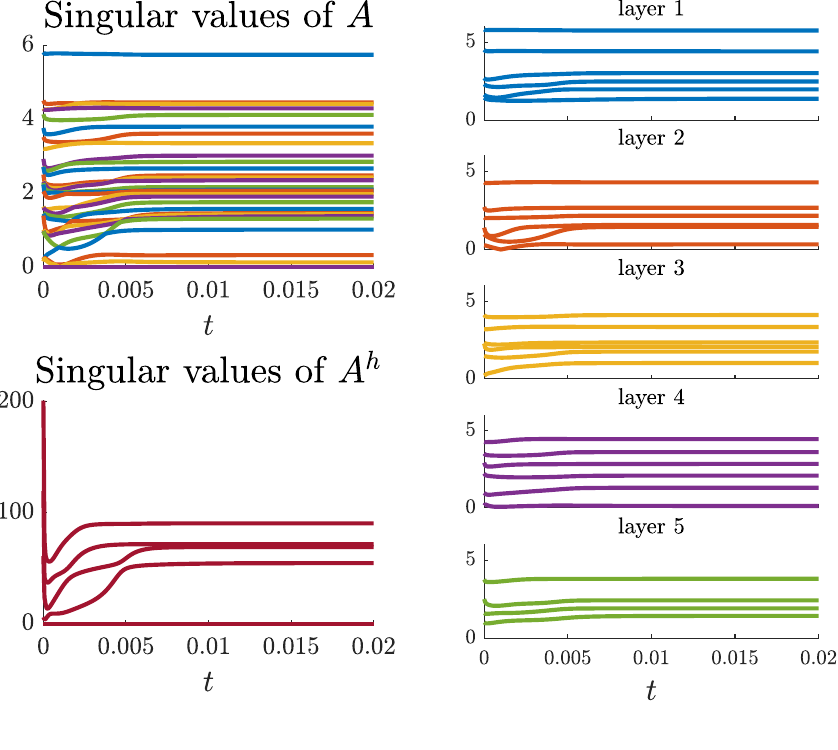}
\end{minipage}}
\caption{Example in Section~\ref{sec:ex-qualitat-anal}. (a): Case with $ A(0) $ near the origin. (b): Case with $ A(0) $ away from the origin. Convergence is to a global minimum in both panels.}
\label{fig:sim1}
\end{figure}

\begin{figure}[htb!]
	\centering
	\subfigure[]{
			\includegraphics[trim=3cm 9cm 3cm 9.5cm, clip=true, width=0.32\textwidth]{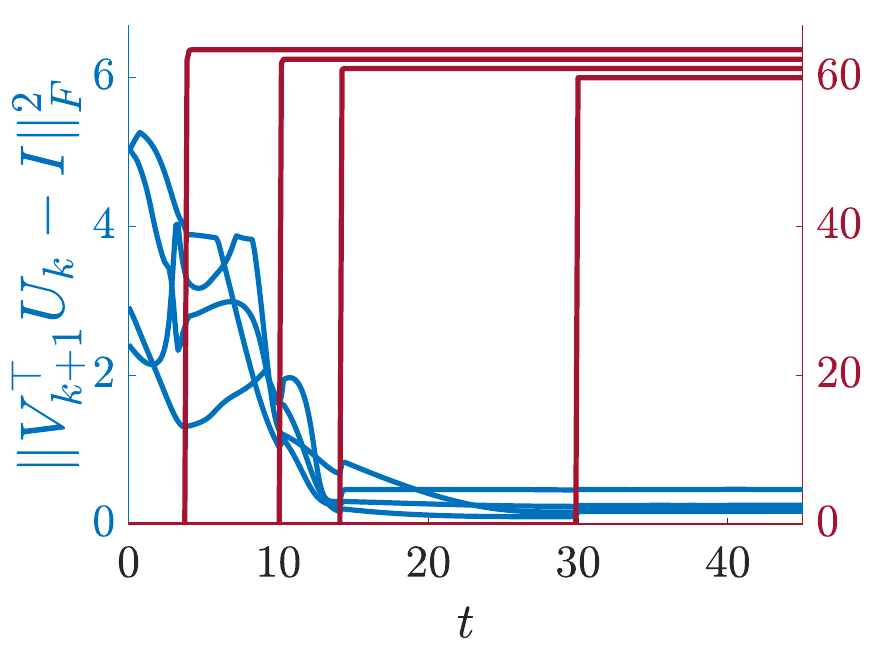}}$ \qquad $ 
	\subfigure[]{
			\includegraphics[trim=3.5cm 9cm 2.5cm 9.5cm, clip=true, width=0.32\textwidth]{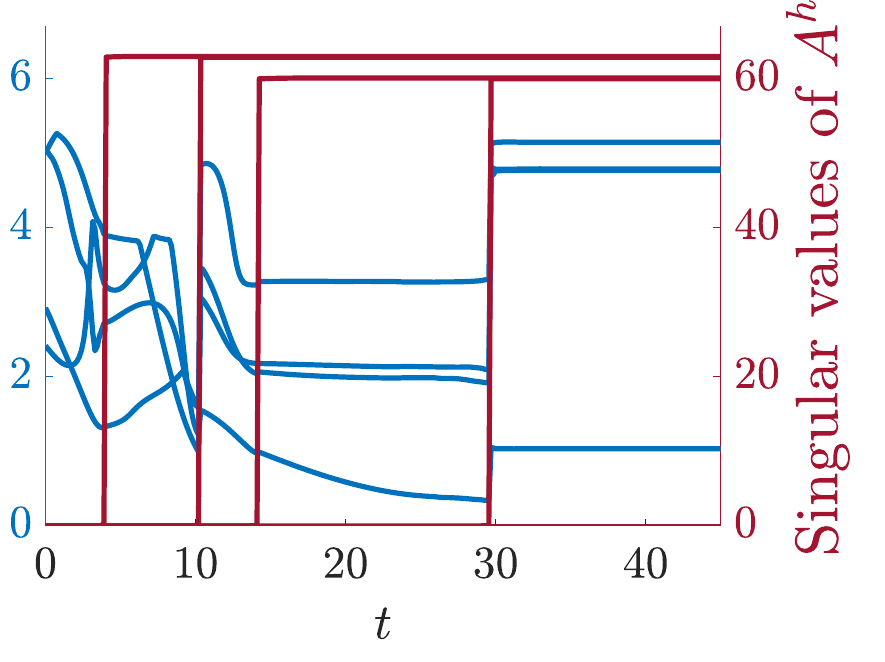}}
\caption{Example of Section~\ref{sec:ex-qualitat-anal}. Singular values and vectors for $\delta$-balanced trajectories, $ \delta = 0.04 $. (a): Singular values of $ \Sigma $ have relatively close spacing, $ \sigma_{j} - \sigma_{j+1} = 1.5 $, $\forall j$. The singular vectors approximately align. (b): Singular values of $ \Sigma$ are pairwise very close, $ \sigma_{j} - \sigma_{j+1} = 0.1 $, $j = 1,3$. The singular vectors do not align.}
\label{fig:dbalance}
\end{figure}

\begin{figure}[htb!]
\centering
\subfigure[]{
\includegraphics[trim=3cm 8cm 1.5cm 7cm, clip=true, width=10cm]{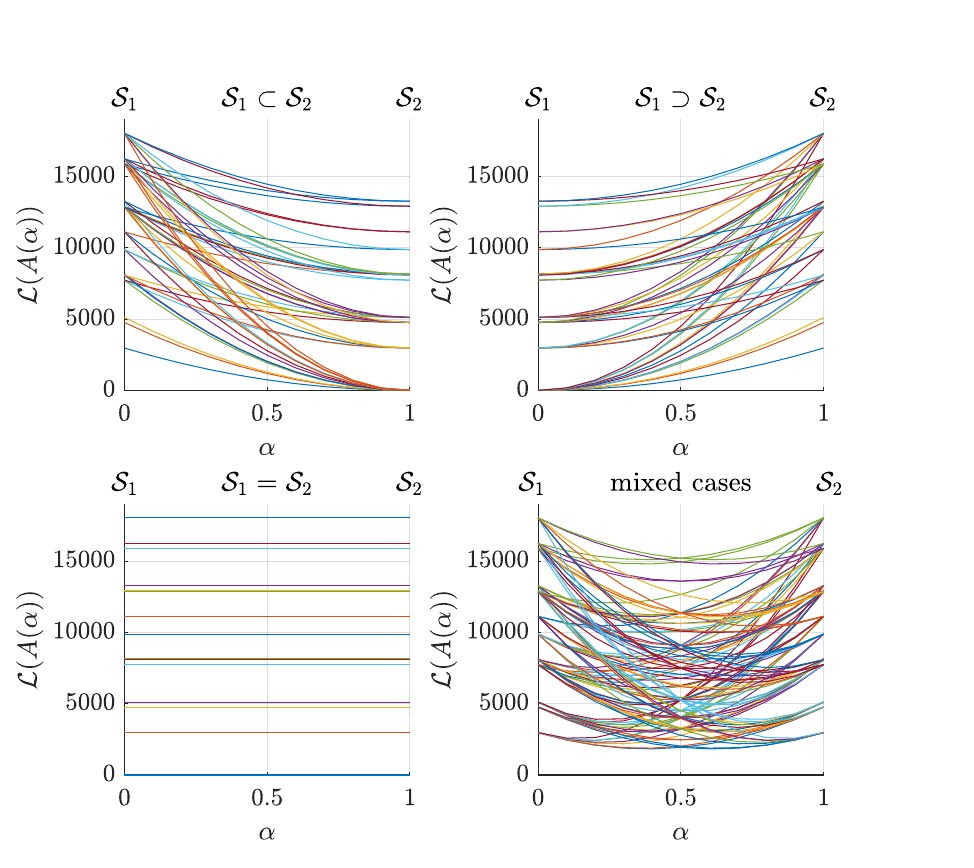}}
\caption{Example in Section~\ref{sec:ex-qualitat-anal}. Arcs $ \mathcal{L}(A(\alpha))$ for the various cases of Theorem~\ref{thm:arcs}.}
\label{fig:sim2}
\end{figure}

\subsection{Loss landscape in small-scale examples}
In this Section we investigate a few very small scale examples. The advantage is that the landscape of critical points can be explored in a precise way, and visualized. 
For all of these examples we proceed by vectorizing the matrix ODEs.

\subsubsection{Single hidden node}
\label{sec:1-node-1-layer}
Consider an example of a single hidden layer with a single hidden node. The gradient flow system must then learn only two parameters $ \bm{w} = \begin{bmatrix} w_1 & w_2\end{bmatrix}^\top$, and obeys to the vectorized ODE
\beq
\label{eq:plane-sys}
\begin{split}
\dot w_1 & = w_2 (\sigma  - w_1 w_2) \\
\dot w_2 & = w_1 (\sigma - w_1 w_2)
\end{split}
\eeq
where $ \sigma>0$ is given (it represents the data of the problem after the---here trivial---SVD decomposition). 
The system has an infinite number of equilibrium points: the origin $ \bm{w}=0 $ is one equilibrium, but also $ \left( w_1, \frac{\sigma}{w_1} \right) $ for any $ w_1 \neq 0$. 
The Jacobian matrix of \eqref{eq:plane-sys} (which is also equal to the negated Hessian for a gradient system) is 
\[
J = \begin{bmatrix} - w_2^2 & \sigma - 2 w_1 w_2 \\ \sigma - 2 w_1 w_2 & - w_1^2 \end{bmatrix}.
\]
Using the Lyapunov indirect method to investigate the stability of the equilibria, we obtain that $ \bm{w}=0 $ is a saddle point (the Jacobian has eigenvalues $ \pm \sigma$, of eigenvectors $ \bm{v}_1 = \begin{bmatrix}1 & 1 \end{bmatrix}^\top $ and $ \bm{v}_2 = \begin{bmatrix} 1 & -1 \end{bmatrix}^\top $), while for $ \left( w_1, \frac{\sigma}{w_1} \right) $ the Lyapunov indirect method fails, as the eigenvalues are $ 0$ and $ -\frac{\sigma^2 + w_1^4}{w_1^2} $. 

While $ \bm{w}=0$ is a strict saddle point (the Hessian has a negative eigenvalue), 
$ \left( w_1, \frac{\sigma}{w_1} \right) $ instead has a singular positive semidefinite Hessian, $\lambda_{\min}(-J) \geq 0 $, at least as long as we disregard the conservation law which restricts the dimension of the state space and hence reduces the size of the ``effective'' Jacobian/Hessian. 
The conservation law for this example becomes the hyperbola $  \pi(\bm{w}) = w_1^2 - w_2^2 - c =0 $ for some scalar constant $ c$. 
Combining this constraint with the expression for the equilibria, we get that each trajectory not passing through the origin has exactly one equilibrium point. 
In fact, the algebraic system 
\beq
 \begin{cases}
  \sigma - w_1 w_2  =0 & \\
  w_1^2 - w_2^2  = c  &
  \end{cases}
\label{eq:1-node-ex}
\eeq
leads to  $ w_1^4 - c w_1^2 - \sigma^2 =0$, which has only one admissible solution: $ w_1^2 = \frac{1}{2} \left( c + \sqrt{c^2 + 4 \sigma^2 } \right) $. The expression for the equilibrium is then  
\beq
\bm{w}_e = \left( \sqrt{ \frac{1}{2} \left( c + \sqrt{c^2 + 4 \sigma^2 } \right) }, \, \frac{\sigma}{\sqrt{ \frac{1}{2} \left( c + \sqrt{c^2 + 4 \sigma^2 } \right) }} \right).
\label{eq:1-node-equil}
\eeq 
Notice how this argument implies that while the quadratic loss function $ \mathcal{L}(\bm{w}) = \frac{1}{2} \left( \sigma - w_1 w_2 \right)^2 $ is only positive semidefinite in $\mathbb{R}^2$, its restriction to a trajectory of \eqref{eq:plane-sys} is indeed positive definite since $ \mathcal{L}(\bm{w}) =0 $ corresponds to solving \eqref{eq:1-node-ex}.
Computing the derivative of the loss function we get 
\[
\dot{\mathcal{L}}(\bm{w}) = - ( \sigma - w_1 w_2 )^2 (w_1^2 + w_2^2)  = - 2(w_1^2 + w_2^2) \mathcal{L}(\bm{w}).
\]
We have to distinguish two types of trajectories. 
\benu\item $0$-balance conservation law: $c=0 $. It corresponds to $ w_1^2 = w_2^2 $, i.e., $ w_1(t)=w_2(t) $ or $ w_1(t) = - w_2(t) $ for all $ t$. This means that the eigenvectors of the linearization at $ \bm{w}=0 $, i.e., $ \bm{v}_1 $ and $ \bm{v}_2 $, form also the trajectories obeying to the 0-balance conservation law. In detail, the $ \bm{v}_1 $ trajectory is the unstable submanifold of the saddle point: it is characterized by an equilibrium point on each side of the origin, $ \bm{w}_e = \pm \sqrt{\sigma} \begin{bmatrix} 1 & 1 \end{bmatrix}^\top$. The system \eqref{eq:plane-sys} reduces to the single equation $ \dot w_1 = w_1 (\sigma - w_1^2)$ which has the origin as unstable equilibrium and $ \bm{w}_e $ as asymptotically stable equilibria, similarly to \eqref{eq:grad-flow-z}. In particular, the basin of attraction of $ \sqrt{\sigma} \begin{bmatrix} 1 & 1 \end{bmatrix}^\top$ is the open half-line $ w_1 = w_2 >0$, while $ - \sqrt{\sigma} \begin{bmatrix} 1 & 1 \end{bmatrix}^\top$ is the attractor for the other half line $ w_1 = w_2 <0$. 
No other equilibrium appears instead on the $ \bm{v}_2 $ eigenvector trajectory, i.e., on the stable submanifold of the saddle point. The single equation becomes $ \dot w_1 = w_1 (-\sigma - w_1^2)$ which has only the origin as an asymptotically stable equilibrium point when $ \sigma>0$.

\item Any other trajectory, corresponding to nonzero balance in the conservation law, i.e., $ c\neq 0$. The solution of \eqref{eq:plane-sys} does not pass through the origin, hence $ \dot{\mathcal{L}} $ is negative definite along the trajectories, vanishing only in the equilibrium point $ \bm{w}_e $, and implying that $ \bm{w}_e $ is asymptotically stable for all $ (w_1(0), \, w_2(0) )$ such that $  w_1^2(0) - w_2^2(0)  = c$. The convergence rate is exponential.
\eenu
The phase portrait of the system \eqref{eq:plane-sys} is shown in Fig.~\ref{fig:1hidden}. As can be observed, the line $ w_1 = - w_2 $ (i.e., the stable submanifold of the origin in the $0$-balance trajectory) acts as a separatrix of the phase plane, subdividing trajectories for which $ c>0 $ from trajectories for which $ c<0$. The two half lines $w_1 = w_2 $ (i.e., the unstable submanifolds for the origin in the $0$-balanced trajectory) act as heteroclinic orbits, and contain the already mentioned two equilibria $ \bm{w}_e = \pm \sqrt{\sigma} \begin{bmatrix} 1 & 1 \end{bmatrix}^\top$. Notice further in Fig.~\ref{fig:1hidden}(b) that near the origin $ \dot{\mathcal{L}}\sim 0$ i.e., the loss landscape is nearly flat. The convergence is slow but nevertheless exponential to some $\bm{w}_e$.

\begin{figure}[htb!]
\centering
\subfigure[]{
\includegraphics[trim=4cm 9cm 4cm 8cm, clip=true, width=7cm]{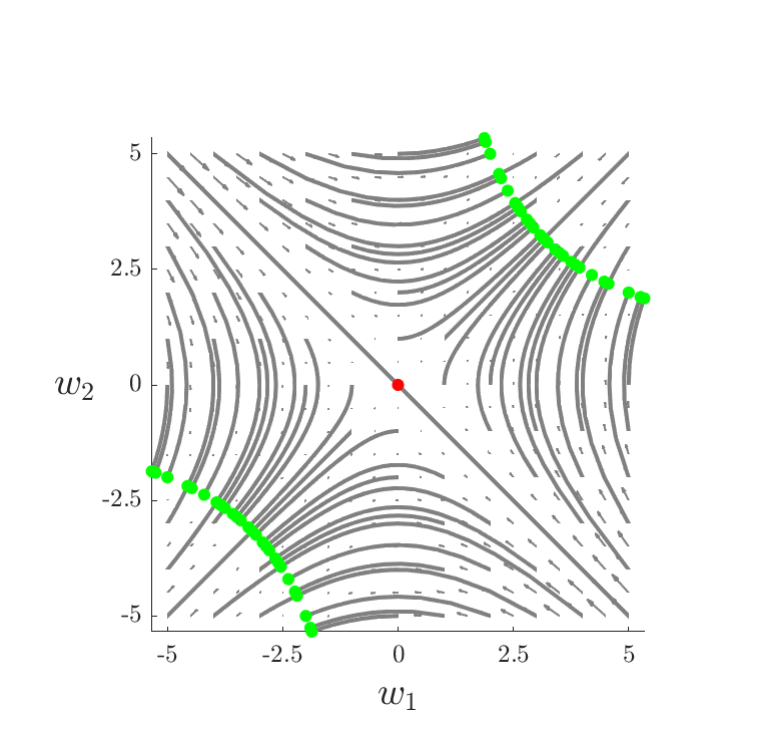}}
\subfigure[]{
\includegraphics[trim=3.5cm 8cm 4cm 10.5cm, clip=true, width=7cm]{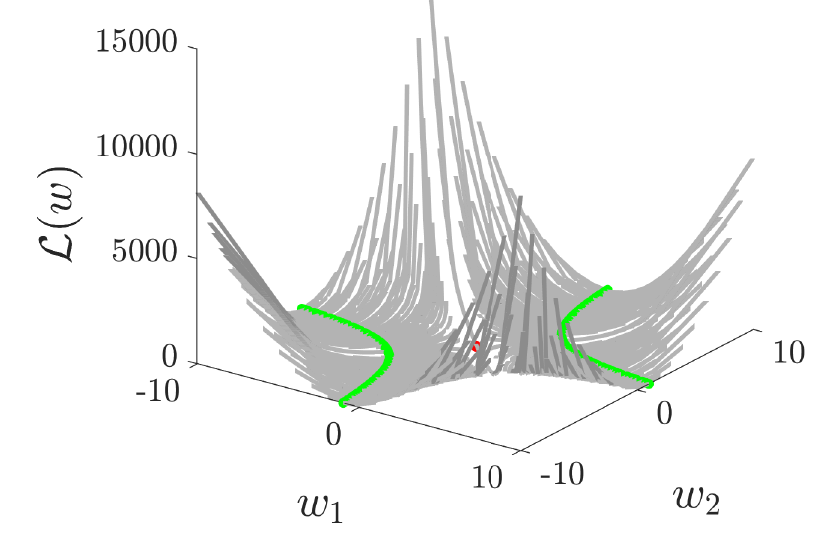}}
\caption{Example of 1 layer with 1 hidden node. Green dots are global minima. The origin in red is a saddle point. (a): Phase space. (b): Loss landscape. Notice the plateau around the origin.}
\label{fig:1hidden}
\end{figure}

\subsubsection{Two hidden layers with one node each}
The gradient flow depends on 3 parameters, $ \bm{w} =\begin{bmatrix} w_1 & w_2 & w_3\end{bmatrix}^\top$. The vectorized equations are
\beq
\label{eq:3layers-1-node-sys}
\begin{split}
\dot w_1 & = w_2 w_3 ( \sigma  - w_1 w_2 w_3 ) \\
\dot w_2 & =  w_1 w_3 ( \sigma  - w_1 w_2 w_3 )  \\
\dot w_3 & =  w_1 w_2 ( \sigma  - w_1 w_2 w_3 ) ,
\end{split}
\eeq
and the conservation laws are 
\[
\begin{split}
    & w_1^2 - w_2^2 = c_1 \\
    & w_2^2 - w_3^2 = c_2 .
\end{split}
\]
The loss function is $ \mathcal{L}(\bm{w})= \frac{1}{2}( \sigma  - w_1 w_2 w_3)^2  $ and its derivative $ \dot{\mathcal{L}}(\bm{w}) = - 2(\sum_{i=1}^3 \prod_{j=1, j\neq i}^3 w_j^2 ) \mathcal{L}(\bm{w}) =- 2 (w_2^2 w_3^2 +w_1^2 w_3^2 +w_1^2 w_2^2 ) \mathcal{L}(\bm{w})  $, which is negative definite as soon as $ c_1, c_2 \neq 0$. The phase space and critical points are visualized in Fig.~\ref{fig:3hidden}.

Rewrite the system \eqref{eq:3layers-1-node-sys} as $ \dot{\bm{w}} = f(\bm{w}) $ and the conservation laws as $ \pi_i (\bm{w}) = w_i^2 - w_{i+1}^2 - c_i =0 $, $ i=1,2$.
If we differentiate $ \pi_i $ we get the two exact one-forms $ d \pi_i(\bm{w}) = 2 w_i - 2 w_{i+1} $. 
Denote $  \Delta_\pi (\bm{w})= {\rm span} ( d \pi_1(\bm{w}) ,\; d\pi_2(\bm{w}) ) $ the codistribution generated by these two one-forms which are always linearly independent. Since both are exact, $ \Delta_\pi (\bm{w}) $ is involutive at any $ \bm{w}$. 
Easy calculations give that $ d \pi_i(\bm{w}) f(\bm{w}) =0$, i.e., the tangent space of the conservation laws is orthogonal to the tangent space of the system: $ \Delta_\pi (\bm{w})  \perp {\rm span}(f(\bm{w})) $.
Together they cover the entire tangent state space of $ \mathbb{R}^3$.
\begin{figure}[ht]
\centering
\subfigure[]{
\includegraphics[trim=3cm 9cm 3cm 9cm, clip=true, width=10cm]{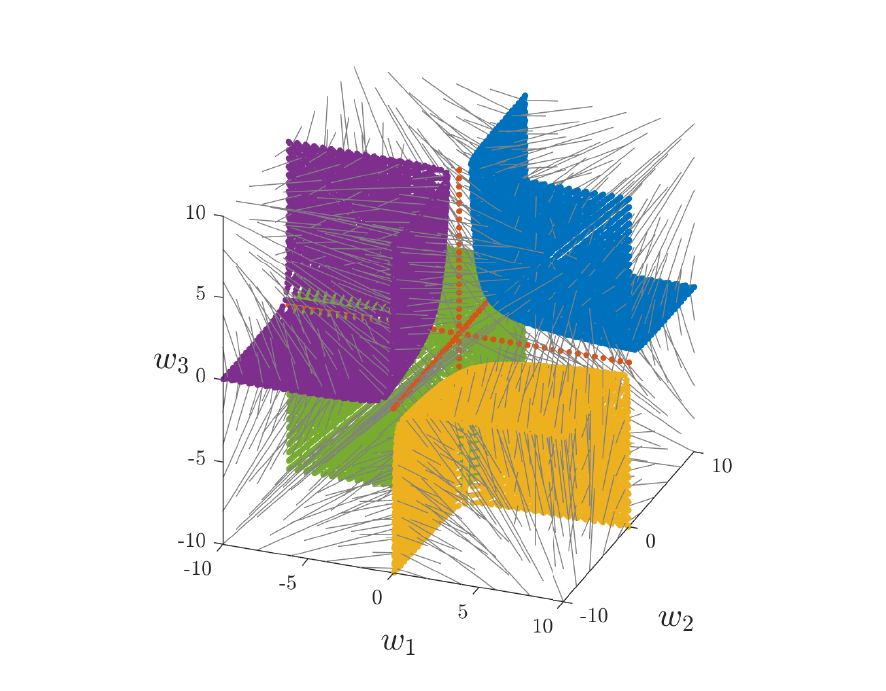}}
\caption{Example of 2 hidden layers with one node each (3D phase space). Colored points represent critical points. All are global minima except the saddle points in red.}
\label{fig:3hidden}
\end{figure}

\subsubsection{One hidden layer with two nodes}
Consider now a shallow network with 2 neurons in the single hidden layer.
If $ W_1=\begin{bmatrix} w_1 \\ w_2 \end{bmatrix} $ and $ W_2=\begin{bmatrix} w_3 & w_4 \end{bmatrix} $, the associated adjacency matrix is 
\[
A =\begin{bmatrix}
0 & 0 & 0 & 0 \\
w_1 & 0  & 0 & 0 \\
w_2 & 0 & 0 & 0  \\
0 & w_3 & w_4 & 0  \\
\end{bmatrix} .
\]
In vector form, $ \bm{w} =\begin{bmatrix} w_1 & w_2 & w_3 & w_4 \end{bmatrix}^\top$, and the gradient flow corresponds to the following 4 ODEs, one for each $ w_i$:  
\beq
\label{eq:2-hiddennodes-sys}
\begin{split}
\dot w_1 & = w_3 ( \sigma  - w_1 w_3 - w_2 w_4 ) \\
\dot w_2 & = w_4 ( \sigma  - w_1 w_3 - w_2 w_4 ) \\
\dot w_3 & = w_1 ( \sigma  - w_1 w_3 - w_2 w_4 ) \\
\dot w_4 & = w_2 ( \sigma  - w_1 w_3 - w_2 w_4 ) .
\end{split}
\eeq
The 4 parameters are linked by three conservation laws:
\[
W_1 W_1^\top - W_2^\top W_2 =\begin{bmatrix} c_1 & c_3 \\ c_3 & c_2 \end{bmatrix} \quad \Longrightarrow \quad 
\begin{cases}
    \pi_1(\bm{w}) =  & w_1^2 - w_3^2 - c_1 =0\\
    \pi_2(\bm{w}) = & w_2^2 - w_4^2 - c_2=0 \\
    \pi_3(\bm{w}) = & w_1 w_2 - w_3 w_4 - c_3 =0 .\\
\end{cases}
\]
The loss function $ \mathcal{L}(\bm{w})=  \frac{1}{2} (\sigma  - w_1 w_3 - w_2 w_4 )^2$ has derivative $ \dot{\mathcal{L}}(\bm{w}) = - 2(\sum_{i=1}^4 w_i^2) \mathcal{L}(\bm{w}) $.
Also in this case $  \Delta_\pi (\bm{w}) = {\rm span} ( d \pi_1 ,\; d\pi_2, \; d \pi_3 ) $ is an involutive distribution of dimension 3 at each $ \bm{w}$ s.t. $  \Delta_\pi (\bm{w}) \perp {\rm span}(f(\bm{w})) $ and $  \Delta_\pi (\bm{w})\oplus {\rm span}(f(\bm{w})) = \mathbb{R}^4$.
As can be seen in the 2D and 3D slices of Fig.~\ref{fig:2hidden}, the (infinitely-many) critical points are scattered over the entire $ \mathbb{R}^4$. Only parameters pairs obeying the conservation laws show a pattern of critical points which is easy to identify, for the other parameter pairs (or triplets) critical points appear everywhere.
The origin is still a saddle point. 
\begin{figure}[ht]
\centering
\subfigure[]{
\includegraphics[trim=5cm 12.7cm 4cm 9cm, clip=true, width=11cm]{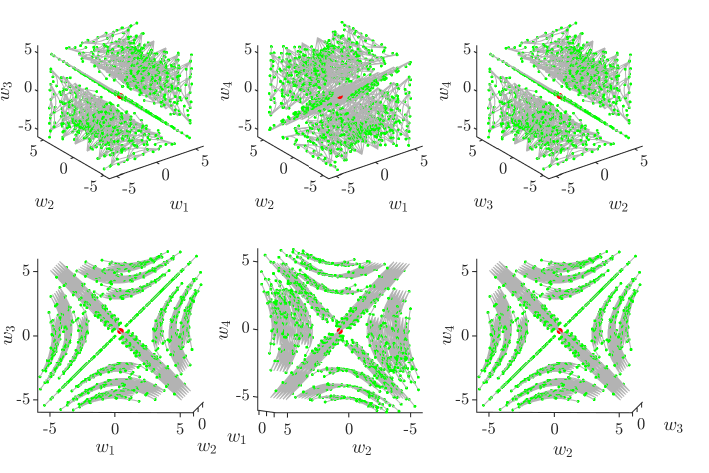}
}
\caption{Example of 1 hidden layer with 2 nodes. (a): 3D slices of the 4D parameter space. Green dots are stable equilibria; the red dot at the origin is a saddle point.  The two rows of panels show the same slices, but from different angles.}
\label{fig:2hidden}
\end{figure}

Notice that as the number of hidden nodes increases, the number of conservation laws quickly outgrows the dimension of the state space: if $ d_1>3$, then the number of conservation laws is $ d_1(d_1+1)/2 $ which, when $ d_x =d_y =1 $, is larger than $ d_xd_1+d_1 d_y = 2 d_1$. Obviously this means that the conservation laws are not all independent, as can be verified by computing $ \dim(\Delta_\pi (\bm{w}))$. 

\section{Extensions and related topics}
\label{sec:extensions}
In this section we briefly review a few extra topics that we do not discuss in detail in the paper, in some cases because they are still beyond reach, and in others because they would require a substantial extra amount of notation and/or technical tools to be introduced, hampering the readability of the paper. 

\paragraph{Discrete gradient descent.}
As mentioned in Section~\ref{sec:grad-flow-basic}, the gradient flow is the limit of a gradient descent algorithm for the step size $ \eta $ that goes to $0$.
For finite $ \eta$, in our block-shift matrix formulation, a gradient descent is just a discretization of \eqref{eq:grad-flowA} and can be expressed as 
\beq
A(k+1) = A(k) - \eta \nabla_{A(k)} \mathcal{L} =  A(k) - \eta \sum_{j=1}^h \left(A(k)^{h-j}\right)^\top \left( E - A(k)^h \right) \left( A(k)^{j-1}
 \right) ^\top.
\label{eq:grad-desc-A}
\eeq
It is straightforward to show that also the gradient descent expression \eqref{eq:grad-desc-A} obeys to properties similar to Proposition~\ref{prop:A-block-struct}: $A(k) \in \mathcal{A}$ is nilpotent $ \forall \, k$, and the gradient descent \eqref{eq:grad-desc-A} is isospectral. 
However, the conservation laws no longer hold exactly, and convergence properties are more delicate to investigate \cite{arora2018convergence}.

\paragraph{Beyond quadratic loss.}
The quadratic loss function is by far the most studied for deep linear networks. However, there exist important results also for other types of losses. In \cite{laurent2018deep}, it is shown how bad local minima are absent for convex differentiable loss functions given that the hidden layers are at least as wide as either the input {\em or} output layers (a more general case than our Assumption~\ref{ass:size}). The fact that bad local minima do not appear in the case of quadratic loss, regardless of the width of the hidden layers, is explained from a geometric point of view in \cite{trager2019pure}.
A few works derive results for the dynamics using general convex loss functions, e.g., \cite{jacot2021saddle,min2023convergence,tu2024mixed}, but do not explicitly treat other loss functions, such as logistic or cross-entropy loss.

\paragraph{Underdetermined problems.}
Most papers treating deep linear networks with quadratic loss deal with overdetermined problems, i.e., the setting where the data matrices have full rank, $\nu \geq d_x$ as in Assumption~\ref{ass:size}.
One well-studied framework that includes linear regression as a special case and allows for formulating underdetemined problems is matrix sensing \cite{arora2019implicit,gunasekar2017implicit,li2020towards}. In matrix sensing, a deep linear network is trained to minimize
\beq
    \mathcal{L}_{\rm ms}(W_1,...,W_h) = \frac{1}{2} \sum_{i=1}^\nu (\langle B_i, W_{1:h} \rangle_F - y_i)^2 \label{eq:matrix-sensing-loss}
\eeq
where $B_i\in \mathbb{R}^{d_y \times d_x}$ are measurement matrices providing measurements $y_i = \langle B_i, W \rangle_F $ of a target matrix $W \in \mathbb{R}^{d_y \times d_x} $.
By taking measurement matrices $ B_i = \mathbf{e}_k \mathbf{x}_\ell^\top $ for $k=1,...,d_y$, $\ell=1,...,\nu$, the quadratic loss \eqref{eq:loss} is recovered.
An important application is matrix completion, where the measurement matrices are $B_i = \mathbb{E}_{k,\ell}$, i.e., each measurement gives information on a single entry of the target $ y_i = \langle \mathbb{E}_{k,\ell}, W \rangle_F = [W]_{k\ell} $.
When only a few measurements are available, the matrix sensing problem is underdetermined, as an infinite number of product matrices $W_{1:h}$ minimizes $ \mathcal{L}_{\rm ms} $.
Interestingly, the overparameterization seems to induce a bias towards low rank (``greedy low rank'') solutions, sometimes coinciding with the minimum nuclear norm solution, although the latter has proven to give an incomplete picture in a general setting \cite{arora2019implicit,li2020towards}.
More specifically, given few measurements of a low-rank target $W$, an overparameterized network with $h\geq 2$ (and small initialization) can recover a low-rank matrix $W_{1:h}$ minimizing the loss \eqref{eq:matrix-sensing-loss}, whereas networks with $h=1$ (i.e., no hidden layers) typically recover a matrix $W_1$ with higher rank than $W$ \cite{arora2019implicit}.

\paragraph{Bottlenecked architectures.}
When some hidden layers have a smaller width than the output layer, $d_{\min} := \min_i d_i < d_y $, the expressivity of the network is restricted to maps of $\rank(A^h) \leq d_{\min} < d_y $. Most important results still hold for these networks: the gradient flow equations and its conservation laws $J\dot{Q}=0$ are unchanged; Proposition~\ref{prop:dotL} still guarantees convergence to a critical point; the results of Proposition~\ref{prop:critical-point-lit} are mostly unchanged except that obviously $r\leq d_{\min}$; and the description of the loss landscape in Proposition~\ref{prop:class-crit-points} is only affected in that the global minimum now corresponds to the signature $\mathcal{S}=\{1,...,d_{\min}\}$, and critical points for which $\rank(A^h) = d_{\min}$ and $\mathcal{S}\neq \{1,...,d_{\min}\}$ are {\em strict} saddle points. 
Importantly, no new {\em non-strict} saddle points are introduced due to the bottleneck, see \cite{achour2021loss} for details. Under generic initializations the network still converges to a global optimum, not a suboptimal rank-constrained solution. One can note that the results on convergence rates in Section~\ref{sec:conv-rate} rely on the assumption $d_{\min} \geq d_x,d_y$. As discussed above, when the loss is other than quadratic, bottlenecks can introduce local minima.

\paragraph{Skip connections.}
Skip (or residual) connections correspond to a layer $k$ being connected not only to the layer $ k+1$ but also to a layer $ k+j$ for $ j>1$. 
When they appear, they alter the ``homogeneity'' of the input-output function \eqref{eq:prodW}, which, in the case of single skip connection $ k \to k+j $ (for simplicity, between layers with identical dimensions $ d_k = d_{k+j}$, and using a skip connection matrix equal to the identity), becomes $ f(X)= (W_{1:h}+ W_{k+j+1:h} W_{1:k}) X$.
This can be accommodated in our matrix representation in block-shift form by expanding the weight matrices
\[
    \tilde{W}_{k+j-1} = \begin{bmatrix}
        W_{k+j-1} & I_{d_k}
    \end{bmatrix}, \,
    \tilde{W}_{k+j-2} = \begin{bmatrix}
        W_{k+j-2} & 0 \\ 0 & I_{d_k}
    \end{bmatrix}, \, \ldots, \,
    \tilde{W}_{k+2} = \begin{bmatrix}
        W_{k+2} & 0 \\ 0 & I_{d_k}
    \end{bmatrix}, \,
    \tilde{W}_{k+1} = \begin{bmatrix}
        W_{k+1} \\ I_{d_k}
    \end{bmatrix}
\]
so that $\tilde{W}_{k+1:k+j-1} = W_{k+1:k+j-1} + I$. The adjacency matrix is
\[
A=\block_1 (W_1, \ldots, W_k, \tilde{W}_{k+1}, \ldots, \tilde{W}_{k+j-1}, W_{k+j}, \ldots, W_h),
\]
and $A^h = \block_h(W_{1:h}+ W_{k+j+1:h} W_{1:k})$. This formulation requires inserting $d_k(j-1)$ artificial nodes, but the weights connecting these nodes are all equal to 1.
The gradient flow dynamics becomes however more complex to express than \eqref{eq:grad-flowA}.
In fact, maintaining constant the edges in the skip connections requires to constrain the gradient flow, for instance via
\[
\dot{A} = A_W \circ \left(\sum_{j=1}^h \left(A^{h-j}\right)^\top \left( E - A^h \right) \left( A^{j-1}
 \right) ^\top \right),
 \]
where $A_W$ is 1 at the positions with learnable weights and 0 otherwise.

\paragraph{Linear vs nonlinear networks.}
The models we consider in this manuscript lack activation functions and other ``layers'' typically seen in modern deep neural networks, such as normalization, pooling, depth concatenations, residual connections, softmax, etc. 
Dealing with these extra ingredients complicates the picture considerably. 
When only standard activation functions such as ReLUs are added, then the topology of the deep neural network is still characterizable in term of $A$, which plays the role of adjacency matrix of the network, but the input-output function becomes dependent on the states of the hidden nodes. In particular, this impacts directly the gradient dynamics, which no longer can be expressed in a compact form, but varies with the batch input data.

\section{Conclusion}
\label{sec:conclusion}
The gradient flow of a deep linear neural network provides a very useful and instructive setup for understanding the training process of a ``true'' deep neural network. 
In particular, the existence of a multitude of degenerate critical points (also of continuous submanifolds of such critical points) is known to play a role also in presence of activation functions, facilitating the practical convergence of the algorithms on one hand, but hampering the understanding of fundamental principles leading to generalization on the other hand.

Even restricting to deep linear neural networks, several aspects remain to be clarified. For instance, we still need to show rigorously under what conditions near the origin the learning is sequential and occurs from the largest to the smallest singular value of the data. 
Also the value of the implicit regularization arguments put forward so far remains to be clarified. In our understanding, they hold only for a specific area of the state space, which puts into question their universal value as generalization principles.
One would like to understand whether there is any equivalent property away from the origin, and if not, why. 
Furthermore, it remains to be understood whether any insight into generalizability obtained on the linear case may transfer to the nonlinear case.

These aspects notwithstanding, we believe that the gradient flow equations of deep linear neural networks, as we have formulated in this paper, provide a novel class of matrix differential equations, which is interesting and elegant, and which we hope will attract the attention of the dynamical systems community, regardless of any specific interest in deep learning.

\appendix

\section{Appendix}

\subsection{Table of symbols}
\label{sec:symb-table}

In Table~\ref{tab:symb-table} we collect the main notation used in the paper. See also Section~\ref{sec:notation} for further basic notation.

\begin{table}[h!]
    \centering
    \caption{Table of symbols and notation.}
    \bgroup
    \def\arraystretch{1.14}
    \begin{tabular}{ll}
         \hline Notation & Description \\ \hline
         $ d_i $ & Width of layer $i$ \\
         $ d_x$ & Dimension of input $d_x = d_0 $ \\
         $ d_y$ & Dimension of output $d_y = d_h $ \\
         $ \nu $ & Number of data \\
         $ n_v $ & Number of network nodes \\
         $ W_i $ & Weight matrix $ W_i \in \mathbb{R}^{d_i \times d_{i-1}}$ \\
         $W_{j:k}$ & Weight matrix product $W_k W_{k-1}\cdots W_j$ \\
         $\Sigma $ & Matrix of data singular values \\
         $\Xi$ & Target-prediction difference $\Xi = \Sigma - W_{1:h} $ \\
         $\mathcal{L}$ & Quadratic loss function \\
         $ C_i $ & Quantities conserved by gradient flow $ C_i = W_iW_i^\top - W_{i+1}W_{i+1}^\top $ \\
         $A$ & Block-shift adjacency matrix,  $A = \block_1(W_1,...,W_h) \in \mathbb{R}^{n_v \times n_v}$ \\
         $\mathcal{A}$ & Set of block-shift adjacency matrices $A$ \\
         $ E $ & Singular value matrix in $ E = \block_h(\Sigma) \in \mathbb{R}^{n_v \times n_v}$ \\
         $\mathfrak{h}(A,V)$ & Hessian quadratic form (at $A$, along a direction $V\in\mathcal{A}$) \\
         $\rho_k(A,V)$ & Sum of $A,V$ products: $\sum_{k {\rm \, factors \,} V}^{k_1+...+k_h = h-k} A^{k_1}VA^{k_2}\cdots A^{k_{h-1}} V A^{k_h}$ \\
         $\phi(A)$ & Power operator $\phi(A) = A^h $ \\
         $ \mathcal{A}_{\phi} $ & Set of matrices $\{\phi(A) = A^h : A\in \mathcal{A} \}$ \\
         $\Phi$ & An element in $\mathcal{A}_{\phi}$ \\
         $A_1 \sim A_2$ & Equivalence relation $A_1,A_2\in \mathcal{A}$ s.t. $ A_1^h = A_2^h $\\
         $\mathcal{P}$ & Set of changes of basis for block-shift adjacency matrices \\
         $\mathcal{P_O}$ & $P\in \mathcal{P}$ with $ P$ orthogonal \\
         $ \mathcal{M}_c $ & Invariant set of $\dot{\mathcal{L}}$ with $\mathcal{L}=c$ \\
         $\mathcal{A}^{\rm diag} $ & Subset of $\mathcal{A}$ with all $W_i$ diagonal \\
         $\Delta$ & Distribution (in differential geometry sense)\\
         $\perp$ & Orthogonality relation: $u\perp v \implies u^\top v = 0$
    \end{tabular}
    \egroup
    \label{tab:symb-table}
\end{table}

\subsection{Reformulation of the loss function}
\label{app:loss-func}

Here, following \cite{chitour2018geometric}, we provide the details on how to reformulate the loss function \eqref{eq:orig_loss} as in \eqref{eq:loss}. Let the weight matrices of the linear neural network be $\bar{W}_1,...,\bar{W}_h$, so that $f(X)=\bar{W}_h\cdots \bar{W}_1X$. We take a singular value decomposition of the inputs
\begin{equation}
    X=U_X\Sigma_X V_X^\top, \quad V_X=\begin{bmatrix}
        V_X^1 & V_X^2
    \end{bmatrix}, \quad \Sigma_X=\begin{bmatrix}
        \Lambda_X & 0
    \end{bmatrix}
\end{equation}
where the columns of $V_X^1$ are the $d_x$ first singular vectors of $V_X$, associated to the singular values $\Lambda_X $, and write \eqref{eq:orig_loss} as
\begin{equation}
    \lVert Y-\bar{W}_h\cdots \bar{W}_1X\rVert_F^2 = \lVert Y-\bar{W}_h\cdots \bar{W}_1U_X\Sigma_X V_X^\top\rVert_F^2 
    \label{eq:loss-reformul1}
\end{equation}
Recall that for any matrix $B$, and for any $U$ unitary we have $\lVert U B \rVert_F^2=\text{trace}(B^*U^*UB)=\lVert B U\rVert_F^2=\text{trace}(U^*B^*BU)=\lVert B\rVert_F^2$. We have that \eqref{eq:loss-reformul1} is equal to
\begin{equation*}
    \lVert YV_X^1-\bar{W}_h\cdots \bar{W}_1U_X \Lambda_X  \rVert_F^2 + \lVert YV_X^2 \rVert_F^2.
\end{equation*}
Taking a SVD of $YV_X^1=U_Y\Sigma V_Y^\top$ we can express the loss as
\begin{equation}\label{eq:sep_loss}
    \lVert U_Y\Sigma V_Y^\top -\bar{W}_h\cdots \bar{W}_1U_X \Lambda_X  \rVert_F^2 + \lVert YV_X^2 \rVert_F^2 = \lVert \Sigma - U_Y^\top\bar{W}_h\cdots \bar{W}_1U_X \Lambda_X V_Y \rVert_F^2 + \lVert YV_X^2 \rVert_F^2.
\end{equation}
By renaming the weight matrices $W_h=U_Y^\top\bar{W}_h$, $W_1=\bar{W}_1U_X \Lambda_X V_Y$, and $W_i=\bar{W}_i$ for $i=2,...,h-1$, and dropping the term $\lVert YV_X^2 \rVert_F^2$ which is never changed by any $W_i$ we have the loss function \eqref{eq:loss}.

\subsection{Boundedness of the trajectories}
\label{app:bounded}
The following Lemma reformulates Lemma 2.4 of \cite{chitour2018geometric} for the matrix $A \in \mathcal{A}$ instead of the weight matrices $W_j$. It is needed to prove boundedness of the gradient flow in Proposition~\ref{prop:dotL}.

\begin{lemma}
    \label{lem:bounded}
    There exists a positive constant $\gamma\in(0,1)$ only depending on $h$ and the dimensions of the problem, such that for all $A\in \mathcal{A}$, there exist two polynomials $p$ and $q$ of degree $\leq h-1$, with nonnegative coefficients only depending on $h$ and $C_j=W_{j+1}^\top  W_{j+1}-W_{j}W_{j}^\top $, $j=1,...,h-1$, such that
\begin{equation}
	\gamma \lVert A \rVert^{2h}_F - p\left(\lVert A \rVert^{2}_F\right) \leq \lVert A^h \rVert^2_F \leq \lVert A \rVert^{2h}_F + q\left(\lVert A \rVert^{2}_F\right)
\end{equation}
\end{lemma}

\begin{proof}
    Recall the conservation law from Proposition~\ref{prop:conserv-law-A}, and note the following relation
\begin{align}
	\tr{(W_{j+1:h}(W_jW_j^\top )^jW_{j+1:h}^\top )} &= \tr{(W_{j+1:h}(W_{j+1}^\top W_{j+1}-C_j)^jW_{j+1:h}^\top )} \nonumber \\
	&= \tr{(W_{j+2:h}(W_{j+1}W_{j+1}^\top )^{j+1}W_{j+2:h}^\top )} + \tr{(W_{j+1:h} K_j W_{j+1:h}^\top )}. \label{eq:tr_split}
\end{align}
Here, $K_j=(W_{j+1}^\top W_{j+1}-C_j)^j-(W_{j+1}^\top W_{j+1})^{j}$, and expanding the first parenthesis we see that $ K_j $ contains terms with factors of $C_j=C_j^\top $ and $W_{j+1}^\top W_{j+1}$ interlaced, and the latter of power at most $j-1$. Consider a general term of $ K_j$, $W_{j+1}^\top \cdots W_{j+1}C_j W_{j+1}^\top \cdots W_{j+1}C_j^\ell W_{j+1}^\top \cdots W_{j+1}C_j^k$, and note that we can make pairings such that we get positive semidefinite factors $(C_j^\top W_{j+1}^\top \cdots W_{j+1}C_j)$, or even powers of $C_j$:s, except possibly for a single factor $C_j$. We can move the $C_j$:s outside the trace by permuting the factors and applying the identity $\lvert \tr{(BC)}\rvert \leq \lVert C\rVert_F\tr{(B)}$ for $B\succeq0$, and submultiplicativity of the trace for positive semidefinite matrices. There will be at most $j-1$ factors $W_{j+1}^\top W_{j+1}$ in each term of $K_j$ which are bounded by 
\begin{equation*}
	\tr{(W_{j+1:h}W_{j+1:h}^\top )}\tr{((W_{j+1}^\top W_{j+1})^{j-1})} \leq \lVert W_{j+1:h}\rVert^{2}_F \lVert W_{j+1}\rVert^{2(j-1)}_F \leq \lVert A \rVert^{2(h-j)}_F\lVert A \rVert^{2(j-1)}_F = \lVert A \rVert^{2(h-1)}_F,
\end{equation*}
meaning that all the terms stemming from $K_j$ in \eqref{eq:tr_split} can be lower and upper bounded by polynomials in $ \lVert A \rVert^{2}_F$ with nonnegative coefficients depending on $C_j$, and of degree $\leq h-1$. In the previous formula, we have also used the fact that $\lVert A \rVert^{2k}_F \geq \lVert A^k \rVert^2_F \geq \lVert W_{\ell:k+\ell-1} \rVert^2_F$, see \eqref{eq:Apowerk}, $k=1,...,h$, i.e., $\lVert A \rVert^{2k}_F$ upper bounds the norm of any product of $k$ consecutive factors $W_\ell\cdots W_{k+\ell-1}$. The high order term in \eqref{eq:tr_split} is on the same form as the term on the left hand side. We can again substitute $W_{j+1}W_{j+1}^\top =W_{j+2}^\top W_{j+2}-C_{j+1}$, and bound the lower order terms with polynomials in $\lVert A \rVert^{2}_F$ of degree $\leq h-1$. Starting from $\lVert A^h \rVert^2_F = \tr{(W_h\cdots W_1W_1^\top \cdots W_h^\top )}$ and iterating gives
\begin{equation*}
	\tr{((W_hW_h^\top )^h)} -\tilde{p}\left(\lVert A \rVert^{2}_F\right) \leq \lVert A^h \rVert^2_F \leq  \tr{((W_hW_h^\top )^h)} + \tilde{q}\left(\lVert A \rVert^{2}_F\right).
\end{equation*}
Applying Lemma~\ref{lem:tracebound}, along with 
\begin{equation*}
	\lVert A \rVert^{2}_F = \sum_{j=1}^{h} \tr{(W_jW_j^\top )} = h\tr{(W_hW_h^\top )} + \underbrace{\sum_{j=1}^{h} c_{jh}}_{=\tilde{c}} \implies \tr{(W_hW_h^\top )} = \frac{ \lVert A \rVert^{2}_F - \tilde{c} }{h}
\end{equation*}
we get
\begin{equation*}
	\gamma\lVert A \rVert^{2h}_F - p\left(\lVert A \rVert^{2}_F\right) \leq \lVert A^h \rVert^2_F \leq  \lVert A \rVert^{2h}_F + q\left(\lVert A \rVert^{2}_F\right),
\end{equation*}
where the terms with $\lVert A \rVert^{2}_F$ of degree $\leq h-1$ are bounded by $p$ and $q$.
\end{proof}

\begin{lemma}
    \label{lem:tracebound}
    For any $k\times k$ matrix $B\succeq 0$, for some $\gamma\in(0,1)$ that depends only on $k$ and $\ell\in(1,\infty)$, it holds that
\begin{equation*}
	\gamma(\tr{(B)})^\ell \leq \tr{(B^\ell)}\leq \tr{(B)}^\ell.
\end{equation*}
\end{lemma}

\begin{proof}
     The statement is trivial for $k=1$, let us therefore assume that $k>1$. The right inequality follows from the submultiplicativity of the trace for positive semidefinite matrices. Let $\lambda_i(B)\geq 0$, $i=1,...,k$ denote the eigenvalues of $B$. For the left inequality
\begin{equation}
	(\tr{(B)})^\ell = \left(\sum_{i=1}^k \lambda_i(B)\right)^\ell \leq k^{\ell-1}\sum_{i=1}^k \lambda_i^\ell(B) = k^{\ell-1} \sum_{i=1}^k \lambda_i(B^\ell) = k^{\ell-1} \tr{(B^\ell)}
\end{equation}
where we in the second step used the H{\"o}lder inequality
\begin{equation*}
	0\leq\sum_{i=1}^k \lambda_i(B)\cdot 1 \leq \left(\sum_{i=1}^k \lambda_i^\ell(B) \right)^{\frac{1}{\ell}}\left(\sum_{i=1}^k 1 \right)^{1-\frac{1}{\ell}} \implies \left(\sum_{i=1}^k \lambda_i^\ell(B)\right)^\ell \leq \left(\sum_{i=1}^k \lambda_i(B) \right) k^{\ell-1}.
\end{equation*}
\end{proof}

\subsection{A standard singular value decomposition for $A$}
\label{sec:svd}

The following proposition constructs a ``standard'' SVD for the adjacency matrix $ A = \block_1(W_1, \ldots , W_h)\in \mathcal{A} $ out of that of its individual blocks $ W_i $. 

\begin{proposition}
\label{prop:svdA}
Let $ A= \block_1(W_1, \ldots , W_h)\in \mathcal{A} $. The singular values of $A$ are given by the union of the singular values of each block $W_i$, with $n_v-q$ extra zeros, where $q=\sum_{i=1}^{h}\min\{d_i,d_{i-1}\}$.
The associated singular vectors of $A$ can also be obtained assembling the singular vectors of the blocks $ W_i$.
\end{proposition}

\begin{proof}
Consider a real SVD of each block $W_i=U_i \Omega_iV_i^\top \in\mathbb{R}^{d_i\times d_{i-1}}$, $U_i\in\mathbb{R}^{d_i\times d_i}$, $\Omega_i\in\mathbb{R}^{d_i\times d_{i-i}}$, and $V_i\in\mathbb{R}^{d_{i-1}\times d_{i-1}}$. As $\Omega_i$ may be rectangular, we write $\Omega_i = \begin{bmatrix}
\Lambda_i & 0
\end{bmatrix}$ (or $\Omega_i = \begin{bmatrix}
\Lambda_i & 0
\end{bmatrix}^\top $) with $\Lambda_i$ diagonal. 
We divide the matrices into columns corresponding to singular values, and columns that have no contribution to $W_i$:
\begin{equation}
	U_i = \begin{bmatrix}
		U_i^\sigma & U_i^{o}
	\end{bmatrix}, \quad V_i = \begin{bmatrix}
		V_i^\sigma & V_i^{o}
	\end{bmatrix}
\end{equation}
with $U_i^\sigma \in\mathbb{R}^{d_i\times q_i}$, $U_i^{o}\in\mathbb{R}^{d_i\times p_i^u}$, $\Lambda_i\in\mathbb{R}^{q_i\times q_i}$, $V_i^\sigma \in\mathbb{R}^{d_{i-1}\times q_i}$, $V_i^{o}\in\mathbb{R}^{d_{i-1}\times p_i^v}$, where we have defined $q_i=\min\{d_i,d_{i-1}\}$, $p_i^u=d_i-q_i$ and $p_i^v=d_{i-1}-q_i$, accordingly, for $i=1,...,h$. Let also $q=\sum_{i=1}^h q_i$, $p^u=\sum_{i=1}^h p_i^u$, $p^v=\sum_{i=1}^h p_i^v$. Note that $W_i=U_i^\sigma \Lambda_iV_i^{\sigma \top}$, corresponding to a thin SVD. Now let
\begin{equation}
	U = 
	\left[
	\begin{array}{cccc|cccc|c}
		0 & & & & 0 & & & & U_{h+1} \\
		U_1^\sigma & 0 & & & U_{1}^o & 0 & & &0  \\
		0 & U_2^\sigma & \ddots & & 0 & U_{2}^o &\ddots  & & \vdots\\
		& \ddots & \ddots &0  & & \ddots & \ddots & 0 & \\
		& & 0 & U_h^\sigma & & & 0 & U_{h}^o & 0\\
	\end{array}
	\right] = \left[
	\begin{array}{c|c|c}
		U^\sigma & U^o & U_{h+1}^o \\
	\end{array}
	\right] \in\mathbb{R}^{n_v\times n_v}
    \label{eq:svd_U}
\end{equation}
where $U^\sigma\in\mathbb{R}^{n_v\times q}$, $U^o\in\mathbb{R}^{n_v\times p^u}$, and $U_{h+1}^o\in\mathbb{R}^{n_v\times d_x}$ is a tall matrix with a unitary block $U_{h+1}\in\mathbb{R}^{d_x\times d_x}$ at the top. Similarly,
\begin{equation}
	V = 
	\left[
	\begin{array}{cccc|cccc|c}
		V_1^{\sigma} & 0 & & & V_{1}^o & 0 & & & 0 \\
		0 & V_2^{\sigma} & \ddots & & 0 & V_{2}^o & \ddots & & \vdots\\
		& \ddots & \ddots & 0 & & \ddots & \ddots & 0 & \\
		& & & V_{h}^{\sigma} & & & & V_{h}^{o} & 0 \\
		& & & 0 & & & & 0 & V_{h+1}^{\sigma} \\
	\end{array}
	\right] = \left[
	\begin{array}{c|c|c}
		V^\sigma & V^o & V_{h+1}^o \\
	\end{array}
	\right] \in\mathbb{R}^{n_v\times n_v}
    \label{eq:svd_V}
\end{equation}
with  $V^\sigma\in\mathbb{R}^{n_v\times q}$, $V^o\in\mathbb{R}^{n_v\times p^v}$, $V_{h+1}^o\in\mathbb{R}^{n_v\times d_y}$, and $V_{h+1}\in\mathbb{R}^{d_y\times d_y}$. Define the diagonal matrix
\begin{equation}\label{eq:svddiag}
	\Omega = \begin{bmatrix}
		\Lambda_1 & & & & \\
		& \Lambda_2 & & & \\
		& & \ddots & & \\
		& & & \Lambda_h & \\
		& & & & 0 \\
	\end{bmatrix} = \begin{bmatrix}
	\Lambda & 0 \\
	0 & 0
	\end{bmatrix} \in\mathbb{R}^{n_v\times n_v}
\end{equation}
where the diagonal block $\Lambda$ is $q\times q$, and the last $n_v-q$ rows and columns of $\Omega$ are zeros. This means that $U\Omega V^\top  = U^\sigma \Lambda V^{\sigma \top}$, hence
\begin{align}
	U\Omega V^\top  & = \begin{bmatrix}
		0 & & & \\
		U_1^\sigma & 0 & & \\
		0 & U_2^\sigma & \ddots & \\
		& \ddots & \ddots & 0\\
		& & 0 & U_h^\sigma \\
	\end{bmatrix}
	\begin{bmatrix}
		\Lambda_1 & & & \\
		& \Lambda_2 & & \\
		& & \ddots & \\
		& & & \Lambda_h \\
	\end{bmatrix}
	\begin{bmatrix}
		V_1^{\sigma \top} & 0 & & &  \\
		0 & V_2^{\sigma \top} & \ddots & & \\
		& \ddots & \ddots & 0 & \\
		& & 0 & V_h^{\sigma \top} & 0 \\
	\end{bmatrix} \nonumber \\
	& =\begin{bmatrix}
		0 & & & & \\
		U_1^\sigma \Lambda_1 V_1^{\sigma \top} & 0 & & & \\
		0 & U_2^\sigma \Lambda_2 V_2^{\sigma \top} & \ddots & & \\
		& \ddots & \ddots & & \\
		& & 0 & U_h^\sigma \Lambda_h V_h^{\sigma \top} & 0 \\
	\end{bmatrix}
\end{align}
As noted before, $W_i=U_i^\sigma \Lambda_iV_i^{\sigma \top}$, so $A=U\Omega V^\top $. $\Omega$ is diagonal by construction, and each column in $U$ is orthogonal to all other columns in $U$ (this follows from each $U_i$ being orthogonal, and columns from different $U_i$:s in $U$ having no nonzero elements in the same position), this of course holds also for $V$. Thus, $A=U\Omega V^\top $ is a singular value decomposition of $A$, and it has the same singular values as the blocks, with $n_v-q$ extra zeros, as seen from \eqref{eq:svddiag}.
\end{proof}

\subsection{A sufficient condition for equivalence of adjacency matrices}

The following proposition provides a sufficient condition for $ A_1 \sim A_2 $ which is slightly more general than the one given in Proposition~\ref{prop:equival-classes-A}. 
\begin{proposition}
\label{prop:equival-classes-power}
Consider $ A_i = \block_1 ( W_1^{(i)}, \, \ldots , \, W_h^{(i)})\in \mathcal{A}$, $i=1,\, 2$.
Then $ A_1 \sim A_2 $ if $ A_2 = P (A_1 + Z ) P^{-1} $, where $ P \in \mathcal{P} $ and $ Z =\block_1(Z_1, \ldots, Z_h ) \in \mathcal{A} $ s.t. $ Z_j \in \mathcal{N}(W_{j+1}^{(1)}) $, $ Z_j^\top \in \mathcal{N}((W_{j-1}^{(1)})^\top) $ and $ Z^h =0$.
\end{proposition}

\begin{proof}
The nilpotent condition $ Z^h=0 $ corresponds in terms of the $ Z_1, \ldots, Z_h $ blocks to $ Z_{1:h}=0$, which is achieved for instance when $ Z_j =0$ for some $j$.  
Let us observe that the statement on $ A_2 = P (A_1 + Z ) P^{-1} $ can be decomposed into the two following conditions
\benu
\item (change of basis in hidden nodes) $ \exists P \in \mathcal{P} $ s.t. $ A_2 = P A_1 P^{-1}$;
\item (mapping to null space in hidden layer product) For some $ j=1,\ldots, h$, $ W_j^{(2)} = W_j^{(1)} + Z_j $ where  $ Z_j \in \mathbb{R}^{d_j \times d_{j-1}} $ is such that $ Z_j \in \mathcal{N}(W_{j+1}^{(1)}) $ and $ Z_j^\top \in \mathcal{N}((W_{j-1}^{(1)})^\top) $, plus $ Z_{1:h}=0$.
\eenu

Let us show that any (or both) of these conditions imply $ A_1 \sim A_2 $.
\benu
\item If $ P \in \mathcal{P}$, then $ P_j $ is a change of basis at the hidden nodes of layer $ j$.
 Each block of weights $ W_j^{(i)} $ is affected by the changes of basis happening at its upstream and downstream hidden nodes: $ W_j^{(2)} = P_j W_j^{(1)} P_{j-1}^{-1}$ (with $ P_0=I_{d_x} $ and $ P_h=I_{d_y} $). 
Computing the product:
\begin{align*}
W_{1:h}^{(2)} & = W_h^{(2)} W_{h-1}^{(2)} \ldots  W_2^{(2)} W_1^{(2)} \\
& = W_h^{(1)} P_{h-1}^{-1} P_{h-1} W_{h-1}^{(1)} P_{h-2}^{-1} \ldots  P_2 W_2^{(1)} P_1^{-1} P_1  W_1^{(1)} \\
& = W_h^{(1)} W_{h-1}^{(1)} \ldots  W_2^{(1)} W_1^{(1)} = W_{1:h}^{(1)} 
\end{align*}
Therefore $ A_1^h = A_2^h$ follows from \eqref{eq:Apowerh}.
\item If $ W_j^{(2)} = W_j^{(1)} + Z_j $ with $ Z_j \in \mathcal{N}(W_{j+1}^{(1)}) $ resp. $ Z_j^\top \in \mathcal{N}((W_{j-1}^{(1)})^\top) $, then  $  W_{j+1}^{(1)} Z_j =0 $ resp.  $  Z_j W_{j-1}^{(1)}  =0 $. Expanding the matrix product
\[
\begin{split}
W_{1:h}^{(2)}  = & W_{1:h}^{(1)} + Z_h W_{1:h-1}^{(1)} + W_h^{(1)} Z_{h-1} W_{1:h-2}^{(1)} + \ldots + W_{2:h}^{(1)}Z_1 + Z_h Z_{h-1} W_{1:h-2}^{(1)} + \ldots \\ & \ldots +Z_h W_{h-1}^{(1)} Z_{h-2} W_{1:h-3}^{(1)} + \ldots + 
Z_h Z_{h-1} \ldots Z_2 W_1^{(1)} + Z_h Z_{h-1} \ldots Z_2 Z_1 
\end{split}
\]
Each term in the sum except the first is equal to zero, including the last one if $ Z_{1:h}=0 $.
\eenu
\end{proof}

Proposition~\ref{prop:equival-classes-A} follows straightforwardly from Proposition~\ref{prop:equival-classes-power}. 
In fact $ A_1Z =0 $ implies $ W_{j+1}^{(1)} Z_j =0 \; \forall \,j$, while $ ZA_1=0$ implies $ Z_j W_{j-1}^{(1)}=0 \; \forall \,j $.

The condition of Proposition~\ref{prop:equival-classes-power} is sufficient but not necessary. For instance it could be that $ Z_j \notin \mathcal{N}(W_{j+1}^{(1)}) $ but $ Z_j \in \mathcal{N}(W_{j+2}^{(1)}W_{j+1}^{(1)}) $.


\begin{thebibliography}{10}

\bibitem{achour2021loss}
E.~M. Achour, F.~Malgouyres, and S.~Gerchinovitz.
\newblock The loss landscape of deep linear neural networks: a second-order analysis.
\newblock {\em J. Mach. Learn. Res.}, 25(242):1--76, 2024.

\bibitem{arora2018convergence}
S.~Arora, N.~Cohen, N.~Golowich, and W.~Hu.
\newblock A convergence analysis of gradient descent for deep linear neural networks.
\newblock In {\em Proceedings of the International Conference on Learning Representations}, 2019.

\bibitem{arora2018optimization}
S.~Arora, N.~Cohen, and E.~Hazan.
\newblock On the optimization of deep networks: Implicit acceleration by overparameterization.
\newblock In {\em Proceedings of the International Conference on Machine Learning}, pages 244--253. PMLR, 2018.

\bibitem{arora2019implicit}
S.~Arora, N.~Cohen, W.~Hu, and Y.~Luo.
\newblock Implicit regularization in deep matrix factorization.
\newblock In {\em Advances in Neural Information Processing Systems}, 2019.

\bibitem{bah2022learning}
B.~Bah, H.~Rauhut, U.~Terstiege, and M.~Westdickenberg.
\newblock Learning deep linear neural networks: Riemannian gradient flows and convergence to global minimizers.
\newblock {\em Information and Inference: A Journal of the IMA}, 11(1):307--353, 2022.

\bibitem{baldi1989neural}
P.~Baldi and K.~Hornik.
\newblock Neural networks and principal component analysis: Learning from examples without local minima.
\newblock {\em Neural Netw.}, 2(1):53--58, 1989.

\bibitem{belabbas2020implicit}
M.~A. Belabbas.
\newblock On implicit regularization: Morse functions and applications to matrix factorization.
\newblock {\em arXiv preprint arXiv:2001.04264}, 2020.

\bibitem{blum1988training}
A.~Blum and R.~Rivest.
\newblock Training a 3-node neural network is np-complete.
\newblock In {\em Advances in Neural Information Processing Systems}, 1988.

\bibitem{blumberg1968logistic}
A.~A. Blumberg.
\newblock Logistic growth rate functions.
\newblock {\em J. Theor. Biol.}, 21(1):42--44, 1968.

\bibitem{chitour2018geometric}
Y.~Chitour, Z.~Liao, and R.~Couillet.
\newblock A geometric approach of gradient descent algorithms in linear neural networks.
\newblock {\em Math. Control Relat. Fields}, 13(3):918--945, 2023.

\bibitem{chizat2024infinite}
L.~Chizat, M.~Colombo, X.~Fern{\'a}ndez-Real, and A.~Figalli.
\newblock Infinite-width limit of deep linear neural networks.
\newblock {\em Comm. Pure Appl. Math.}, 77(10):3958--4007, 2024.

\bibitem{cohen2024lecture}
N.~Cohen and N.~Razin.
\newblock Lecture notes on linear neural networks: A tale of optimization and generalization in deep learning.
\newblock {\em arXiv preprint arXiv:2408.13767}, 2024.

\bibitem{de1992structure}
B.~L. De~Moor.
\newblock On the structure and geometry of the product singular value decomposition.
\newblock {\em Linear Algebra Appl.}, 168:95--136, 1992.

\bibitem{domine2023exact}
C.~C. Domin{\'e}, L.~Braun, J.~E. Fitzgerald, and A.~M. Saxe.
\newblock Exact learning dynamics of deep linear networks with prior knowledge.
\newblock {\em J. Stat. Mech.: Theory Exp.}, 2023(11):114004, 2023.

\bibitem{du2018algorithmic}
S.~S. Du, W.~Hu, and J.~D. Lee.
\newblock Algorithmic regularization in learning deep homogeneous models: Layers are automatically balanced.
\newblock In {\em Advances in Neural Information Processing Systems}, 2018.

\bibitem{fukumizu1998effect}
K.~Fukumizu.
\newblock Effect of batch learning in multilayer neural networks.
\newblock In {\em Proceedings of the International Conference on Neural Information Processing}, pages 67--70, 1998.

\bibitem{gidel2019implicit}
G.~Gidel, F.~Bach, and S.~Lacoste-Julien.
\newblock Implicit regularization of discrete gradient dynamics in linear neural networks.
\newblock In {\em Advances in Neural Information Processing Systems}, 2019.

\bibitem{gissin2019implicit}
D.~Gissin, S.~Shalev-Shwartz, and A.~Daniely.
\newblock The implicit bias of depth: How incremental learning drives generalization.
\newblock In {\em Proceedings of the International Conference on Learning Representations}, 2020.

\bibitem{glorot2010understanding}
X.~Glorot and Y.~Bengio.
\newblock Understanding the difficulty of training deep feedforward neural networks.
\newblock In {\em Proceedings of the thirteenth international conference on artificial intelligence and statistics}, pages 249--256. JMLR Workshop and Conference Proceedings, 2010.

\bibitem{golub2013matrix}
G.~Golub and C.~Van~Loan.
\newblock {\em Matrix Computations}.
\newblock Johns Hopkins Studies in the Mathematical Sciences. Johns Hopkins University Press, 4th edition, 2013.

\bibitem{gunasekar2017implicit}
S.~Gunasekar, B.~E. Woodworth, S.~Bhojanapalli, B.~Neyshabur, and N.~Srebro.
\newblock Implicit regularization in matrix factorization.
\newblock In {\em Advances in Neural Information Processing Systems}, volume~30, 2017.

\bibitem{hirsch1974differential}
M.~Hirsch, R.~Devaney, and S.~Smale.
\newblock {\em Differential Equations, Dynamical Systems, and Linear Algebra}.
\newblock Pure and Applied Mathematics. Elsevier Science, 1974.

\bibitem{jacot2021saddle}
A.~Jacot, F.~Ged, B.~{\c{S}}im{\c{s}}ek, C.~Hongler, and F.~Gabriel.
\newblock Saddle-to-saddle dynamics in deep linear networks: Small initialization training, symmetry, and sparsity.
\newblock {\em arXiv preprint arXiv:2106.15933}, 2021.

\bibitem{kawaguchi2016deep}
K.~Kawaguchi.
\newblock Deep learning without poor local minima.
\newblock In {\em Advances in Neural Information Processing Systems}, 2016.

\bibitem{khalil2002nonlinear}
H.~Khalil.
\newblock {\em Nonlinear Systems}.
\newblock Pearson Education. Prentice Hall, 2002.

\bibitem{kunin2024get}
D.~Kunin, A.~Ravent{\'o}s, C.~Domin{\'e}, F.~Chen, D.~Klindt, A.~Saxe, and S.~Ganguli.
\newblock Get rich quick: exact solutions reveal how unbalanced initializations promote rapid feature learning.
\newblock In {\em Advances in Neural Information Processing Systems}, pages 81157--81203, 2024.

\bibitem{lampinen2018analytic}
A.~K. Lampinen and S.~Ganguli.
\newblock An analytic theory of generalization dynamics and transfer learning in deep linear networks.
\newblock In {\em Proceedings of the International Conference on Learning Representations}, 2019.

\bibitem{laurent2018deep}
T.~Laurent and J.~Brecht.
\newblock Deep linear networks with arbitrary loss: All local minima are global.
\newblock In {\em Proceedings of the International Conference on Machine Learning}, pages 2902--2907. PMLR, 2018.

\bibitem{lecun2015deep}
Y.~LeCun, Y.~Bengio, and G.~Hinton.
\newblock Deep learning.
\newblock {\em Nature}, 521(7553):436--444, 2015.

\bibitem{lee2019first}
J.~D. Lee, I.~Panageas, G.~Piliouras, M.~Simchowitz, M.~I. Jordan, and B.~Recht.
\newblock First-order methods almost always avoid strict saddle points.
\newblock {\em Math. Program.}, 176:311--337, 2019.

\bibitem{li2020towards}
Z.~Li, Y.~Luo, and K.~Lyu.
\newblock Towards resolving the implicit bias of gradient descent for matrix factorization: Greedy low-rank learning.
\newblock In {\em Proceedings of the International Conference on Learning Representations}, 2021.

\bibitem{lojasiewicz1982trajectoires}
S.~Lojasiewicz.
\newblock Sur les trajectoires du gradient d’une fonction analytique.
\newblock {\em Seminari di geometria}, 1983:115--117, 1982.

\bibitem{min2023convergence}
H.~Min, R.~Vidal, and E.~Mallada.
\newblock On the convergence of gradient flow on multi-layer linear models.
\newblock In {\em Proceedings of the International Conference on Machine Learning}, pages 24850--24887. PMLR, 2023.

\bibitem{panageas2016gradient}
I.~Panageas and G.~Piliouras.
\newblock Gradient descent only converges to minimizers: Non-isolated critical points and invariant regions.
\newblock In {\em 8th Innovations in Theoretical Computer Science Conference (ITCS 2017)}, pages 2:1--2:12, 2017.

\bibitem{Petersen2012matrix}
K.~Petersen and M.~Petersen.
\newblock The matrix cookbook.
\newblock {\em http://matrixcookbook.com}, 2012.

\bibitem{saxe2013exact}
A.~Saxe, J.~McClelland, and S.~Ganguli.
\newblock Exact solutions to the nonlinear dynamics of learning in deep linear neural networks.
\newblock In {\em Proceedings of the International Conference on Learning Represenatations 2014}, 2014.

\bibitem{saxe2019mathematical}
A.~M. Saxe, J.~L. McClelland, and S.~Ganguli.
\newblock A mathematical theory of semantic development in deep neural networks.
\newblock {\em Proc. Natl. Acad. Sci. USA}, 116(23):11537--11546, 2019.

\bibitem{tam1994circularity}
B.-S. Tam.
\newblock Circularity of numerical ranges and block-shift matrices.
\newblock {\em Linear and Multilinear Algebra}, 37(1-3):93--109, 1994.

\bibitem{tarmoun2021understanding}
S.~Tarmoun, G.~Franca, B.~D. Haeffele, and R.~Vidal.
\newblock Understanding the dynamics of gradient flow in overparameterized linear models.
\newblock In {\em Proceedings of the International Conference on Machine Learning}, pages 10153--10161. PMLR, 2021.

\bibitem{trager2019pure}
M.~Trager, K.~Kohn, and J.~Bruna.
\newblock Pure and spurious critical points: a geometric study of linear networks.
\newblock In {\em Proceedings of the International Conference on Learning Representations}, 2020.

\bibitem{tsoularis2002analysis}
A.~Tsoularis and J.~Wallace.
\newblock Analysis of logistic growth models.
\newblock {\em Math. Biosci.}, 179(1):21--55, 2002.

\bibitem{tu2024mixed}
Z.~Tu, S.~T. Aranguri~Diaz, and A.~Jacot.
\newblock Mixed dynamics in linear networks: Unifying the lazy and active regimes.
\newblock In {\em Advances in Neural Information Processing Systems}, pages 106059--106104, 2024.

\bibitem{yun2017global}
C.~Yun, S.~Sra, and A.~Jadbabaie.
\newblock Global optimality conditions for deep neural networks.
\newblock In {\em Proceedings of the International Conference on Learning Representations}, 2018.

\bibitem{zhang2017matrix}
X.-D. Zhang.
\newblock {\em Matrix analysis and applications}.
\newblock Cambridge University Press, 2017.

\bibitem{zhou2017critical}
Y.~Zhou and Y.~Liang.
\newblock Critical points of neural networks: Analytical forms and landscape properties.
\newblock In {\em Proceedings of the International Conference on Learning Representations}, 2018.

\end{thebibliography}

\newpage 

\end{document}